%% file: main.tex
\documentclass[accepted]{uai2026} % after acceptance, for a revised version; 
\usepackage[american]{babel}
\usepackage{natbib} % has a nice set of citation styles and commands
\usepackage{mathtools} % amsmath with fixes and additions
\usepackage{tikz} % nice language for creating drawings and diagrams

\usepackage{graphicx}
\usepackage{url}
\usepackage{amsmath}
\usepackage{amsthm}
\usepackage{amssymb}
\usepackage{booktabs}
\usepackage{subcaption}
\usepackage{multirow}
\usepackage{float}
\usepackage{algpseudocode}
\usepackage[ruled,lined,noend]{algorithm2e}
\SetKwInOut{Input}{Input}
\SetKwInOut{Output}{Output}
\usepackage{cleveref}
\usepackage{enumitem}
\usepackage{makecell}
\usepackage{fixltx2e}
\usepackage{bm}
\usepackage{chngcntr}
\usepackage{xcolor}

\newtheorem{remark}{Remark}
\newtheorem{theorem}{Theorem}
\newtheorem{lemma}{Lemma}

\usepackage{newfloat}
\usepackage{listings}

\title{Communication-Efficient Distributed Training for \\ Collaborative Flat Optima Recovery in Deep Learning}

\author[1]{\href{mailto:<td2249@nyu.edu>?Subject=UAI 2026 - Communication-Efficient Distributed Training for Collaborative Flat Optima Recovery in Deep Learning}{Tolga Dimlioglu}{}}
\author[1]{Anna Choromanska}

\affil[1]{%
    Electrical Engineering Dept.\\
    New York University\\
    NY, USA
}

\begin{document}
\maketitle

% \begin{abstract}
% We study centralized distributed data parallel training of deep neural networks (DNNs), aiming to improve the trade-off between communication efficiency and model performance of the local gradient methods. We revisit the flat-minima hypothesis, which suggests that models with better generalization tend to lie in flatter regions of the loss landscape. We introduce a simple, yet effective, sharpness measure, Inverse Mean Valley, and demonstrate its strong correlation with the generalization gap of DNNs. We incorporate an efficient relaxation of this measure into the distributed training objective as a lightweight regularizer that encourages workers to collaboratively seek wide minima. The regularizer exerts a pushing force that counteracts the consensus step pulling the workers together, giving rise to the Distributed Pull-Push Force (DPPF) algorithm. Empirically, we show that DPPF outperforms other communication-efficient approaches and achieves better generalization performance than local gradient methods and synchronous gradient averaging, while maintaining communication efficiency. In addition, our loss landscape visualizations confirm the ability of DPPF to locate flatter minima. On the theoretical side, we show that DPPF guides workers to span flat valleys, with the final valley width governed by the interplay between push and pull strengths, and that its pull-push dynamics is self-stabilizing. We further provide generalization guarantees linked to the valley width and prove convergence in the non-convex setting.
% \end{abstract}

\begin{abstract}
We study centralized distributed data parallel training of deep neural networks (DNNs), aiming to improve the trade-off between communication efficiency and model performance of local gradient methods. Motivated by the flat-minima hypothesis, we first introduce a simple sharpness measure, Inverse Mean Valley, and show it strongly correlates with the generalization gap of DNNs. We then incorporate an efficient relaxation of this measure into the distributed objective as a lightweight regularizer that encourages workers to seek wide minima collaboratively. The regularizer exerts a pushing force that counteracts the consensus step pulling the workers together, giving rise to the Distributed Pull-Push Force (DPPF) algorithm. Empirically, DPPF generalizes better than other local gradient methods and synchronous gradient averaging while maintaining communication efficiency. In addition, our loss landscape visualizations confirm the ability of DPPF to locate flatter minima. Theoretically, we show that DPPF drives workers to span flat valleys with valley width governed by push–pull strengths, it yields self-stabilizing dynamics, it obeys generalization guarantees that depend on valley width, and it converges in the non-convex setting.
\end{abstract}

\section{Introduction}
\label{sec:intro}

We consider a distributed data-parallel setup with $M$ workers\footnote{a \textit{worker} refers to a single hardware unit (e.g., GPU)} for training a DNN, with samples assumed to be independent and identically distributed (IID). The goal is to collaboratively optimize a shared model vector $\mathbf{x}$ that minimizes the global training loss, formally defined in Equation~\ref{problem_form:objective}. Here, $f$ is a non-convex objective, and $F_m(x; \xi)$ denotes the stochastic loss observed by worker $m$. Each worker receives independent, unbiased stochastic gradients of its local objective.

\begin{equation}
\min_{x \in \mathbb{R}^d} \frac{1}{M} \sum_{m=1}^M f_m(\mathbf{x}) \; \text{where} \; f_m(x) = \mathbb{E}[F_m(x; \xi)] \label{problem_form:objective}
\end{equation}

\begin{figure*}[htp]
  \centering
  \begin{subfigure}[t]{0.242\textwidth}
    \includegraphics[width=\linewidth]{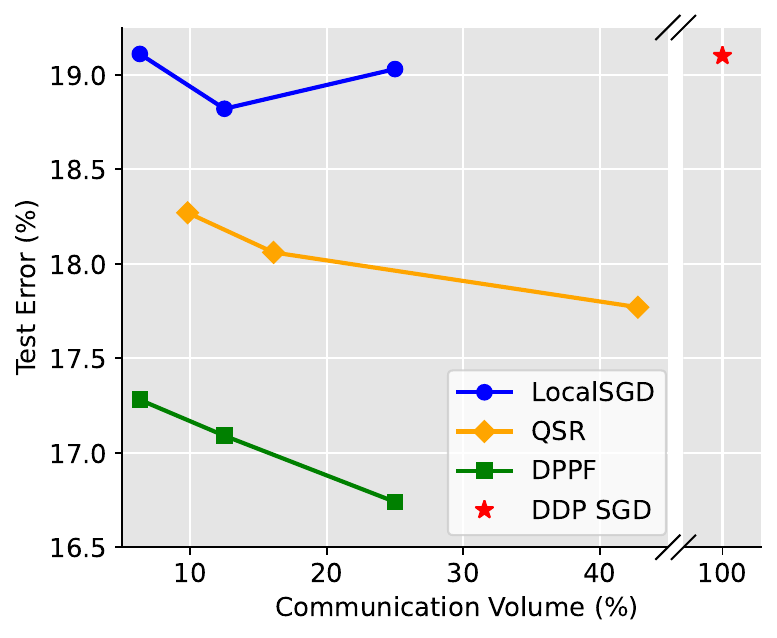}
    \caption{PyNet, CIFAR-100, 8W}
    \label{fig:comm_vol_pyramidnet}
  \end{subfigure}
  \hfill
  \begin{subfigure}[t]{0.245\textwidth}
    \includegraphics[width=\linewidth]{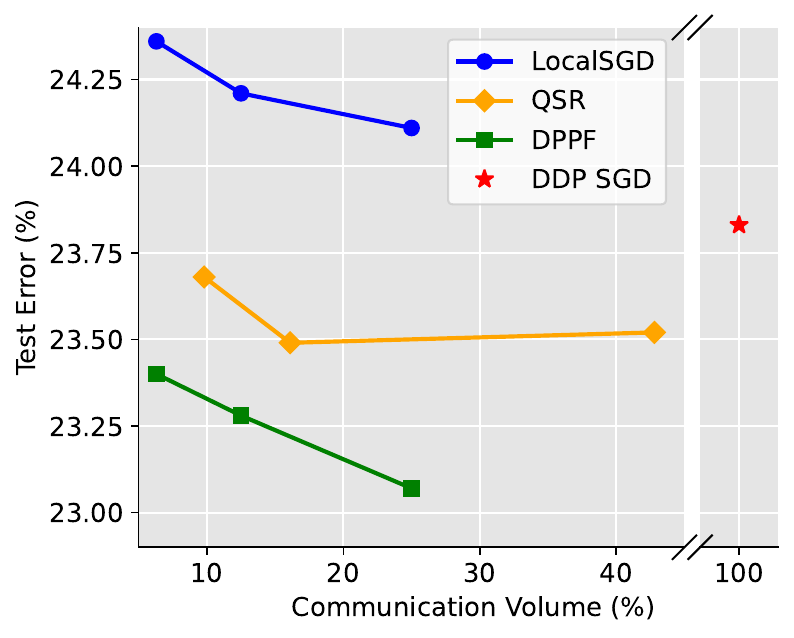}
    \caption{RN-50, ImageNet, 4W}
    \label{fig:comm_vol_resnet50}
  \end{subfigure}
    \hfill
  \begin{subfigure}[t]{0.242\textwidth}
    \includegraphics[width=\linewidth]{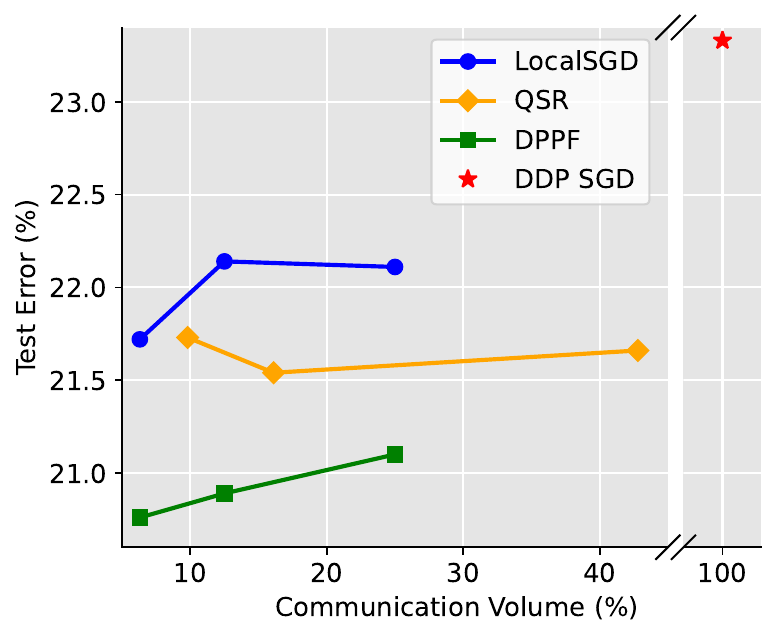}
    \caption{RN-101, ImageNet, 4W}
    \label{fig:comm_vol_resnet101}
  \end{subfigure}
  \hfill
  \begin{subfigure}[t]{0.243\textwidth}
    \includegraphics[width=\linewidth]{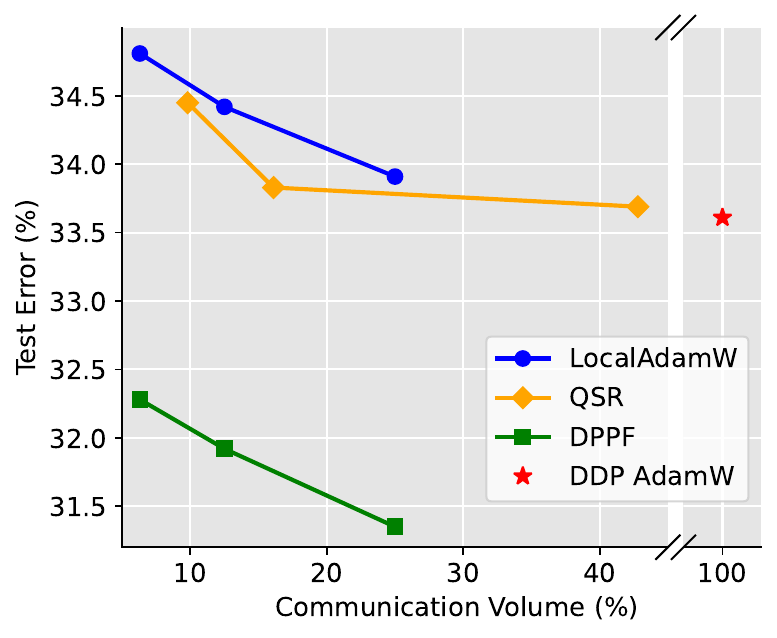}
    \caption{ViT, ImageNet, 4W}
    \label{fig:comm_vol_vit}
  \end{subfigure}
  \hfill
  \caption{Communication volume (ratio between the number of communication rounds till convergence to the total number of local iterations; lower is better) vs. test error (\%) for baseline distributed training methods (LocalSGD, QSR and DDP SGD) and our approach, DPPF. (W: Worker)}
  \label{fig:error_vs_comm_volume}
\end{figure*}

In standard data parallelism, each worker computes gradients on its data shard, which are then aggregated (typically via averaging) and applied synchronously across devices~\citep{mcdonald2009efficient, pytorch_distributed}. Despite its effectiveness, this method introduces a significant communication bottleneck~\citep{harlap2018pipedream} due to frequent synchronization. Alternative strategies mitigate this by allowing workers to perform several local gradient steps independently before periodic synchronization, either through hard resets~\citep{stich2019local} or soft pulling~\citep{zhang2015deep}. However, such approaches often yield inferior performance compared to fully synchronous gradient averaging~\citep{yu2019parallel, ortiz2021trade}.

Our goal in this work is to develop a mechanism that enables communication-efficient training methods to match/exceed the performance of synchronous gradient averaging, without incurring significant computational overhead. To this end, we revisit the flatness hypothesis in the literature, which suggests that models with better generalization tend to converge to flatter regions of the loss landscape~\citep{keskar2016large}. Our key contributions can be listed as follows:

\begin{itemize}[wide]
    \item We introduce a new sharpness measure, Inverse Mean Valley (Inv. MV), which shows strong correlation with the generalization gap in DNNs trained using communication-efficient methods and outperforms existing sharpness metrics in comparative evaluations.
    \item We propose a lightweight relaxation of the Inv. MV measure that can be efficiently integrated into the training objective, inducing a push force that counteracts the periodic consensus steps in communication-efficient methods. This gives rise to our algorithm, Distributed Pull-Push Force (DPPF), where the push component keeps workers apart, preventing collapse, avoiding convergence to narrow valleys, and positioning them near the boundaries of wide, flat regions in the loss landscape.
    \item We theoretically characterize the interplay between the pull and push forces in DPPF and show how they govern the final distance between the workers and their average. We also provide generalization guarantees tied to this distance.
    
    \item We experimentally validate our approach on standard benchmark datasets and architectures, conduct extensive ablation studies to analyze its underlying mechanisms, and confirm its ability to recover wide minima through loss landscape visualizations.
\end{itemize}

\textbf{Motivating Plot} Figure~\ref{fig:error_vs_comm_volume} highlights the problem with existing baselines that struggle to sustain good performance and remain communication efficient at the same time. This challenge motivates our work. Our proposed method, allows a better tradeoff between communication and performance.

\section{Related Work}\label{sec:RW}

\textbf{Data Parallel Training of DNNs} A widely adopted approach in data-parallel training requires that each worker computes a stochastic gradient on its local data, aggregates these gradients via averaging, and updates the model using the aggregated gradient. This scheme, variously called Parallel Mini‑Batch~\citep{dist_mini_batch}, All‑Reduce~\citep{lin2018don}, or DistributedDataParallel (DDP)~\citep{pytorch_distributed}, is communication‑intensive because gradients must be exchanged every iteration \citep{harlap2018pipedream}. To reduce communication costs, local‑update schemes let each worker take several SGD steps before a global average. LocalSGD \citep{stich2019local,zhang2016parallel, zhou2018convergence, lian2015asynchronous} exemplifies this idea and is also employed in large‑scale DNN training \citep{su2015experiments, chen2016scalable}. Although it enables communication efficiency, performance degrades when either the communication interval or worker count grows \citep{stich2019local, yu2019parallel, ortiz2021trade}; adding a global momentum term partially offsets this loss \citep{wangslowmo}. Another way to improve model performance while training with longer communication periods was explored in the past work. Theory shows that the gradient noise from local updates can boost generalization: LocalSGD outperforms DDP when the learning rate is small and training is long \citep{whyandwhenlocalsgd}. Extending this analysis, the quadratic synchronization rule (QSR) adaptively updates the communication period inversely proportional to the learning rate, yielding improvement in model generalization while being communication-efficient \citep{guquadratic}. Additionally, it has been shown that gossip communication in the decentralized setting can also improve test accuracy, contradicting the notion that decentralization hurts generalization \citep{zhu2023decentralized}. Since full consensus can hinder local exploration, several soft-consensus methods have been proposed to balance synchronization and exploration. Elastic Averaging SGD (EASGD) softly pulls workers to a moving‑average center to preserve exploration \citep{zhang2015deep}. Leader SGD (LSGD) pulls all workers toward the one with the lowest loss to accelerate convergence \citep{teng2019leader}, while GRAWA takes a flatness-aware approach by weighting workers inversely with their gradient norms \citep{dimlioglu2024grawa}.

\textbf{Flatness-Aware Optimizers} The concept of flat minima has motivated extensive research on the geometry of DNN loss landscapes~\citep{hochreiter1994simplifying, goodfellow2014qualitatively, dauphin2014identifying, baldassi2015subdominant, li2018visualizing} and its connection to generalization~\citep{keskar2016large, jastrzkebski2017three, andriushchenko2023modern}. While many studies associate sharpness with generalization, relatively few optimization methods are explicitly designed to seek flat minima. Some examples include Entropy-SGD~\citep{chaudhari2019entropy}, which smooths the loss surface via local entropy, and Low-Pass Filtering SGD~\citep{bisla2022low}, which optimizes a Gaussian-smoothed objective function. A widely popularized method is Sharpness-Aware Minimization (SAM)~\citep{SAM,kwon2021asam, zhuang2022surrogate, li2024enhancing}, which seeks parameters that lie in neighborhoods having uniformly low loss via a min-max optimization.% In contrast to these techniques, Random Weight Perturbation (RWP)~\citep{li2022efficient, lirevisiting} applies random directional perturbations. Another notable approach is Stochastic Weight Averaging (SWA)~\citep{izmailovaveraging}, where averaging SGD iterates along with using a cyclical or constant learning rate yields flatter solutions and improved generalization.

\textbf{Note:} Review of additional related works can be found in Section~\ref{appendix:add_related_works} of the Appendix.

\section{Preliminary}
The data-parallel training procedure to minimize the objective presented in Equation~\ref{problem_form:objective} is detailed in Algorithm~\ref{alg:generic}. The algorithm initializes the model randomly and distributes it to all workers, then partitions the dataset into non-overlapping shards. Each worker trains its local model (e.g., with SGD), and synchronizes every $\tau$ iterations by computing a consensus variable $x_C$—typically the average model $x_A$, but potentially any combination of worker parameters. Each worker then updates its parameters by interpolating between $x_m$ and $x_C$. When $x_C = x_A$, $\tau = 1$, and $\alpha = 1$, the method reduces to standard gradient averaging (DDP). Setting $\tau > 1$ yields LocalSGD, and using $\alpha < 1$ implements soft consensus, where workers are only partially pulled toward $x_C$. Finally, the algorithm returns the averaged parameters. The DPPF-specific pushing mechanism shown in the pseudo-code is explained later.

\begin{algorithm}[h]
\caption{Data Parallel Training}
\label{alg:generic}
\small
\Input{Pull $\alpha \in (0,1]$, comm. period $\tau$, learning rate $\eta$}
\textbf{Initialize} parameters $x_1, ..., x_M$, $t_1=\ldots=t_M=0$, and worker exclusive data shards $\Psi_1,...,\Psi_M$

\For{each worker $m$ in parallel}{
  \While{not converged}{
    Draw batch $\xi_m \in \Psi_m$ \\
    $x_m \leftarrow x_m - \eta \nabla f(x_m; \xi_m)$ \\
    $t_m \leftarrow t_m + 1$ \\
    \If{$t_m \bmod \tau = 0$}{
      Obtain $x_C$ via LocalSGD, EASGD etc. \\
      $x_m \leftarrow (1-\alpha)x_m + \alpha x_C$ \\
      \If{push}{
        $x_m \leftarrow x_m + \lambda \frac{x_m - x_A}{\|x_m - x_A\|}$ \\
      }
    }
  }
  $x_A = \frac{1}{M}\sum_{i=1}^M x_i$
}
\textbf{return} $x_A$
\end{algorithm}

\section{Mean Valley Measure}
\label{sec:MV_measure}
In this section, we introduce the \textit{Mean Valley} (MV) measure, a simple yet effective metric for quantifying the flatness of minima \emph{after full convergence} in distributed training methods with independent local gradient steps. At a high level, MV estimates the average diameter of the valley surrounding the converged worker parameters. The procedure for computing MV begins with obtaining the average model $x_A = \frac{1}{M}\sum_{m=1}^M x_{m}$ and evaluating the training loss at that point $f( x_A, \mathcal{D}_{train})$. For each worker $m$, the unit vector $\delta_m$ pointing from $x_A$ to $x_m$ is then calculated. A line search is performed along this direction to find the point $x_m^b$ where the average loss increases by a factor of $\kappa$ ($\kappa > 1$), identifying $x_m^b$ as the valley boundary in that direction. Mathematically, this procedure can be written as follows:
\begin{align}
    \delta_m = \frac{x_m - x_A}{||x_m - x_A||_2}, \; & \; \mathcal{L}_A = f \left ( x_A, \mathcal{D}_{train} \right )\\
    x_{m}^b = x_A + \beta_m  \delta_m  \; \; \;\text{s.t.} & \; \; \; f \left ( x_{m}^b, \mathcal{D}_{train} \right ) \approx \kappa \mathcal{L}_A 
\end{align}

Notice that $\beta_m$ is the distance we must move from $x_A$ along $\delta_m$ to reach the $\kappa$-loss contour, in particular, $ \beta_m = \| x_{m}^b - x_A \|_2$.  Finally, the MV measure outputs the average distance between each boundary point $x_m^b$ and the average $x_A$: $\text{MV} = \frac{1}{M} \sum_{m=1}^M || x_{m}^b - x_A||_2$. Thus, for a fixed $\kappa > 1$, a larger MV indicates lower average directional curvature.

\subsection{Comparison with Other Measures}
Although MV is inherently a flatness measure, for consistent comparison with existing sharpness metrics, we define the \textit{Inverse Mean Valley (Inv. MV)} by taking the additive inverse of MV so that the larger values indicate sharper minima. We compare Inv. MV with seven sharpness measures drawn from the literature. To quantify the correlation between sharpness measures and the generalization gap (validation$-$train error (\%)), we use the Kendall rank correlation coefficient. We conduct this analysis by training ResNet-style~\citep{resnet} models on the CIFAR-10~\citep{cifar10} dataset under two settings with and without augmentation: (1) single-worker training (i.e., without distributed training), and (2) distributed training using the parameter-sharing method EASGD with $4$ workers. To generate minima with different geometries, we vary several hyperparameters, model capacity, and repeat each configuration across three random seeds (see Section~\ref{appendix:subsec:gen_vs_sharpness} in the Appendix for more details and sensitivity analysis of Inv. MV on $\kappa$).

\begin{table}[h]
\centering
\caption{Kendall rank coefficients calculated between the generalization gap and the sharpness measures.}
\label{table:sharpness_indicators}
\resizebox{1.00\columnwidth}{!}{ 
\begin{tabular}{c|cc|cc}
\textbf{}         & \multicolumn{2}{c|}{\textbf{Single Worker}} & \multicolumn{2}{c}{\textbf{EASGD}} \\ \cline{2-5} 
\textbf{Measures} & \textbf{w/ Aug.}  & \textbf{w/o Aug.} & \textbf{w/ Aug.}   & \textbf{w/o  Aug.} \\ \hline
Shannon Ent.      & 0.695          & 0.575            & -0.213          & -0.161           \\
$\epsilon$-sharpness       & 0.784          & \textbf{0.799}   & 0.254           & 0.472            \\
Fisher-Rao        & \textbf{\textit{0.799}} & 0.735            & \textbf{0.665}  & 0.101            \\
LPF               & 0.730          & 0.738            & 0.074           & \textbf{0.553}   \\
$\lambda_{\max}(H)$       & 0.773          & \textbf{0.799}   & 0.444           & 0.166            \\
Trace($H$)       & \textbf{0.817} & \textbf{\textit{0.792}}           & 0.484           & 0.188            \\
$||H||_{\text{frob}}$      & 0.796          & 0.787   & 0.510           & 0.170            \\
Inv. MV ($\kappa=2$)    & NA             & NA               & \textbf{\textit{0.616}}  & \textbf{\textit{0.485}}   
\end{tabular}
}
\end{table}

We present the correlation results in Table~\ref{table:sharpness_indicators}. As can be seen, the sharpness measures that have the strongest correlation with the generalization in the single-worker setting are no longer the best performers in the distributed setting.  Similarly, a notable finding is that the strength of the correlations varies between scenarios with and without augmentations, e.g., the Fisher-Rao metric is well-correlated with the generalization gap for the scenario with augmentations, but without augmentations, it is one of the worst metrics. The table finally demonstrates a strong correlation of Inv. MV with the generalization gap (second-best metric), as well as it exhibits consistent behavior in settings with and without augmentations as opposed to the rest of the measures.

\section{Distributed Pull-Push Force} % Algorithm
\label{sec:dppf}
Motivated by the strong correlation between the Inv. MV measure and the generalization gap in DNNs trained with local-gradient methods, we propose incorporating it as a regularization term into the training objective to encourage collaborative exploration of wide minima. However, like other measures, Inv. MV suffers from computational inefficiency. Calculating these measures typically requires either loss landscape sampling (as in $\epsilon$-sharpness, LPF, and Inv. MV) or computing second-order information such as the Hessian. In particular, Inv. MV involves locating boundary points via line search along worker directions, making it impractical for time-constrained training scenarios.

To address this limitation, we propose a relaxation of Inv. MV suitable for efficient, lightweight incorporation into the training objective. Instead of conducting the exhaustive boundary-point search, we directly treat current worker parameters as approximate boundary points, thus circumventing the costly line-search procedure altogether. This corresponds to adding the following regularization term to the objective in Equation~\ref{problem_form:objective}, scaled by the regularization coefficient $\lambda_{r}$: $\lambda_r \mathcal{R} = -\frac{\lambda_r}{M} \sum_{i=1}^M \| x_i - x_A \|_2$, where $x_A$ denotes the average of the worker parameters. Essentially, this is equivalent to increasing the consensus distance, contrary to the typical practice in distributed training  \citep{kong2021consensus}, which aims to reduce it. We then derive the update that arises due to the presence of $\mathcal{R}$ (the full derivation steps can be found in Section~\ref{appendix:subsec:mv_update_rule}):

\begin{align}
-\lambda_r \frac{\partial \mathcal{R}}{\partial x_m}
&\overset{(a)}{=} \frac{\lambda_r}{M^2} \left( 
    M\frac{x_m - x_A}{\|x_m - x_A\|} 
    - \sum_{j = 1}^{M} \frac{x_j - x_A}{\|x_j - x_A\|} 
\right) \notag \\ % remove numbering here
&\overset{(b)}{\approx} 
\lambda \frac{x_m - x_A}{\|x_m - x_A\|}. \label{eq:main:dppf_update}
\end{align}

Observe that in Equation~\ref{eq:main:dppf_update}, the expression on the right-hand side (RHS) of $(a)$ contains two terms within parentheses. The first term exerts a force pushing the worker away from the average point $x_A$. The second term is the negative of the averaged normalized worker directions, which approaches zero when the workers are symmetrically distributed around $x_A$. Specifically, under the assumption that each normalized vector $\frac{x_j - x_A}{\|x_j - x_A\|}$ is uniformly distributed on the unit sphere in a $d$-dimensional space, its expectation is exactly zero. Thus, we omit this term in practice. Furthermore, absorbing the constant factor $M$ in the denominator into $\lambda_r$ ($\lambda = \frac{\lambda_r}{M}$) simplifies the expression, yielding the approximation on the RHS of $(b)$ that exerts a unit-normed pushing force directly scaled with regularization coefficient $\lambda$.

Recall that the distributed consensus step involves a pull force, implemented as the convex combination of the consensus variable $x_C$ and worker parameters $x_m$: $x_m \leftarrow (1-\alpha) x_m + \alpha x_C$. When our regularization term is active, it introduces a push force that opposes this pull step. Hence, we refer to our method as \textit{Distributed Pull-Push Force} (DPPF). The push update is captured in Algorithm~\ref{alg:generic}. Notice that when $x_C = x_A$, the pull and push updates elegantly combine into a single, concise expression presented in Equation~\ref{eq:single_line}. 

\begin{equation}
\label{eq:single_line}
x_m \leftarrow x_m + (x_A - x_m) \left ( \alpha - \frac{\lambda}{\| x_m - x_A \|} \right )
\end{equation}

The term within parentheses on the right-hand side explicitly captures the interplay between the pull $(\alpha)$ and push $(\lambda)$ forces. This combined formulation efficiently applies both pull and push updates in a single step.

\section{Theoretical Analysis}
\label{sec:thm}

In this section, we analyze how DPPF’s pull and push forces determine the final distance between worker parameters and the global average, and we derive generalization guarantees tied to this distance. Consider $M$ workers, indexed by $m\in\{1,\dots,M\}$, running DPPF with communication period $\tau$ to minimize function $f$ collaboratively, i.e. \ref{problem_form:objective}. Let $x_{m,k}^+$ denote worker $m$’s parameters immediately after communication round $k$. We define the post-update average ($x_{A,k}^+$) and the corresponding gap vector $(\Delta_{m,k}^+)$ as
\begin{equation*}
x_{A,k}^+ := \frac{1}{M}\sum_{j=1}^M x_{j,k}^+,
\qquad
\Delta_{m,k}^+ := x_{A,k}^+ - x_{m,k}^+ .
\end{equation*}
Within each communication round $k$, workers perform $\tau$ local steps. We index these local iterations by $t\in\{1,\dots,\tau\}$ and denote as $x_{m,k}^t$ worker $m$’s parameters at local step $t$ after communication round $k$. We assume a homogeneous IID data regime: each worker draws mini-batches independently from the same distribution $\mathcal{D}$. Accordingly, the stochastic gradients $g^t_{m,k}$ are unbiased with bounded variance:

\begin{itemize}
    \item \( \mathbb{E}\bigl[ g_{m,k}^t\bigr | x_{m,k}^t] = \nabla f(x_{m,k}^t) \;\; \forall m,k,t \)
    \item \( \mathbb{E}\bigl [ \|g_{m,k}^t - \nabla f(x_{m,k}^t)\bigr\|^2 |  x_{m,k}^t] \le \sigma_0^2 \;\; \forall m,k,t \)
\end{itemize}

% \begin{equation*}
% \mathbb{E}\bigl[ g_{m,k}^t\bigr | x_{m,k}^t] = \nabla f(x_{m,k}^t),
% \;\;\;\;
% \mathbb{E}\bigl [ \|g_{m,k}^t - \nabla f(x_{m,k}^t)\bigr\|^2 |  x_{m,k}^t] \le \sigma_0^2 .
% \end{equation*}
% Within each communication round $k$, workers perform $\tau$ local steps. We index these local iterations by $t\in\{1,\dots,\tau\}$ and write $x_{m,k}^t$ for worker $m$’s parameters at local step $t$ after communication round $k$. We assume worker data shards are IID and see assume unbiased stochastic gradients ($g$) with bounded variance: for all $m,k,t$:
% \begin{equation*}
% \mathbb{E}\bigl[g_{m,k}^t\bigr] = \nabla f(x_{m,k}^t),
% \;\;\;\;
% \mathbb{E}\bigl\|g_{m,k}^t - \nabla f(x_{m,k}^t)\bigr\|^2 \le \sigma_0^2 .
% \end{equation*}

We now begin by showing how the tug‑of‑war between DPPF’s pull ($\alpha$) and push ($\lambda$) forces determines the asymptotic \emph{valley width}\footnote{We refer to a worker's distance to $x_A$ as valley width since DPPF treats workers as valley boundaries.}, defined as the Euclidean distance between each worker and the average, i.e. $\|\Delta_{m,k}^+\|$.  Theorem~\ref{main:thm:valley_width} proves that, under mild assumptions, this width concentrates around the ratio $\lambda/\alpha$ when the learning rate decays and the number of workers is sufficiently large.

\begin{theorem}[\textbf{DPPF: Final Valley Width}]
\label{main:thm:valley_width}
Consider $M$ workers running DPPF with pull strength $\alpha$, push strength
$\lambda$, communication period $\tau$, and local step size $\eta$. Let
$x_{m,k}^t$ denote worker $m$’s parameters at local iteration $t$ within round
$k$, and define the post-update gap $\Delta_{m,k}^+ := x_{A,k}^+ - x_{m,k}^+$.
Assume unbiased stochastic gradients with bounded variance:
\(
\mathbb{E}[g_{m,k}^t] = \nabla f(x_{m,k}^t)
\)
and
\(
\mathbb{E}\|g_{m,k}^t - \nabla f(x_{m,k}^t)\|^2 \le \sigma_0^2
\)
for all $m,k,t$. Then
\[
\lim_{k\to\infty}\mathbb{E}\bigl\|\Delta_{m,k}^{+}\bigr\|
   = \frac{\lambda}{\alpha} + \mathcal{O}\!\bigl(\eta\sigma_0 + M^{-1/2}\bigr),
\]
where the expectation is over the stochastic gradient noise. In particular,
letting $\eta\to0$ and $M\to\infty$ yields
\(
\lim_{k\to\infty}\mathbb{E}\|\Delta_{m,k}^{+}\| = {\lambda}/{\alpha}.
\)
\end{theorem}

Since $\lambda$ and $\alpha$ are user‑chosen hyperparameters, we can preset the target valley size. In other words, DPPF embeds prior information about how wide a basin it should seek simply through the force ratio $\lambda/\alpha$. Next, we translate this controllable valley width into a PAC‑Bayes guarantee.  % For a geometric grid of candidate widths, we attach isotropic Gaussian priors/posteriors whose variances scale with the valley radius and derive a unified bound on the expected test loss.

% Note that the push ($\lambda$) and pull ($\alpha$) strengths in DPPF are hyperparameters set prior to training. As shown in Theorem~\ref{main:thm:valley_width}, the ratio $\lambda / \alpha$ determines the final valley width. This implies that by selecting these values, we are implicitly injecting \textit{a priori} information about the desired valley size in the loss landscape! We then proceed to analyze how the change in valley width affects the PAC-Bayes generalization bound, in Theorem~\ref{main:thm:pac_bayes}. 

\begin{theorem}[\textbf{DPPF: Wider Valley Tightens the Generalization Gap}]
\label{main:thm:pac_bayes}
Consider a geometric grid for candidate valley sizes governed by the DPPF algorithm's pull ($\alpha_j$) and push ($\lambda_j$) strengths: $\mathcal G=\{r_j=r_{\min}(1+\gamma)^j\}_{j=0}^J$, where each $r_j=\frac{\lambda_j}{\alpha_j}$. For every $r_j$ assume the spherical‑Gaussian prior $P_{r_j}=\mathcal N(0,r_j\sigma_0^2I_d)$ and let the training algorithm return the posterior  $Q_{r_j}=\mathcal N(\mu_{r_j},c_{r_j}r_j\sigma_0^2I_d)$ over model parameters, where $c_{r_j}\ge1$ is a data-dependent scalar. Assume there are constants $D_0\!>\!0$ and $0\le\beta<1$ such that $\lVert\mu_{r_j}\rVert_2^2\le D_0\,r_j^{\beta}$ for every $r_j\in\mathcal G$. Then with probability $1-\delta$ over the draw of the sample set $S$ with $|S| = n$, for all $r_j\in\mathcal G$, we can write:
\begin{align*}
\mathbb{E}_{x \sim Q_{r_j}} [&L_{\mathcal{D}}(x)]
\le \mathbb{E}_{x \sim Q_{r_j}} [L_S(x)] \nonumber \\
&\quad + \underbrace{\sqrt{ \frac{ 
    \tfrac{d}{2}(c_{r_j} - 1 - \log c_{r_j})
    + \frac{D_0}{2 \sigma_0^2 r_j^{1 - \beta}}
    + \log \frac{nJ}{\delta}
}{2(n - 1)} }}_{\textstyle \text{gap}(r_j)}.
\end{align*}

because $1-\beta>0$, $\text{gap}(r_{j+1})<\text{gap}(r_j)$ for every consecutive pair in $\mathcal G$.
\end{theorem}

In Theorem~\ref{main:thm:pac_bayes}, the bound’s leading term, $\text{gap}(r_j)$, is strictly decreasing in $r_j$ because $1-\beta>0$. Hence, selecting a larger valley within the geometric grid of candidate valley sizes $\mathcal G$, achieved by choosing a valley with a higher $\lambda/\alpha$ ratio, provably reduces the PAC-Bayes generalization gap. Importantly, this result does not imply unbounded benefits from indefinitely enlarging the valley width. The improvement in generalization holds in the grid $\mathcal G$ of \citet{langford2001caruna} when the technical condition $\lVert\mu_{r_j}\rVert_2^2 \le D_0 r_j^{\beta}$ is satisfied. We provide empirical evidence supporting these theoretical assumptions and claims in Section~\ref{appendix:subsec:hyperparam_sens} of the Appendix. Taken together, Theorems~\ref{main:thm:valley_width} and \ref{main:thm:pac_bayes} link a targeted valley knob (the push‑to‑pull ratio) to an explicit statistical benefit, a tighter generalization bound.  We provide proofs of both theorems in Section~\ref{appendix:sec:all_theory} of the Appendix, which also includes insights into the arrangement of workers at the loss boundary and the non-convex convergence analysis.

\begin{table*}[t]
\centering
\caption{Comparison of DPPF with DDP SGD, LocalSGD and LocalSGD + QSR.}
\label{dppf:table:comm_eff}
\renewcommand{\arraystretch}{1.08}
\resizebox{0.85\textwidth}{!}{
\begin{tabular}{cc|ccccc|c}
                                                                                        &                        & \begin{tabular}[c]{@{}c@{}}ResNet-18\\ CIFAR-10\end{tabular} & \begin{tabular}[c]{@{}c@{}}PyramidNet\\ CIFAR-100\end{tabular} & \begin{tabular}[c]{@{}c@{}}ResNet-50\\ ImageNet\end{tabular} & \begin{tabular}[c]{@{}c@{}}ResNet-101\\ ImageNet\end{tabular} & \begin{tabular}[c]{@{}c@{}}ViT\\ ImageNet\end{tabular} & \begin{tabular}[c]{@{}c@{}}Comm. \\ (\%)\end{tabular} \\ \hline
\multicolumn{2}{c|}{DDP SGD / DDP AdamW}                                                                               & $4.33_{\pm0.08}$                                             & $19.10_{\pm0.06}$                                              & $23.83_{\pm0.17}$                                            & $23.33_{\pm0.09}$                                             & $33.61_{\pm0.07}$                                             & 100.0                                                 \\ \hline
\multirow{3}{*}{\begin{tabular}[c]{@{}c@{}}LocalSGD /\\ LocalAdamW\end{tabular}}         & $\tau=4$               & $4.36_{\pm0.06}$                                             & $19.03_{\pm0.28}$                                              & $24.11_{\pm0.04}$                                            & $22.11_{\pm0.17}$                                             & $33.91_{\pm0.30}$                                             & 25.0                                                  \\
                                                                                        & $\tau=8$               & $4.40_{\pm0.02}$                                             & $18.82_{\pm0.21}$                                              & $24.21_{\pm0.13}$                                            & $22.14_{\pm0.21}$                                             & $34.42_{\pm0.18}$                                             & 12.5                                                  \\
                                                                                        & $\tau=16$              & $4.49_{\pm0.21}$                                             & $19.11_{\pm0.21}$                                              & $24.36_{\pm0.19}$                                            & $21.72_{\pm0.13}$                                             & $34.81_{\pm0.22}$                                             & \textbf{6.3}                                          \\ \hline
\multirow{3}{*}{\begin{tabular}[c]{@{}c@{}}LocalSGD /\\ LocalAdamW \\ +QSR\end{tabular}} & $\tau_{\text{base}}=2$ & $4.21_{\pm0.08}$                                             & $17.77_{\pm0.06}$                                              & $23.68_{\pm0.02}$                                            & $21.66_{\pm0.22}$                                             & $33.69_{\pm0.18}$                                             & 42.8                                                  \\
                                                                                        & $\tau_{\text{base}}=4$ & $4.32_{\pm0.03}$                                             & $18.06_{\pm0.20}$                                              & $23.49_{\pm0.19}$                                            & $21.54_{\pm0.27}$                                             & $33.83_{\pm0.27}$                                             & 16.1                                                  \\
                                                                                        & $\tau_{\text{base}}=8$ & $4.29_{\pm0.04}$                                             & $18.27_{\pm0.26}$                                              & $23.52_{\pm0.21}$                                            & $21.73_{\pm0.21}$                                             & $34.45_{\pm0.31}$                                             & \textit{9.8}                                          \\ \hline
\multirow{3}{*}{DPPF}                                                                   & $\tau=4$               & $\mathbf{3.93_{\pm0.09}}$                                    & $\mathbf{16.74_{\pm0.03}}$                                     & $\mathbf{23.07_{\pm0.25}}$                                   & $21.10_{\pm0.06}$                                             & $\mathbf{31.35_{\pm0.30}}$                                    & 25.0                                                  \\
                                                                                        & $\tau=8$               & $\mathit{4.01_{\pm0.05}}$                                    & $\mathit{17.09_{\pm0.12}}$                                     & $\mathit{23.28_{\pm0.22}}$                                   & $\mathit{20.89_{\pm0.24}}$                                    & $\mathit{31.92_{\pm0.09}}$                                    & 12.5                                                  \\
                                                                                        & $\tau=16$              & $4.13_{\pm0.08}$                                             & $17.28_{\pm0.04}$                                              & $23.40_{\pm0.15}$                                            & $\mathbf{20.76_{\pm0.13}}$                                    & $32.28_{\pm0.26}$                                             & \textbf{6.3}                                         
\end{tabular}
}
\end{table*}

\section{Empirical Results}
\label{sec:Exp}
The empirical results presented here employ the increasing schedule for $\lambda$. An ablation study on different schedules is deferred to Section~\ref{appendix:subsec:scheduling_lambda} in the Appendix.
% Next we report our experimental results. This will be followed by ablation studies.

\subsection{Push Mechanism with Soft-Consensus}
\label{subsec:experiments_param_share}

\begin{table}[H]
\renewcommand{\arraystretch}{1.15}
\centering
\caption{Test performance of soft-consensus optimizers with and without DPPF. (NC: see Remark~\ref{rem:LSGD} in the Appendix).}
\label{table:param_sharing_improvement}
\resizebox{\columnwidth}{!}{
\begin{tabular}{c|cc|cc}
                        & \multicolumn{2}{c|}{CIFAR-10}                     & \multicolumn{2}{c}{CIFAR-100}                      \\ \cline{2-5} 
                        & \multicolumn{1}{c|}{4 Workers} & 8 Workers & \multicolumn{1}{c|}{4 Workers} & 8 Workers \\ \hline
SimpleAvg      & \multicolumn{1}{c|}{$4.24_{\pm 0.15}$}  & $4.31_{\pm 0.04}$ & \multicolumn{1}{c|}{$21.19_{\pm 0.25}$ }   & $21.44_{\pm 0.12}$    \\
\textbf{DPPF\textsubscript{SimpleAvg}} & \multicolumn{1}{c|}{$3.93_{\pm 0.09}$}  & $3.98_{\pm 0.10}$  & \multicolumn{1}{c|}{$20.59_{\pm 0.06}$}    & $20.77_{\pm 0.20}$  \\
EASGD          & \multicolumn{1}{c|}{$4.19_{\pm 0.21}$}                  & $4.21_{\pm 0.17}$ & \multicolumn{1}{c|}{$21.04_{\pm 0.22}$}  &  $21.42_{\pm 0.06}$  \\
\textbf{DPPF\textsubscript{EASGD}}     & \multicolumn{1}{c|}{$4.04_{\pm 0.08}$}                  & $3.97_{\pm 0.11}$  & \multicolumn{1}{c|}{$20.59_{\pm 0.29}$}                  &  $20.76_{\pm 0.17}$                  \\
LSGD           & \multicolumn{1}{c|}{$4.31_{\pm 0.03}$}                  & $4.33_{\pm 0.13}$ & \multicolumn{1}{c|}{$21.08_{\pm 0.34}$}     & $21.75_{\pm 0.43}$    \\
\textbf{DPPF\textsubscript{LSGD}}      & \multicolumn{1}{c|}{NC}  &   NC    & \multicolumn{1}{c|}{NC}  &   NC  \\
MGRAWA         & \multicolumn{1}{c|}{$4.35_{\pm 0.09}$}                  & $4.22_{\pm 0.11}$ & \multicolumn{1}{c|}{$21.31_{\pm 0.14}$}                  & $21.04_{\pm 0.14}$  \\
\textbf{DPPF\textsubscript{MGRAWA}}    & \multicolumn{1}{c|}{$4.03_{\pm 0.08}$}                  & $3.99_{\pm 0.10}$  & \multicolumn{1}{c|}{$20.46_{\pm 0.18}$}                  &    $20.87_{\pm 0.24}$               
\end{tabular}}
\end{table}

We integrate the proposed pull-push framework into existing soft-consensus distributed optimizers: EASGD, LSGD, and MGRAWA. These methods inherently apply a pulling force, guiding workers toward a consensus variable $x_C$. We also introduce \textit{SimpleAvg}, which sets $x_C = x_A$, as a soft-consensus variant of LocalSGD~\citep{stich2019local}. We refer to SimpleAvg with the pull-push mechanism as DPPF\textsubscript{SimpleAvg}, and adopt the same naming convention when incorporating the push force into EASGD, LSGD, and MGRAWA. To see how much improvement in generalization we attain by introducing the wide minima seeking pushing force, we train ResNet-18 \citep{resnet} models on CIFAR-10 \citep{cifar10} and CIFAR-100 \citep{cifar100}  datasets using the vanilla distributed trainers and their DPPF variants. The experiments are conducted on $4$ and $8$ GPUs to assess the DPPF's scalability. For the vanilla methods, we vary the pulling strength as follows: $\alpha \in \{0.05, 0.1, 0.3, 0.5\}$, and for DPPF variants, $\alpha=0.1$ is fixed and the pushing force is varied as follows: $\lambda \in \{0.05, 0.1, 0.25, 0.5, 0.75\}$. More details can be found in Section~\ref{appendix:subsec:dppf_vs_soft_consensus} in the Appendix. The experiments are repeated for three different seeds and we report the test error as the average across the seeds, specifying the standard deviation in the subscript. (NC: not converged)

As shown in Table~\ref{table:param_sharing_improvement}, the push mechanism brings considerable improvements over the vanilla distributed optimizers that only have the pulling force. DPPF lowers test error by up to 0.3\% on CIFAR‑10 and 0.7\% on CIFAR‑100 experiments respectively.  Among the DPPF variants that converged, we do not observe a significant difference in performance. Therefore, we use DPPF\textsubscript{SimpleAvg} moving forward to be able to execute pull and push updates in a single step, as mentioned in Equation~\ref{eq:single_line} and refer to it simply as DPPF.

 \subsection{Comparison with Other Communication-Efficient Methods}
\label{sec:dppf:comp_other_comm_eff}
In this section, we compare \textit{DPPF} with other communication-efficient methods, namely LocalSGD~\citep{stich2019local} and LocalSGD combined with QSR~\citep{guquadratic}, which, to the best of our knowledge, represents the current state of the art approach. QSR proposes increasing the communication period \(\tau\) throughout training to be inversely proportional to the squared learning rate~\citep{guquadratic}. Specifically, the authors schedule the communication period as \(\tau_t = \max \left\{ \tau_{\text{base}}, \left\lfloor \left( \beta/\eta_t \right)^2 \right\rfloor \right\},\) where \(\tau_{\text{base}}\) is the minimum number of local steps taken before QSR is applied, and \(\beta\) is a growth coefficient that controls how aggressively the communication period \(\tau_t\) is updated based on changes in the learning rate \(\eta_t\). In~\citep{guquadratic}, the authors experiment with \(\tau_{\text{base}} \in \{2, 4, 8\}\) and \(\beta \in \{0.2, 0.25, 0.3\}\), and keep \(\eta_{\text{max}} = 0.8\). To reproduce comparable effects in our own settings, we scale the \(\beta\) values accordingly.  Further training details are provided in Section~\ref{appendix:subsec:comm_eff} of the Appendix. 

Table~\ref{dppf:table:comm_eff} compares the test error and communication volumes of DPPF against LocalSGD and LocalSGD+QSR, across a diverse set of models (ResNet-\{18, 50, 101\} \citep{resnet}, PyramidNet(110,270) \citep{pyramidnet}, Vision Transformer (ViT) \citep{dosovitskiy2020image} with 12 layers) and datasets (CIFAR-10 \citep{cifar10}, CIFAR-100 \citep{cifar100}, ImageNet \citep{imagenet}). We also report the communication volume as the percentage of communication rounds relative to DDP SGD. We observe that DPPF consistently achieves the lowest test errors at dramatically reduced communication cost. The improvements are especially pronounced on ImageNet. For example, with ResNet-50, DPPF achieves a test error of 23.07\% at $\tau=4$, outperforming all baselines, including DDP SGD (23.83\%), while reducing communication by 4$\times$. The gains are even more substantial for deeper models: on ImageNet with ResNet-101, DPPF achieves 21.10\% error, substantially better than DDP SGD (23.33\%) and all local baselines, again using only 25\% of the communication. Even at more aggressive communication intervals ($\tau=8,16$), DPPF maintains a clear advantage, achieving 20.89\% and 20.76\% error, respectively, as communication drops to just 12.5\% and 6.3\%. For the ViT, DPPF also delivers consistent improvements. At $\tau=4$ it achieves 31.35\% error versus 33.61\% for DDP AdamW, reducing communication by 4$\times$, and it continues to outperform all local baselines at $\tau=8,16$. These results further demonstrate that DPPF is effective not only with CNNs but also with transformer architectures, and is compatible with AdamW training. On CIFAR-10 and CIFAR-100 benchmarks, DPPF also outperforms the baselines while operating at lower communication cost. Overall, DPPF offers substantial gains in both accuracy and communication efficiency across model and data scales, as further shown in Figure~\ref{fig:error_vs_comm_volume}.

% We observe that DPPF consistently achieves the lowest test errors across all settings while operating at reduced communication costs. For instance, with \(\tau=4\), DPPF attains a CIFAR-10 test error of 3.93\%, outperforming all baselines—including DDP SGD (4.40\%)—despite using only 25\% of the communications. Similarly, on CIFAR-100 with PyramidNet, DPPF achieves 16.74\% test error, again surpassing the rest, including the best LocalSGD+QSR setting, relying on 42.8\% communication cost. On ImageNet with ResNet-50, DPPF achieves 23.07\% test error at \(\tau=4\), notably better than other communication-efficient methods, and DDP SGD (23.83\%) while using a 4x lower than it. DPPF also scales gracefully to longer communication intervals. With \(\tau=\{8,16\}\), it still outperforms most baselines while reducing communication volume to 12.5\% and 6.3\%, respectively. These results collectively demonstrate that DPPF offers a gain on both generalization and communication efficiency, as also illustrated in Figure~\ref{fig:error_vs_comm_volume}.

\subsection{SAM-like Performance While Being Communication-Efficient}

\begin{table}[h]
\renewcommand{\arraystretch}{1.01}
\centering
\caption{Test errors(\%) reported by employing four different flatness encouraging schemes in distributed optimization.}
\label{tab:dppf_vs_sam}
\resizebox{1.01\columnwidth}{!}{
\begin{tabular}{ccc|cc|cc}
\multicolumn{1}{l}{}                                                           &                                                                               &                & \begin{tabular}[c]{@{}c@{}}DDP\\ SGD\end{tabular} & \begin{tabular}[c]{@{}c@{}}DPPF\\ SGD\end{tabular} & \begin{tabular}[c]{@{}c@{}}DDP\\ SAM\end{tabular} & \begin{tabular}[c]{@{}c@{}}DPPF\\ SAM\end{tabular} \\ \cline{4-7} 
\multicolumn{1}{c|}{\multirow{6}{*}{\rotatebox[origin=c]{90}{4 Workers}}} & \multicolumn{1}{c|}{\multirow{3}{*}{\rotatebox[origin=c]{90}{C10}}}  & ResNet-18      & $4.33_{\pm 0.08}$                                          & $\mathbf{3.93_{\pm 0.09}}$                                  & $3.97_{\pm 0.08}$                                          & $\mathbf{3.74_{\pm 0.05}}$                                  \\
\multicolumn{1}{c|}{}                                                          & \multicolumn{1}{c|}{}                                                         & WRN-16x8       & $4.09_{\pm 0.10}$                                          & $\mathbf{3.76_{\pm 0.06}}$                                  & $3.72_{\pm 0.08}$                                          & $\mathbf{3.70_{\pm 0.05}}$                                  \\
\multicolumn{1}{c|}{}                                                          & \multicolumn{1}{c|}{}                                                         & PyNet(110,270) & $3.94_{\pm 0.03}$                                          & $\mathbf{3.11_{\pm 0.04}}$                                  & $\mathbf{2.84_{\pm 0.10}}$                                 & $2.89_{\pm 0.07}$                                           \\ \cline{2-7} 
\multicolumn{1}{c|}{}                                                          & \multicolumn{1}{c|}{\multirow{3}{*}{\rotatebox[origin=c]{90}{C100}}} & ResNet-18      & $21.29_{\pm 0.23}$                                         & $\mathbf{20.59_{\pm 0.06}}$                                 & $20.82_{\pm 0.14}$                                         & $\mathbf{20.53_{\pm 0.13}}$                                 \\
\multicolumn{1}{c|}{}                                                          & \multicolumn{1}{c|}{}                                                         & WRN-16x8       & $20.10_{\pm 0.21}$                                         & $\mathbf{18.99_{\pm 0.09}}$                                 & $19.36_{\pm 0.06}$                                         & $\mathbf{18.96_{\pm 0.18}}$                                 \\
\multicolumn{1}{c|}{}                                                          & \multicolumn{1}{c|}{}                                                         & PyNet(110,270) & $19.18_{\pm 0.10}$                                         & $\mathbf{16.87_{\pm 0.18}}$                                 & $16.50_{\pm 0.19}$                                         & $\mathbf{15.68_{\pm 0.09}}$                                 \\ \hline
\multicolumn{1}{c|}{\multirow{6}{*}{\rotatebox[origin=c]{90}{8 Workers}}} & \multicolumn{1}{c|}{\multirow{3}{*}{\rotatebox[origin=c]{90}{C10}}}  & ResNet-18      & $4.67_{\pm 0.17}$                                          & $\mathbf{3.98_{\pm 0.10}}$                                  & $4.23_{\pm 0.07}$                                          & $\mathbf{3.92_{\pm 0.11}}$                                  \\
\multicolumn{1}{c|}{}                                                          & \multicolumn{1}{c|}{}                                                         & WRN-16x8       & $4.22_{\pm 0.08}$                                          & $\mathbf{3.78_{\pm 0.12}}$                                  & $3.79_{\pm 0.14}$                                          & $\mathbf{3.79_{\pm 0.10}}$                                  \\
\multicolumn{1}{c|}{}                                                          & \multicolumn{1}{c|}{}                                                         & PyNet(110,270) & $4.08_{\pm 0.18}$                                          & $\mathbf{3.18_{\pm 0.07}}$                                  & $3.06_{\pm 0.05}$                                          & $\mathbf{2.93_{\pm 0.07}}$                                  \\ \cline{2-7} 
\multicolumn{1}{c|}{}                                                          & \multicolumn{1}{c|}{\multirow{3}{*}{\rotatebox[origin=c]{90}{C100}}} & ResNet-18      & $21.26_{\pm 0.24}$                                         & $\mathbf{20.77_{\pm 0.20}}$                                 & $21.14_{\pm 0.14}$                                         & $\mathbf{20.60_{\pm 0.12}}$                                 \\
\multicolumn{1}{c|}{}                                                          & \multicolumn{1}{c|}{}                                                         & WRN-16x8       & $20.58_{\pm 0.29}$                                         & $\mathbf{18.90_{\pm 0.16}}$                                 & $19.87_{\pm 0.15}$                                         & $\mathbf{19.22_{\pm 0.13}}$                                 \\
\multicolumn{1}{c|}{}                                                          & \multicolumn{1}{c|}{}                                                         & PyNet(110,270) & $19.10_{\pm 0.06}$                                         & $\mathbf{16.74_{\pm 0.03}}$                                 & $16.68_{\pm 0.03}$                                         & $\mathbf{15.94_{\pm 0.24}}$                                
\end{tabular}
}
\end{table}

\begin{figure*}[h]
  \centering
  % First group of two subfigures (left figure)
  \begin{minipage}[t]{0.65\textwidth}
    \centering
    \begin{subfigure}[b]{0.47\textwidth}
      \includegraphics[width=\linewidth]{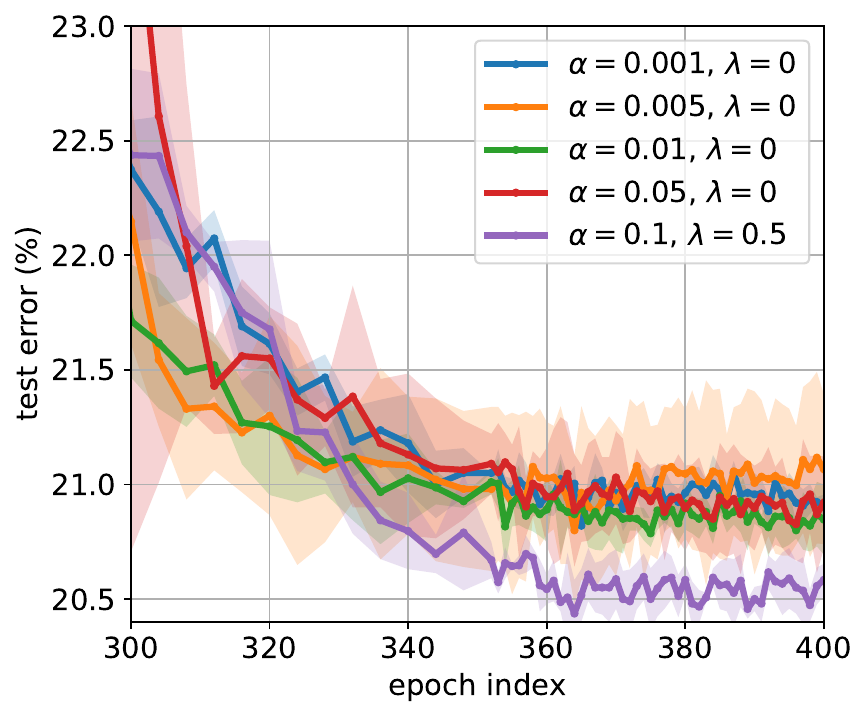}
        \caption{Test error curves}
        \label{fig:ablation_weak_lambda_1}
    \end{subfigure}
    \hspace{3mm}
    \begin{subfigure}[b]{0.47\textwidth}
      \includegraphics[width=\linewidth]{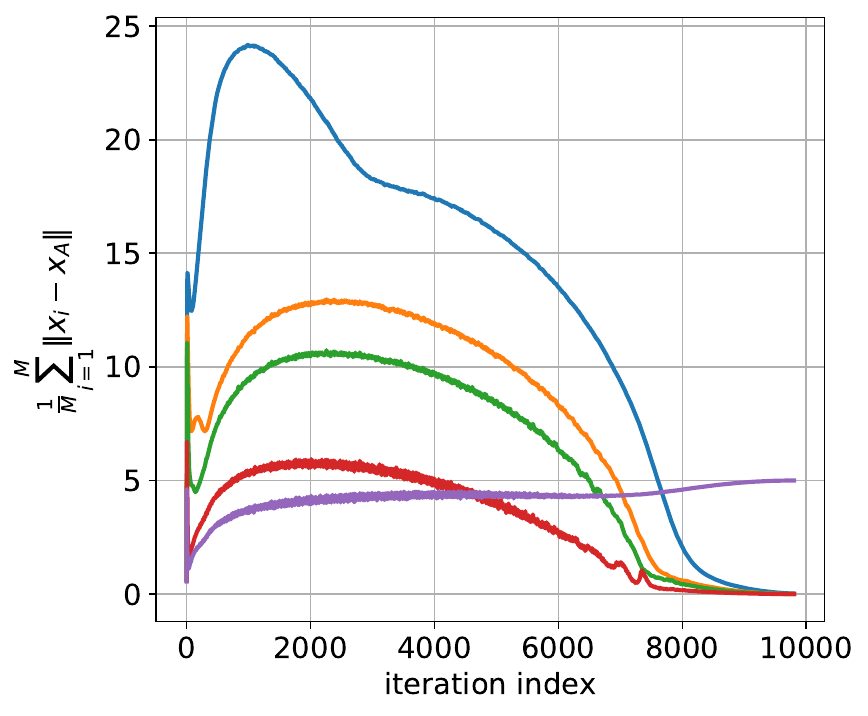}
      \caption{Consensus distance}
      \label{fig:ablation_weak_lambda_2}
    \end{subfigure}
    \caption{Comparison of training with weaker pulling force and using the pull-push mechanism.}
    \label{fig:ablation_weak_lambda_total_figure}
  \end{minipage}
  \hspace{3mm}
  % Second group of two subfigures (right figure)
  \begin{minipage}[t]{0.32\textwidth}
    \centering
    \begin{subfigure}[b]{0.9\textwidth}
      \includegraphics[width=\linewidth]{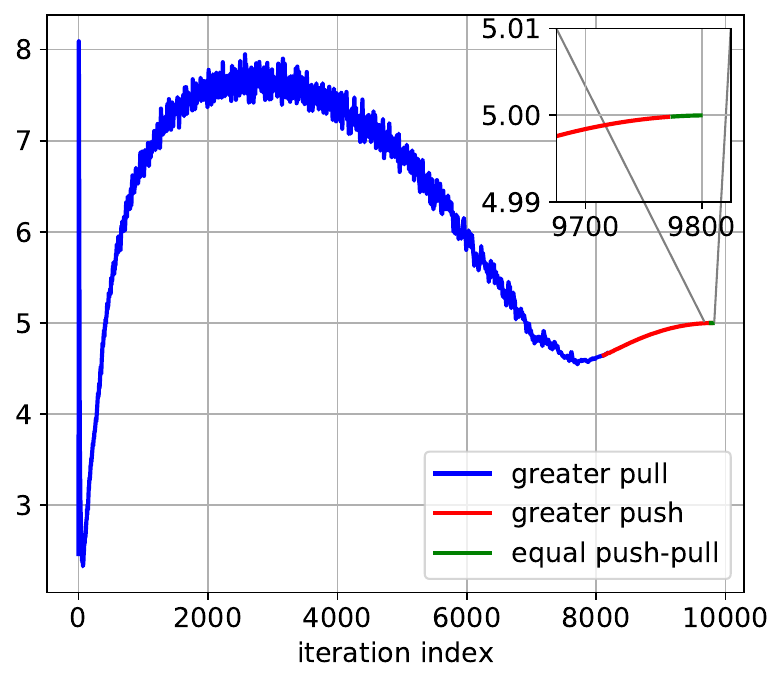}
        \caption{Entire training + zoomed}
        \label{fig:ablation_push_pull_overall}
    \end{subfigure}
        \caption{Analysis of the tug-of-war between pull and push forces.}
        \label{fig:ablation_push_pull_figure}
  \end{minipage}
\end{figure*}

In this section, we evaluate DPPF’s ability to seek flat minima at the distributed level by comparing it with SAM. As a baseline, we include DDP SGD, which uses standard gradient averaging with SGD as the local optimizer. Unlike SAM, which promotes flatness locally, DPPF enforces wide, high-quality minima through inter-worker repulsion. To compare local versus distributed flatness strategies, we include DDP SAM, which applies SAM locally with standard distributed gradient averaging. Finally, to evaluate the benefit of combining both local and distributed flatness-promoting mechanisms, we introduce DPPF SAM, which uses DPPF for distributed training and SAM as the local optimizer.

To compare these methods, we train three different models, namely ResNet-18 \citep{resnet}, WideResNet-16x8 (WRN-16x8) \citep{wideresnet} and PyramidNet-(110,270) (PyNet(110,270)) \citep{pyramidnet} on CIFAR-10 (C10) and CIFAR-100 (C100) datasets. We run the experiments using $4$ and $8$ workers. Configurations with SGD optimizer are run for $400$ epochs while those with SAM are run for $200$ epochs for computational parity, as suggested in \citep{SAM}. The SAM's maximization hyperparameter is varied as $\rho \in \{0.05, 0.1, 0.2\}$. For DPPF, the communication period is fixed at $\tau=4$. More details can be found in Section~\ref{appendix:subsec:dppf_sgd_sam}. Conclusions from Table~\ref{tab:dppf_vs_sam} are consistent with that of Table~\ref{dppf:table:comm_eff}: Our distributed flatness-promoting method, \textit{DPPF}, outperforms standard DDP in generalization while retaining communication efficiency. Notably, DPPF SGD alone matches or surpasses DDP SAM in many cases, predominantly in CIFAR-100 experiments. Moreover, combining local and distributed flatness mechanisms (DPPF SAM) yields further improvements in generalization.

\section{Ablation Studies}
\label{sec:ablations}
\subsection{Analysis of the Pull-Push Mechanism}
\label{subsec:ablation_1}
Instead of introducing a pushing force to recover wide valleys, as done in DPPF, one might ask: why not simply reduce the pulling strength to keep workers apart? However, intuitive physical reasoning suggests the two approaches are not equivalent, as having only a one-directional merging force inevitably leads to worker collapse. To test this, we conduct experiments using SimpleAvg without a pushing force and vary the pull strength across $\alpha = \{0.0001, 0.005, 0.01, 0.05\}$. We compare these results with DPPF using $\alpha = 0.1$ and $\lambda=0.5$, as reported in previous tables. All experiments are run with three random seeds. As shown in Figure~\ref{fig:ablation_weak_lambda_1}, the wide-minima-seeking behavior enabled by DPPF’s pushing force cannot be replicated by merely weakening the pull. To further analyze this, we track the mean distance between workers and the average variable (consensus distance); equivalently, the relaxed MV measure defined in Section~\ref{sec:dppf}; over the course of training (Figure~\ref{fig:ablation_weak_lambda_2}). The results show that, regardless of pulling strength, workers without a push force steadily collapse toward one another. This effect is consistent across SimpleAvg, EASGD, LSGD, and MGRAWA, and it restricts exploration of the loss landscape toward the end of training. We refer to this phenomenon as \textit{valley collapse}, and our findings show that the presence of a push force is essential to preventing it. The valley collapse phenomenon is also visually evident in Figure~\ref{2d_train_error_vis_simpleavg}.

Finally, we analyze the interplay between the pull and push forces, identifying the phases of training where the pulling force dominates and those where the pushing force prevails. We take one of the workers and highlight these phases as we plot the change of the worker's distance to the average variable throughout the training. The results are shared in Figure~\ref{fig:ablation_push_pull_figure}. We see that the pulling force is stronger than the pushing force in the earlier phase of the training, yet the workers drift away from the average variable. In this phase, the exploration of the workers is propelled by the underlying local optimizer (SGD or variant), and the learning rate hasn't decayed significantly yet to prevent the exploration. Close to convergence, the pushing force starts to overwhelm the pulling force, preventing the valley collapse and encouraging wider minima. At the end of the training, the interplay between the pull-push force is stabilized around $5$, when $\lambda=0.5$ and $\alpha=0.1$, aligned with our Theorem~\ref{main:thm:valley_width}.

% \begin{remark}{(Limitation of DPPF: Hyperparameters)} Although DPPF has two tunable hyperparameters, the pull strength ($\alpha$) and the push strength ($\lambda$), our ablation studies across various experimental settings show that the method consistently improves final performance over a wide range of $(\alpha, \lambda)$ values. Additional details are provided in Section~\ref{appendix:subsec:hyperparam_sens} of the Appendix. \label{remark:lim_hyperparam}
% \end{remark}

\begin{figure}[H]
  \centering
  \begin{subfigure}[t]{0.47\columnwidth}
      \centering
      \includegraphics[width=\linewidth]{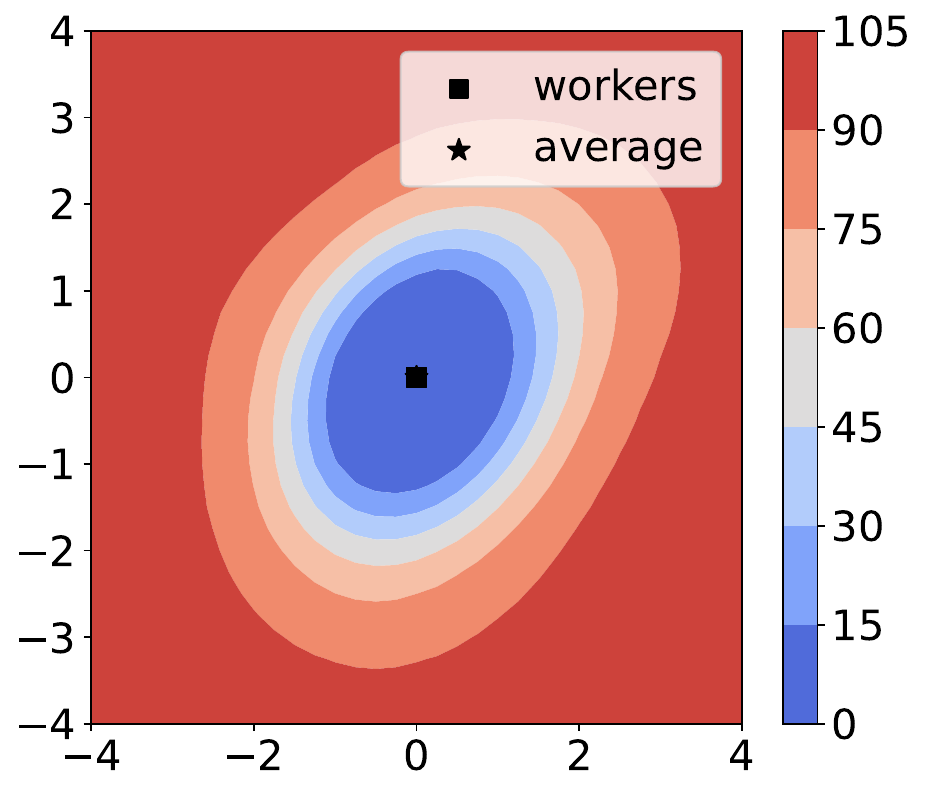}
      \caption{SimpleAvg}
      \label{2d_train_error_vis_simpleavg}
  \end{subfigure}
  \hfill
  \begin{subfigure}[t]{0.49\columnwidth}
      \centering
      \includegraphics[width=\linewidth]{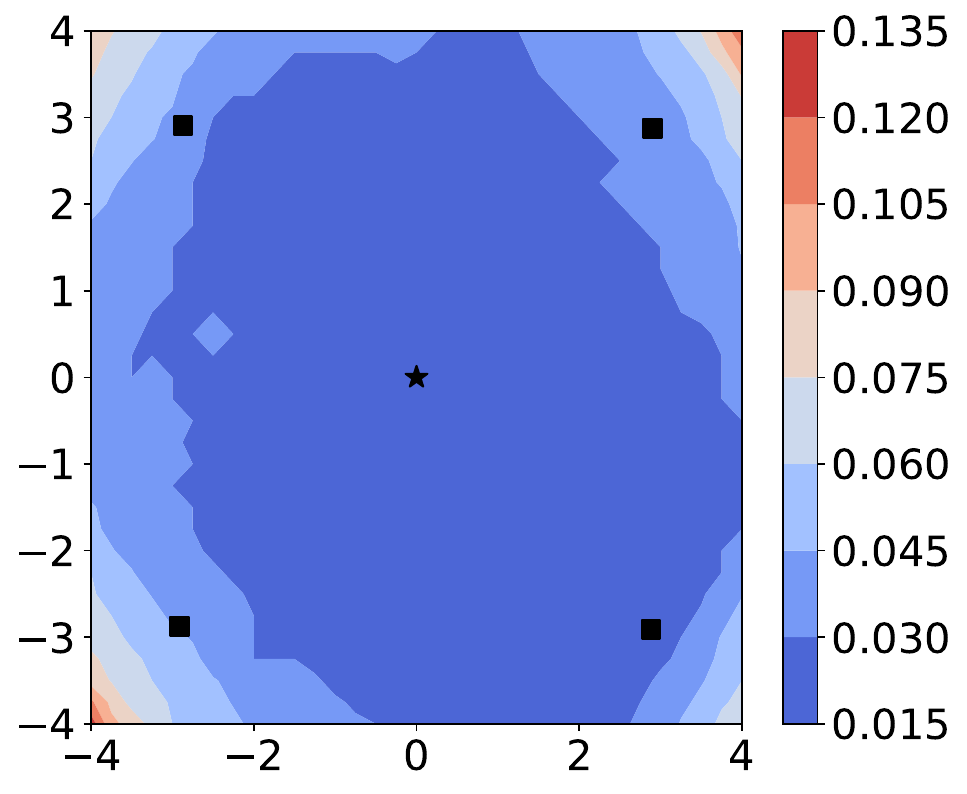}
      \caption{DPPF\textsubscript{SimpleAvg}}
      \label{2d_train_error_vis_dppf}
  \end{subfigure}
  \caption{2D contour plots of training error (\%) around the average variable ($x_A$). For SimpleAVG, the error quickly rises to 100\%. In contrast, the error remains stable at approximately 0.03\% around the minima found by DPPF.}
  \label{2d_train_error_vis}
\end{figure}
\begin{figure}[H]
  \centering
  \begin{subfigure}[t]{0.48\columnwidth}
      \centering
      \includegraphics[width=\linewidth]{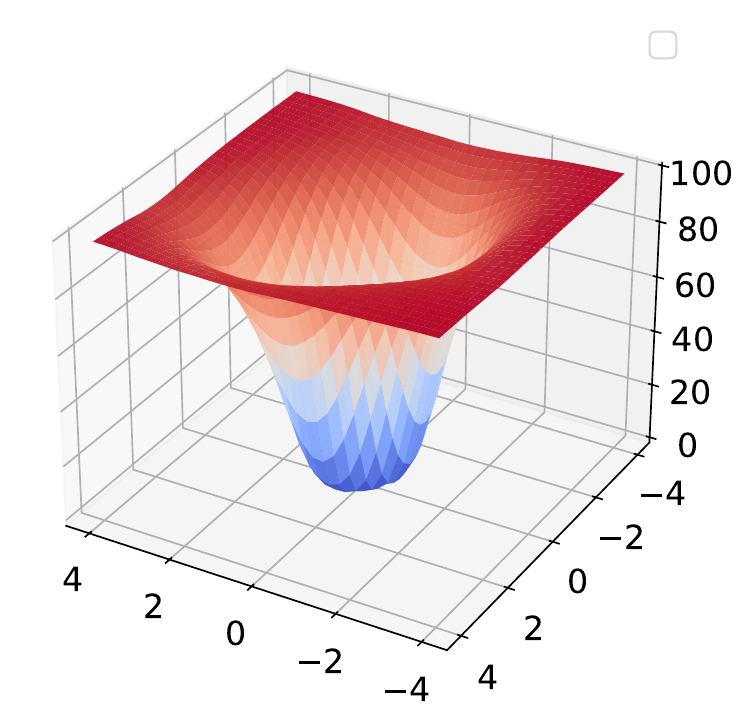}
      \caption{SimpleAvg}
      \label{3d_train_error_vis_simpleavg}
  \end{subfigure}
  \hfill
  \begin{subfigure}[t]{0.48\columnwidth}
      \centering
      \includegraphics[width=\linewidth]{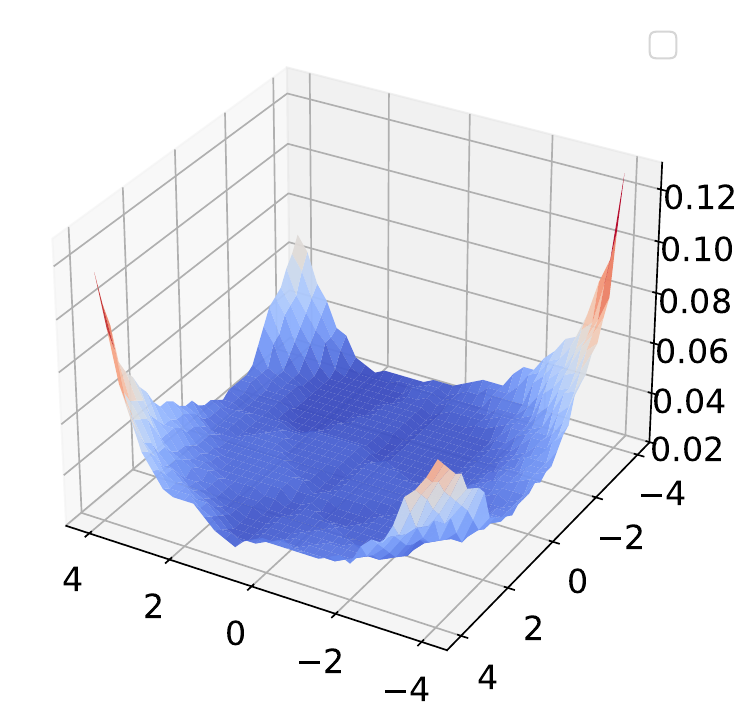}
      \caption{DPPF\textsubscript{SimpleAvg}}
      \label{3d_train_error_vis_dppf}
  \end{subfigure}
  \caption{3D train error (\%) plots.}
  \label{3d_train_error_vis}
\end{figure}

\subsection{Error Landscape Visualizations}

To evaluate the effectiveness of DPPF’s pushing mechanism in locating wide minima, we visualize how train and test errors evolve around the minima found by SimpleAvg and DPPF\textsubscript{SimpleAvg} after training ResNet-18 on the CIFAR-100 dataset. Specifically, Figure~\ref{2d_train_error_vis} shows 2D contour maps of the training error landscape, with worker locations marked. As shown in Figure~\ref{2d_train_error_vis_simpleavg}, the training error around the final point found by SimpleAvg rapidly rises to $100\%$, and the workers collapse onto the average variable, illustrating the valley collapse phenomenon discussed in Section~\ref{subsec:ablation_1}. In contrast, DPPF’s pushing mechanism keeps the workers apart and ultimately finds a wide basin stably spanned by the workers, where the training error remains around $0.03\%$ (Figure~\ref{2d_train_error_vis_dppf}). The corresponding 3D visualizations are provided in Figure~\ref{3d_train_error_vis}, where the broader basin found by DPPF is even more evident. In these two settings, the final test errors of SimpleAvg and DPPF\textsubscript{SimpleAvg} are $21.50\%$ and $20.56\%$, respectively. Additional landscape visualizations can be found in Section~\ref{appendix:subsec:landscape_vis} of the appendix.

\subsection{DPPF in Non-IID Setting}
Though we mainly focus on the IID setting in this work, here we ask whether the effectiveness of DPPF carries also to the non-IID settings. For this, we couple the push mechanism with a well-established Federated Learning (FL) method SCAFFOLD \citep{karimireddy2020scaffold} and a more recent, flat-minima seeking method FedLESAM \citep{fedlesam}. Following the standard heterogeneous data construction protocol in the FL literature~\citep{karimireddy2020scaffold,fedlesam,fedsam}, we partition the training data across workers using a Dirichlet distribution (Dir.) at initialization and keep the partition fixed throughout training (i.e., no reshuffling across epochs). We consider $M=4$ workers with communication period $\tau=16$. Non-IID splits are generated with Dirichlet concentration parameters $\alpha \in \{0.1, 0.6\}$, where smaller $\alpha$ induces stronger heterogeneity. We train ResNet-18 on CIFAR-10 and PyramidNet on CIFAR-100 under these heterogeneous settings. More details are provided in Section~\ref{appendix:subsec:noniid_exps} of the Appendix.

\begin{table}[h!]
\centering
\caption{Test errors ($\%$) in the non-IID settings}
\label{tab:noniid}
\resizebox{0.96\columnwidth}{!}{
\begin{tabular}{c|cc|cc}
         & \multicolumn{2}{c|}{ResNet-18 CIFAR-10} & \multicolumn{2}{c}{PyramidNet CIFAR-100} \\ \hline
Method   & Dir. 0.1    & Dir. 0.6   & Dir. 0.1    & Dir. 0.6    \\ \hline
SCAFFOLD &   $10.40_{\pm{0.05}}$   &  $6.52_{\pm{0.12}}$    &  $24.08_{\pm{0.09}}$   &   $19.95_{\pm{0.07}}$ \\
DPPF\textsubscript{SCAFFOLD}   &  $\mathbf{8.76_{\pm{0.26}}}$ &  $\mathbf{5.95_{\pm{0.08}}}$   &  $\mathbf{23.11_{\pm{0.17}}}$ &  $\mathbf{18.69_{\pm{0.22}}}$   \\ \hline
FedLESAM &  $10.63_{\pm{0.34}}$   &  $6.71_{\pm{0.20}}$  &  $23.89_{\pm{0.41}}$   &  $20.00_{\pm{0.15}}$   \\
DPPF\textsubscript{FedLESAM}    &  $\mathbf{9.08_{\pm{0.10}}}$  &  $\mathbf{6.40_{\pm{0.11}}}$    &  $\mathbf{23.36_{\pm{0.05}}}$  &   $\mathbf{19.59_{\pm{0.10}}}$
\end{tabular}
}
\end{table}

We report the test errors in Table~\ref{tab:noniid}. Across all model–dataset combinations and heterogeneity levels, incorporating DPPF consistently reduces test error relative to the base FL methods. Notably, the improvement is more pronounced under stronger heterogeneity (Dir. 0.1), suggesting that the push mechanism becomes increasingly beneficial in encouraging wider minima that can accommodate diverse worker-specific solutions. Furthermore, the gains are larger when DPPF is combined with SCAFFOLD compared to FedLESAM. While SCAFFOLD corrects client drift through control variates, it does not explicitly regularize curvature. In contrast, DPPF introduces a push force along disagreement directions, preventing synchronized convergence into sharp solutions. This complementary interaction could explain the larger performance gains.
% \begin{remark}{(Limitation of DPPF: Memory Usage)} Similar to the soft‑consensus methods, DPPF requires each worker to have an extra parameter vector for computing the average variable. Because this vector is needed only during communication rounds, it can be offloaded to the CPU between communications, easing GPU‑memory pressure at the expense of minor I/O overhead. \label{remark:lim_memory}
% \end{remark}

\section{Final Remarks}
\label{sec:conclusion}

\textbf{Limitations.} DPPF has two limitations. First, it introduces two hyperparameters: the pull strength ($\alpha$) and push strength ($\lambda$). Nonetheless, our ablations show consistent gains across a broad range of $(\alpha,\lambda)$ (Appendix~\ref{appendix:subsec:hyperparam_sens}). Second, like soft-consensus methods, DPPF requires an extra parameter vector per worker to compute the average variable. Since this vector is only needed at communication rounds, it can be offloaded to CPU memory between communications, reducing GPU pressure at the cost of minor I/O overhead.

\textbf{Conclusion.} In this work, we address the performance limitations of communication-efficient distributed training methods by proposing DPPF, a training strategy that introduces a push mechanism to prevent worker entrapment in sharp valleys of the loss landscape and to collaboratively recover flat minima. DPPF outperforms both the synchronous gradient-averaging method and other baselines by a considerable margin, while attaining comparable performance to SAM. Future work could apply the push mechanism to ensemble learning and extending DPPF to federated and self-supervised settings presents promising research directions for general-purpose, scalable deep learning.

\section{Acknowledgement}
The authors acknowledge that the NSF Award \#2041872 sponsored the research in this paper.

% References
\bibliography{references}
\newpage

\onecolumn

\title{Communication-Efficient Distributed Training for Collaborative Flat Optima Recovery in Deep Learning (Supplementary Material)}
\maketitle
\appendix

\input{supplement}

\input{supplement_landscapes}

\end{document}

%% file: supplement.tex
\section{Additional Related Works}
\label{appendix:add_related_works}

\textbf{Data Parallel Training of DNNs} Local SGD has been shown to match the variance reduction properties and convergence rate of DDP SGD, achieving linear speedups with the number of workers \citep{stich2019local}. In the conventional LocalSGD method, the communication period \(\tau\) is set at the beginning of training and kept constant throughout. A series of theoretical works has shown that the convergence rate of LocalSGD is inversely proportional to the communication period \(\tau\)~\citep{whyandwhenlocalsgd, stich2019local, yu2019parallel}, highlighting a trade-off between communication efficiency and model performance. This theoretical result has also been empirically validated in a study~\citep{ortiz2021trade}. Another notable work \citep{Gupta2020Stochastic} proposes a hybrid approach and suggests performing All-Reduce SGD until a certain level of performance is reached and then switching to LocalSGD for final model averaging for improved performance. However, \citep{ortiz2021trade} reports that the final performance highly depends on the transition point between the two training practices and usually, the optimal switching is late in the total training. Prior work has also focused on optimizing the communication period by incorporating adaptivity. \citep{kamp2014communication} proposed a policy to skip or perform a consensus update in local communication based on the worker parameter variance. \citep{haddadpour2019local} suggests linearly increasing the communication period as the training progresses. \citep{shen2021stl} proposes a scheme in which the communication period is doubled whenever a two-fold learning rate drop is applied based on the predefined milestones. On the contrary, \citep{wang2019adaptive} recommends starting with a high communication period and gradually decreasing toward the end of training.  Although increasing the communication period usually hurts optimization, some recent works theoretically showed the benefits of increased communication periods in terms of generalization.

\textbf{Flat Minima and Generalization} The concept of seeking flat minima dates back to \citep{hochreiter1994simplifying}. Since then, numerous studies have investigated the geometry of the loss landscapes of DNNs \citep{goodfellow2014qualitatively, dauphin2014identifying, baldassi2015subdominant, li2018visualizing} and the relationship between the flat minima and the generalization ability of the model \citep{keskar2016large, jastrzkebski2017three, andriushchenko2023modern}. One early approach \citep{keskar2016large} develops a sharpness metric that quantifies the maximum fluctuation in the loss value under controlled weight perturbations. The authors show that the minima from large-batch training are sharper and generalize worse than those from small-batch training. Another approach \citep{neyshabur2017exploring} opposes the definition of the sharpness metric proposed in \citep{keskar2016large}, criticizing it for falling short on accurately reflecting loss sharpness. They experiment with several norm-based sharpness measures that correlate well with generalization. Furthermore, \citep{jastrzkebski2017three} characterizes how learning rate and batch size influence the final minima. A comprehensive study examining over $40$ complexity measures from theoretical bounds and empirical studies, along with their causal relationship to generalization, is provided in \citep{fantastic_measures}. \citep{dinh2017sharp} points out that one needs to be careful in designing measures to quantify sharpness for DNNs with ReLU activations, as simple reparameterizations can alter the geometry of the landscape. Finally, \citep{andriushchenko2023modern} raises concerns about the validity of flatness measures in the modern DL. 

\textbf{Parameter Averaging} Parameter averaging is also a widely adopted practice in Federated Learning (FL) \citep{mcmahan2017communication}, where clients collaboratively train a DL model through a central server without sharing their local data due to privacy constraints.  However, FL faces challenges arising from data heterogeneity and the limited participation of clients in global updates, which can degrade performance. Early approaches~\citep{karimireddy2020scaffold, acarfederated, fan2022fedskip} introduced statistical mechanisms to mitigate discrepancies between local updates and the global objective. More recent works~ \citep{fedsam, FedGAMMA} have incorporated flat-minima-seeking optimization strategies into client objectives and locally estimated the sharpness of the global objective to improve generalization~\citep{fedsmoo, fedlesam}. Beyond FL, model averaging is also leveraged in ensemble learning, where multiple models are trained independently on the full dataset while periodically sharing their weights to enhance robustness and performance~\citep{jolicoeurpopulation, fournier2024wash}.

\section{Details of the Main Experiments}
\label{appendix:sec:details_of_experiments}
In this section, we offer further details on the primary experiments conducted, along with the best-performing hyperparameter configurations.

\textbf{Licenses:} Below we list the existing assets we use in our work and their corresponding license:
\begin{itemize}
    \item PyTorch Framework (version 2.6) - BSD‑3‑Clause
    \item CIFAR-10/100 Datasets - Apache-2.0
    \item ImageNet Dataset -  Custom license
\end{itemize}

\textbf{Computational Resources:} We run experiments using both our local systems and cloud GPU services. Locally, we have 3 machines dedicated to the experiments conducted in this work and each machine is equipped with 4 x GTX-1080 GPUs and they are connected over Ethernet. We also reserve machines with 4 x H-100 GPUs from cloud services, mainly for the ImageNet experiments.

\textbf{Codebase:} The code used to run the experiments is available at \url{https://github.com/tolgadimli/DPPF}.

\subsection{Generalization Gap vs. Sharpness Measures}
\label{appendix:subsec:gen_vs_sharpness}

In order to establish the correlation between different sharpness measures or indicators from the literature as well as MV and generalization gap (test error - train error \%) of DNNs, we vary several hyperparameters to obtain minima with different quality. \cref{appendix:table:sharpness_vs_gen} lists the hyperparameters we vary and their corresponding search grids. Note that \textit{model width} enables us to manipulate the model capacity as the input channels of the convolution layers are scaled linearly with the value \textit{model width} takes. Furthermore, we also repeat each configuration for 3 seeds (176, 448, 782) which makes the total number of experiment runs $729$. We train models for 300 epochs and apply $10$-fold learning rate drops at epochs $100$ and $200$. We discard any model that attains more than $1\%$ training error from further analysis.

\begin{table}[h]
\renewcommand{\arraystretch}{1.15}
\caption{List of hyperparameters to be varied.}
\label{appendix:table:sharpness_vs_gen}
\centering
\begin{tabular}{c|c}
\textbf{Hyperparameter} & \textbf{Search Grid} \\ \hline
Learning Rate           & $[0.05, 0.075, 0.1]$ \\
Weight Decay            & $[0, 0.0001, 0.01]$  \\
Momentum                & $[0.1, 0.5, 0.9]$    \\
Batch Size              & $[32, 128, 512]$     \\
Model Width             & $[4, 6, 8]$         
\end{tabular}

\end{table}

 The model we use is based on ResNet-18 architecture (aside from the varying convolutional channels) with skip connections and batchnorm layers. Additionally, since the model activation functions are ReLU, the model is scale-invariant \citep{dinh2017sharp}, i.e. the model output can be preserved with re-parameterization of the model. Hence, we normalize all model components (convolution layers, linear layers, learnable batchnorm parameters) such that their Frobenius norms are made $1$ before we conduct the sharpness analysis following \citep{bisla2022low}. This normalization ensures that the conclusions drawn from the analysis cannot be altered by simple rescaling of intermediate model layers, thereby reflecting the true representativeness of the sharpness measures. To assess how well the sharpness measures are aligned with the generalization gap, we use the Kendall correlation coefficient that quantifies the similarity between the orderings of two lists. Below, we provide details on the measures that we inherit from the literature and our measure MV.

\textbf{Shannon Entropy \citep{pereyra2017regularizing}} measures the confidence of a DNN based on the output distributions. Because having a confidence prediction signals that the model is overfitted, this sharpness measure can be expressed as negative of the Shannon Entropy which can be formulated as $\frac{1}{N}\sum_{i=1}^n \sum_{j=1}^c \phi_j(x; u_i) \log \phi_j(x; u_i)$ where $n$ is the number of samples, $u_i$ is the data sample $i$, $x$ is the model parameter vector, $ \phi_j(x; u_i)$ class $j$'s probability and $c$ is the number of classes. Also, note that this measure is not tied to the geometry of the landscape.

\textbf{$\epsilon$-Sharpness \citep{keskar2016large}} can be expressed as the difference between the maximum loss obtained with perturbed model parameters and the original parameters where the perturbation is limited to a box with size $\epsilon$ for each parameter. The parameters that attain the maximum loss in the perturbed region can be determined by first calculating a full-batch gradient and mapping it to the perturbation space.

\textbf{Fisher-Rao Norm \citep{liang2019fisher}} is a capacity measure that can be approximated as $\langle x, H x \rangle$ where $x$ is the model parameters and $H$ is the Hessian matrix, i.e.$H= \mathbb{E}_{\xi \sim D }[\nabla^2 f(x; \xi)]$. In PyTorch, the product $H x$ can be easily calculated by using the HVP (Hessian vector product) function.

\textbf{Low Pass Filter (LPF) \citep{bisla2022low}} accumulates the loss $M$ times due to Markov Chain Monte Carlo (MCMC) iterations for the convolution approximation. In each MCMC iteration, a random vector is drawn from Normal distribution, i.e.  $\varepsilon \sim \mathcal{N}(0, \sigma I)$, and the loss is calculated at $\theta + \varepsilon$ contributes to the total sum which is then averaged across all MCMC iterations. In the experiments, we set $\sigma=0.01$ and $M=100$.

\textbf{Hessian-based Measures \citep{fantastic_measures}} The measures $\lambda_{max}(H)$, Trace($H$), $||H||_{frob}$ are the maximum eigenvalue, trace, and the Frobenius norm of $H$, respectively. We use the Lanczos algorithm to approximate the Hessian matrix (with 3 draws), and then apply the operation (trace, Frobenius norm, etc.) specific to each measure to calculate its value.

\textbf{Inv. Mean Valley} is the additive inverse of the Mean Valley metric which measures the valley width by pushing workers away from the average point. The measure is explained in detail in \cref{sec:MV_measure}. We provide the pseudo-code that calculates this measure in \cref{appendix:alg:imv}. Additionally, we conduct a study to assess the sensitivity of our measure on its hyperparameter $\kappa$. In particular, we vary $\kappa\in \{1.5, 1.75, 2.0, 2.25, 2.5, 3.0, 3.5, 4.0 \}$ and re-calculate the correlation strength when the analysis is carried out on the model parameters trained with and without data augmentation.

\begin{table}[h]
\centering
\caption{Hyperparameter sensitivity analysis of the Inv. MV measure}
\label{appendix:table:inv_mv_sensitivity}
\begin{tabular}{c|cccccccc}
$\mathbf{\kappa}$   & \textbf{1.5} & \textbf{1.75} & \textbf{2.0} & \textbf{2.25} & \textbf{2.5} & \textbf{3.0} & \textbf{3.5} & \textbf{4.0} \\ \hline
\textbf{w/ Aug.}  &   0.590      &    0.606     &    0.616    &   0.614      &  0.612      &   0.613    &    0.610      &    0.604    \\ \hline
\textbf{w/o Aug.} &   0.516           &   0.503            &  0.485           &   0.489            &    0.489          &   0.488     &    0.486           &  0.482          
\end{tabular}
\end{table}

As can be seen from the results in Table~\ref{appendix:table:inv_mv_sensitivity}, the Inv. MV Measure consistently maintains its ability to exhibit strong correlation across a wide range of $\kappa$ values.

\begin{algorithm}[t]
    \caption{Inverse Mean Valley}
    \label{appendix:alg:imv}

    \Input{trained model parameters from $M$ workers $x_1, x_2, ..., x_M$; threshold value $\kappa$ for valley boundary detection ($\kappa=2$ by default); step size $s$ for the resolution of the line search ($s=0.1$ by default).}
    Normalize all model parameters to make them scale-invariant.

    $x_A \leftarrow 0$
    
    \For{$m = 1$ to $M$;}{
    $x_A \leftarrow x_A + \frac{x_m}{M}$
    }
    $l_A \leftarrow f(x_A; D_{train})$
    
    \For{$m = 1$ to $M$;}{
    $d_m \leftarrow \frac{x_m - x_A}{\| x_m - x_A \|_2}$
    
    $x_{m,b} \leftarrow x_A $
    
    \While{True}{
    $x_{m,b} \leftarrow x_{m,b} + s d_m$ 
    
    $l_{m,b} \leftarrow f(x_{m,b}; D_{train})$ 
    
    \If{$l_{m,b} \geq \kappa l_A$;}{
    break
    }
    }
    }

    $\Psi \leftarrow 0$
    
    \For{$m = 1$ to $M$;}{
    $\Psi \leftarrow \Psi + \frac{\| x_{m,b} - x_{A}\|_2}{M}$
    }
    \textbf{Output}: $-\Psi$

\end{algorithm}

\subsection{DPPF with Soft-Consensus Methods}
\label{appendix:subsec:dppf_vs_soft_consensus}

In the experiment set, we pair the wide-minima seeking pushing force, that results from incorporating Inv. MV regularization to the objective, with other Soft-Consensus Methods namely SimpleAvg, EASGD, LSGD, and MGRAWA. We refer to the paired versions with pushing force DPPF\textsubscript{SimpleAvg}, DPPF\textsubscript{EASGD}, DPPF\textsubscript{LSGD} and DPPF\textsubscript{MGRAWA}. As the underlying optimizer, we use SGD with a momentum value of $0.9$ and, weight decay of $1e-3$. We train ResNet-18 models on CIFAR-10 and CIFAR-100 methods for $400$ epochs. Per-GPU batch size is set to $128$. For 4-GPU experiments, the learning rate is set to $0.1$ and for 8-GPU experiments, the learning rate is scaled linearly with the total effective batch size hence it is $0.2$. We only use the basic image augmentations on the train dataset, i.e. random horizontal flip and random cropping from padded images with a padding value of $4$.

In all distributed methods, we fix the communication period to $\tau=4$. For standalone soft-consensus methods, we vary the pulling force $\alpha \in \{0.05, 0.1, 0.3, 0.5\}$ and we report that for each method $\alpha=0.05$ and $\alpha=0.1$ produces very similar test errors and $0.5$ performs the worst. For DPPF variants, we fix $\alpha=0.1$ and vary the strength of the pushing force $\lambda \in \{0.05, 0.1, 0.25, 0.5, 0.75\}$. We also repeat each experiment for $3$ different seed $182, 437, 965$ and each configuration's performance is determined by taking the average of the training statistics across seeds. In \ref{table:param_sharing_improvement}, we report the lowest average test error ($\%$) achieved across all configurations. In \ref{appendix:table:best_dppf_sc2}, we report the best $\lambda$ values for DPPF variants and as can be seen, $\lambda=0.5$ proves to be a solid value.

\begin{table}[h]
\renewcommand{\arraystretch}{1.15}
\centering
\caption{Best $\lambda$ values of DPPF variants in different settings.}
\label{appendix:table:best_dppf_sc2}
\resizebox{0.5\columnwidth}{!}{
\begin{tabular}{c|cc|cc}
     & \multicolumn{2}{c|}{\textbf{CIFAR-10}} & \multicolumn{2}{c}{\textbf{CIFAR-100}} \\ \hline
     & \textbf{4 GPUs}        & \textbf{8 GPUs}        & \textbf{4 GPUs}        & \textbf{8 GPUs}        \\ \hline
\textbf{DPPF\textsubscript{SimpleAvg}} &      0.5         &     0.5       &     0.5          &    0.75           \\
\textbf{DPPF\textsubscript{EASGD}}     &      0.75         &     0.75       &     0.5          &    0.5       \\
\textbf{DPPF\textsubscript{MGRAWA}}    &      0.5         &     0.5       &     0.5          &    0.1       
\end{tabular}}
\end{table}

\begin{remark}
\label{rem:LSGD}
In LSGD, the pull force is directed toward the leader, while the push force acts away from the average variable. This misalignment leads to instability. However, if the pushing force is redefined to oppose the leader's direction, the method converges. For example, for $4$ workers on CIFAR-10 achieves the test error of $4.07\pm0.15\%$ and on CIFAR-100 it achieves $20.81\pm0.39\%$, which is better than vanilla LSGD that uses no pushing force.
\end{remark}

\subsection{Comparison with Communication-Efficient Methods}
\label{appendix:subsec:comm_eff}
In the experiments with LocalSGD, we use fixed communication periods \(\tau\) as specified in Table~\ref{dppf:table:comm_eff}, and set \(\alpha = 1.0\) to achieve full consensus at each communication step.

For employing QSR on top of LocalSGD, we adopt the \(\beta\) values reported in~\citep{guquadratic}. (Note: the original paper denotes this hyperparameter as \(\alpha\); we use \(\beta\) here to avoid confusion with the distributed pull strength.) The authors tuned \(\beta \in \{0.2, 0.25, 0.3\}\) under a maximum learning rate of \(\eta_{\text{max}} = 0.8\). To preserve the same scheduling dynamics in our setting, we match the ratio \(\frac{\beta}{\eta_t}\) in the QSR formulation:

\[
\tau_t = \max \left\{ \tau_{\text{base}}, \left\lfloor \left( \frac{\beta}{\eta_t} \right)^2 \right\rfloor \right\}.
\]

Based on this, we scale the \(\beta\) values proportionally to our chosen \(\eta_{\text{max}}\) in each training setup. Specifically, we search:

\begin{itemize}
    \item \(\beta \in \{0.025, 0.03125, 0.0375\}\) when \(\eta_{\text{max}} = 0.1\) in ResNet-18 training,
    \item \(\beta \in \{0.050, 0.0625, 0.075\}\) when \(\eta_{\text{max}} = 0.2\) in PyramidNet training,
    \item \(\beta \in \{0.1875, 0.234375, 0.28125\}\) when \(\eta_{\text{max}} = 0.75\) in ResNet-50 training.
\end{itemize}

\paragraph{Experiments with CNNs.} We train ResNet-{18,50,101}, and ViT models on the CIFAR-10, CIFAR-100, and ImageNet datasets using 4-GPU (GTX-1080), 8-GPU (GTX-1080), and 4-GPU (H100) setups, respectively. The batch sizes are set to 512 (CIFAR-10), 1024 (CIFAR-100), and 3072 (ImageNet). ResNet-18 and PyramidNet models are trained for 400 epochs, and ResNet-50 is trained for 200 epochs, following the training recipe in~\citep{SAM}. In the ImageNet experiments, we use a weight decay of 0.0001, while for CIFAR-10 and CIFAR-100, the weight decay is set to 0.001. To reduce computation, we tune the optimal $\beta$ for QSR on PyramidNet with CIFAR-100, and scale it proportionally with the learning rate in other settings.

For DPPF, we fix \(\alpha = 0.1\) and search over \(\lambda \in \{0.5, 0.75, 1.0\}\) for CIFAR-10 and CIFAR-100 experiments, after seeing the trend in the first experiment set presented in Section~\ref{appendix:subsec:dppf_vs_soft_consensus}. For ImageNet experiments with ResNet-{50,101}, we fix the ratio between push ($\lambda$) and pull ($\alpha$) forces to 10, i.e, $\lambda /\alpha = 10$ and search over the following grid: $(\lambda, \alpha) \in \{(0.1, 1.0), (0.5, 5.0), (0.9, 9.0)\}$. For LocalSGD and DPPF, we evaluate fixed communication periods \(\tau \in \{4, 8, 16\}\). For LocalSGD+QSR, we search over \(\tau \in \{2, 4, 8\}\), consistent with the experimental protocol in~\citep{guquadratic}. We repeat all the experiments for $3$ seeds. We report the best hyperparameters in Table~\ref{appendix:table:qsr_dppf_optimal_hyperparams}.

\begin{table}[h]
\renewcommand{\arraystretch}{1.15}
\centering
\caption{Optimal hyperparameters reported for LocalSGD + QSR and DPPF}
\label{appendix:table:qsr_dppf_optimal_hyperparams}
\resizebox{0.8\columnwidth}{!}{
\begin{tabular}{c|ccc|ccc}
                     & \multicolumn{3}{c|}{LocalSGD + QSR ($\beta$)} & \multicolumn{3}{c}{DPPF ($\alpha, \lambda$)}       \\ \cline{2-7} 
                     & $\tau_{\text{base}}=2$   & $\tau_{\text{base}}=4$   & $\tau_{\text{base}}=8$  & $\tau=4$ & $\tau=8$ & $\tau=16$ \\ \hline
ResNet-18 - C10      & 0.025  &  0.03125  & 0.03125 &   (0.1,0.5)    &   (0.1,0.5)    &   (0.1,0.5)   \\
PyramidNet - C100    & 0.050  &  0.0625  &  0.0625  &   (0.1,1.0)   &   (0.1,0.75)   &    (0.1,0.75)   \\
ResNet-50 - ImageNet & 0.1875  &  0.234375  &  0.234375   &  (0.5, 5.0)    &  (0.5, 5.0)   &  (0.9, 9.0) \\ 
ResNet-101 - ImageNet & 0.1875  &  0.234375  &  0.234375   &  (0.9, 9.0)    &  (0.9, 9.0)   &  (0.9, 9.0) \\ 
\end{tabular}
}
\end{table}

\paragraph{Experiments with ViT.} 
Unlike the experiments with CNNs, we use the AdamW optimizer to train the ViT model, as it is the widely adopted choice for transformer-based architectures. Specifically, we use a \texttt{timm}-provided ViT variant, \texttt{vit\_relpos\_medium\_patch16\_224.sw\_in1k}, which consists of 12 layers and 39M parameters. In our experiments, we do not use pretrained weights; instead, we initialize the model parameters randomly according to the specified seed.

We begin with a hyperparameter search to determine the optimal learning rate and weight decay values. The model is trained on four H100 GPUs with a per-GPU batch size of 512 using the DDP training, the standard synchronous gradient averaging method. The optimal learning rate and weight decay are identified as $0.0005$ and $0.01$, respectively when cosine annealing scheduling is used. For DPPF, we initially used the coefficients found in the ResNet-101 setting from Table~\ref{appendix:table:qsr_dppf_optimal_hyperparams}. Although we observed improvements over DDP AdamW training with $\alpha = 0.9$ and $\lambda = 9.0$ up to $0.8\%$, we decided to further increase the final valley width to explore a wider valley for the ViT model, since it has a higher parameter vector norm than ResNet-101. Specifically, we increased $\lambda$ by a factor of ten, from $9.0$ to $90.0$, while keeping $\alpha = 0.9$ fixed, thereby expanding the valley width from $\lambda / \alpha = 10.0$ to $100.0$. The DPPF results reported in Section~\ref{sec:dppf:comp_other_comm_eff} are obtained using $\alpha = 0.9$ and $\lambda = 90.0$ for all communication periods $\tau \in \{4, 8, 16\}$. For a fair comparison, we also experimented with increasing the $\beta$ coefficient of QSR by a factor of ten; however, this adjustment did not lead to any notable improvement in performance. 

\subsection{Comparison with DDP and SAM}
\label{appendix:subsec:dppf_sgd_sam}

Here we provide more details on the experiments that we compare DPPF with DDP, the most commonly adopted distributed training method. Note that we refer to DPPF\textsubscript{SimpleAvg} by DPPF as mentioned in \ref{sec:dppf:comp_other_comm_eff}. In this experiment set, we also switch the underlying optimizer from SGD to SAM which minimizes the maximum loss value attained in a region and it is known to encourage flatter, good-quality minima. These four combinations - DDP SGD, DPPF SGD, DDP SAM, and DPPF SAM — allow us to compare the effectiveness of flat-minima-seeking updates incorporated into the optimization at different levels. DDP SGD does not have any flat minima seeking mechanisms, DPPF SGD encourages wide minima discovery in the distributed level, DDP SAM promotes flatness with its local optimizer's objective and DPPF SAM is equipped with both distributed and local level flat, wide minima encouraging mechanisms. 

\begin{table}[h]
\renewcommand{\arraystretch}{1.15}
\centering
\caption{\centering Optimal $\lambda$ values of DPPF variants and optimal $\rho$ values of DDP SAM in various settings.}
\label{appendix:table:best_dppf_sc}
\resizebox{0.45\columnwidth}{!}{
\begin{tabular}{ccc|cc|cc}
                                                                                                   &                                                 &                   & \multicolumn{2}{c|}{\textbf{CIFAR-10}} & \multicolumn{2}{c}{\textbf{CIFAR-100}} \\ \cline{4-7} 
                                                                                                   &                                                 &                   & \textbf{4 GPUs}    & \textbf{8 GPUs}   & \textbf{4 GPUs}    & \textbf{8 GPUs}   \\ \hline
\multicolumn{1}{c|}{\multirow{3}{*}{\textbf{RN18}}}                                                & \multicolumn{1}{c|}{\multirow{2}{*}{$\lambda$}} & \textbf{DPPF SGD} &   0.5      &      0.5        &       0.5        &        0.75        \\
\multicolumn{1}{c|}{}                                                                              & \multicolumn{1}{c|}{}                           & \textbf{DPPF SAM} &       0.05       &     0.05        &       0.1      &   0.15       \\ \cline{2-7} 
\multicolumn{1}{c|}{}                                                                              & \multicolumn{1}{c|}{$\rho$}                     & \textbf{DDP SAM}  &        0.1        &      0.2         &      0.2          &      0.2         \\ \hline
\multicolumn{1}{c|}{\multirow{3}{*}{\textbf{\begin{tabular}[c]{@{}c@{}}WRN\\ -16x8\end{tabular}}}} & \multicolumn{1}{c|}{\multirow{2}{*}{$\lambda$}} & \textbf{DPPF SGD} &       0.25         &     0.25         &      0.5           &     0.5          \\
\multicolumn{1}{c|}{}                                                                              & \multicolumn{1}{c|}{}                           & \textbf{DPPF SAM}  &      0.05         &     0.05         &      0.15           &     0.1           \\ \cline{2-7} 
\multicolumn{1}{c|}{}                                                                              & \multicolumn{1}{c|}{$\rho$}                     & \textbf{DDP SAM}  &        0.1        &        0.2        &         0.2      &        0.2        \\ \hline
\multicolumn{1}{c|}{\multirow{3}{*}{\textbf{PyNet}}}                                               & \multicolumn{1}{c|}{\multirow{2}{*}{$\lambda$}} & \textbf{DPPF SGD} &          0.5     &      0.5        &        0.75         &     0.75           \\
\multicolumn{1}{c|}{}                                                                              & \multicolumn{1}{c|}{}                           & \textbf{DPPF SAM} &          0.1       &     0.1          &      0.15              &    0.15               \\ \cline{2-7} 
\multicolumn{1}{c|}{}                                                                              & \multicolumn{1}{c|}{$\rho$}                     & \textbf{DDP SAM}  &       0.1         &      0.2        &       0.2         &          0.2    
\end{tabular}}
\end{table}

We train ResNet-18, WideResNet-16x8 (WRN-16x8) and Pyramidnet-110 additive with $\alpha$=270 (PyNet(110,270)) models on CIFAR-10 and CIFAR-100 datasets in 4-GPU and 8-GPU setups. For the methods with SGD, the models are trained for $400$ epochs with a momentum value of $0.9$ and a weight decay coefficient of $0.001$. We want to note that most papers developing SAM variants, including the original SAM paper, report $0.0005$ as the optimal weight decay factor for the SGD optimizer. In our experiments, we observe that setting SGD's weight decay to $0.0005$ undermines its performance and report $0.001$ as its optimal weight decay value. The per-GPU batch size is fixed at $128$ in all experiments and the learning rate is scaled linearly with the total effective batch size. Particularly, the learning rates are $0.1$ and $0.2$ for the 4-GPU and 8-GPU experiments respectively. For the settings with SAM, the models are trained for $200$ epoch since a single iteration of SAM is equivalent to double SGD iterations due to gradient ascent and descent steps. SAM's underlying optimizer is also SGD with the same hyperparameters specified previously. Again, we only use the basic image augmentations on the train dataset, i.e. random horizontal flip and random cropping from padded images with a padding value of $4$. 

For DDP SAM, we search SAM's ascent step coefficient $\rho \in \{0.05, 0.1, 0.2\}$. Although \citep{SAM} searches $\rho$ in a wider grid, only these $\rho$ values are reported as optimal in various settings. We initially also experimented with $\rho \in \{0.05, 0.1, 0.2, 0.3, 0.4\}$ with ResNet-18 and observed that higher values than $\rho=0.2$ starts hurting the performance on both CIFAR-10 and CIFAR-100. Hence we ultimately settle on $\rho \in \{0.05, 0.1, 0.2\}$. 

For DPPF SGD, we use the same search space as before $\lambda \in \{0.05, 0.1, 0.25, 0.5, 0.75\}$. For DPPF SAM however, we fix $\rho=0.1$ and then opt for a more conservative search space as two wide-minima seeking mechanisms might overwhelm and hinder the minimization of classification loss. Particularly, we search $\lambda \in \{0.03, 0.05, 0.1, 15, 0.2\}$ in the DPPF SAM configuration. We also fix the pulling force $\alpha = 0.1$. Ideally, all hyperparameters should be searched jointly, but this would exponentially increase the number of experiments. Therefore, we focus on varying $\lambda$ alone which is the pushing force. We believe the generalization performance of this setting could be further improved with a more thorough, comprehensive hyperparameter search that also includes varying $\rho$ and $\alpha$. Similar to the previous experiments set, we run each configuration for 3 different seeds $182, 437, 965$ and each configuration's performance is determined by averaging the statistics across seeds. In Table~\ref{appendix:table:best_dppf_sc}, we share the hyperparameters of the best configurations for reproducibility.

\section{Details of the Ablation Studies}
Here, we give details on the ablation studies presented in the main body of the paper. In all
 ablation studies, we run training or conduct analysis in a 4-GPU training setup unless otherwise stated.
 
\subsection{Pull-Push Mechanism}
\label{appendix:subsec:pull_push_mechanism}
Here we give details on the experiments to analyze the importance of pushing mechanisms and the interplay between the pull-push forces throughout the training. First, to demonstrate the necessity of our pushing force, we compare DPPF\textsubscript{SimpleAvg} that has pulling force $\alpha=0.1$ and pushing force $\lambda=0.5$ with vanilla SimpleAvg methods that have weaker pulling forces, particularly $\alpha \in \{ 0.001, 0.005, 0.01, 0.05 \}$. We train ResNet-18 models on the CIFAR-100 dataset in these configurations for $3$ seeds $182, 437, 965$ and average test errors curves across seeds to obtain the plot in \ref{fig:ablation_weak_lambda_1} where the shaded regions indicate the standard deviation. We also log the Euclidian distance workers to the average variable during the training at every iteration, i.e. $\| x_m - x_A \|_2$ for each worker $m$. For \ref{fig:ablation_weak_lambda_2}, we calculate the simplified MV, which is the average of these distances and plot its change with respect to training iterations.

In addition to logging the distance between each worker, we also record the strength of the applied pulling force and the pushing force. At each iteration, we compare the strength of these two forces to characterize the interplay between them. The results are presented in \ref{fig:ablation_push_pull_figure}.

\subsection{How to Schedule The Pushing Force?}
\label{appendix:subsec:scheduling_lambda}
We train ResNet-18 models on the CIFAR-100 dataset to compare the effect of different schedulings of the pushing force $\lambda$ to the end performance. Particularly, we compare three schedulings: fixed, decreasing and increasing throughout the training. These schedulings are plotted in \ref{appendix:fig:schedule_schedules} when $\lambda=0.5$ and below, we provide more details to each scheduling:

\begin{itemize}
    \item \textbf{Fixed:} The strength of the pushing force $\lambda$ is kept constant throughout the training.
    \item \textbf{Decreasing:} The $\lambda$ value is decayed in parallel with the learning rate. Since we are using cosine annealing scheduler for the learning rate, at any iteration $t$, $\lambda_t = \frac{\lambda}{2} \left ( 1 + \cos \left ( \frac{t}{T}\pi \right )  \right )$ where $T$ is the total number of iterations.
    \item \textbf{Increasing:} In this setting, the strength of $\lambda$ is amplified towards the end of the training. We again base the amplification on the learning rate for simplicity and use flipped cosine annealing for scheduling. More specifically, $\lambda_t = \frac{\lambda}{2}\left ( 1 - \cos \left ( \frac{t}{T}\pi \right )  \right )$ where $T$ is the total number of iterations.
\end{itemize}

\begin{figure}[htp]
  \centering
  \begin{subfigure}[t]{0.32\textwidth}
    \includegraphics[width=\linewidth]{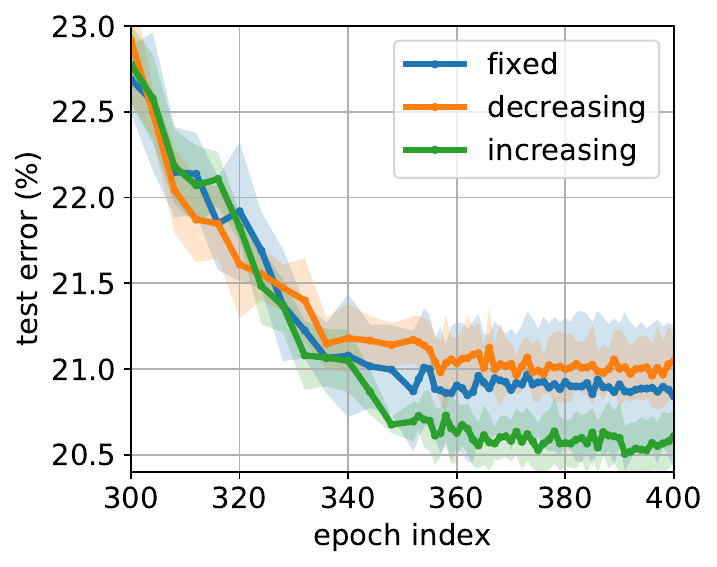}
    \caption{Test errors (\%)}
    \label{appendix:fig:schedule_errors}
  \end{subfigure}
  \hfill
   \begin{subfigure}[t]{0.32\textwidth}
    \includegraphics[width=\linewidth]{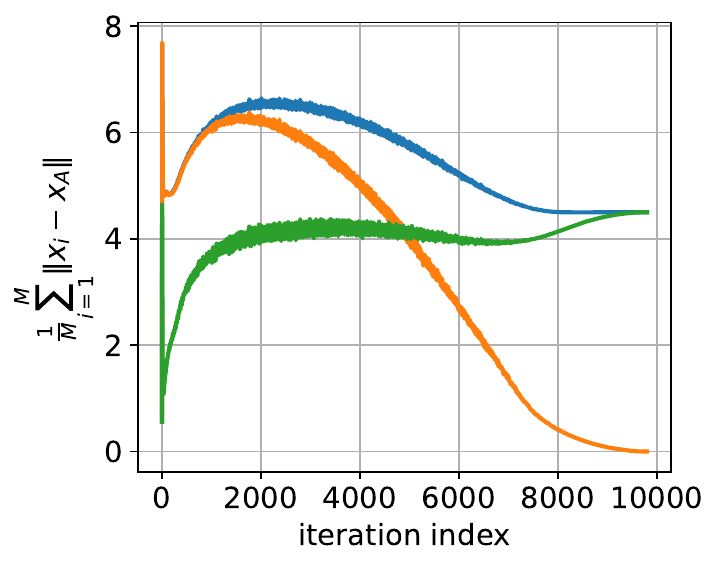}
    \caption{Consensus distance}
    \label{appendix:fig:schedule_valley_change}
  \end{subfigure}
\hfill
  \begin{subfigure}[t]{0.32\textwidth}
    \includegraphics[width=\linewidth]{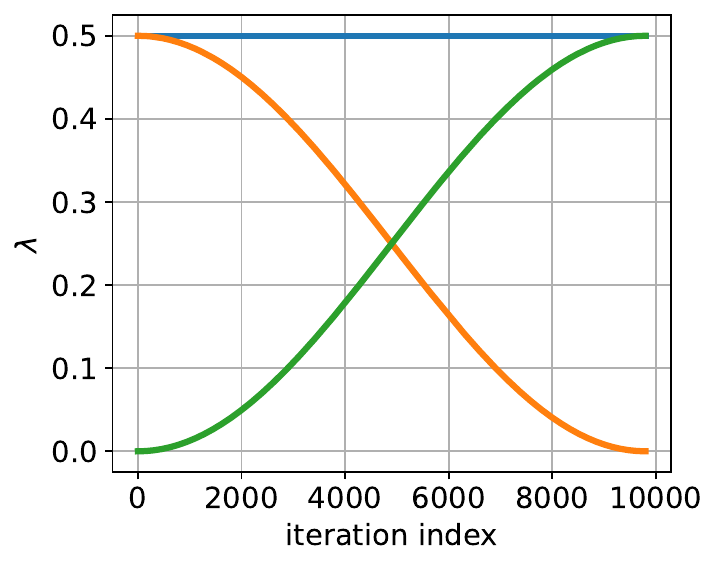}
    \caption{Change of $\lambda$}
    \label{appendix:fig:schedule_schedules}
  \end{subfigure}
  \caption{Comparison of different schedulings.}
\end{figure}

For each scheduling, we repeat the experiment for $5$ different seeds $42, 182, 437, 965, 1283$ and compare the test error of different scheduling averaged over seeds. The results in \ref{appendix:fig:schedule_errors} show that increasing the strength of $\lambda$ towards the end of training is more effective, as recovering wide basins becomes more important at the end. More specifically, the test errors and standard deviations, calculated across $5$ seeds, are $20.84_{\pm 0.41}$, $21.05_{\pm 0.17}$,  $20.61_{\pm 0.15}$ for fixed, increasing, and decreasing schedules, respectively.

\subsection{Details on the Non-IID Experiments}
\label{appendix:subsec:noniid_exps}

To simulate heterogeneous client distributions, we adopt the standard Dirichlet-based partitioning protocol used in the federated learning literature~\cite{karimireddy2020scaffold,fedlesam,fedsam}. 
Given $M=4$ workers, we allocate class proportions for each worker by sampling from a Dirichlet distribution with concentration parameter $\alpha \in \{0.1, 0.6\}$. 
Smaller $\alpha$ induces stronger heterogeneity which yields workers that concentrate on a subset of classes. 
The sampled class proportions are used to split the training data once at initialization, and the resulting partition is kept fixed throughout training (i.e., no reshuffling across epochs unlike the IID experiments).

We evaluate on the following setting:
\begin{itemize}
    \item CIFAR-10 with ResNet-18,
    \item CIFAR-100 with PyramidNet(110, 270).
\end{itemize}
All models use standard CIFAR data augmentation (random crop with padding and random horizontal flip) and normalization. 
For ResNet-18, the per-worker batch size is 256, and for PyramidNet it is 128. In all settings, we use SGD with momentum as the local optimizer, with initial learning rate 0.1, momentum 0.9, and weight decay $10^{-3}$. A cosine learning rate schedule is applied over 200 training epochs.  We use $M=4$ workers and communication period $\tau=16$, meaning each worker performs $\tau$ local SGD steps before synchronization. We report mean and standard deviation over three independent runs with random seeds $\{182, 437, 965\}$.

When coupling DPPF with SCAFFOLD or FedLESAM (denoted DPPF$_{\text{SCAFFOLD}}$ and DPPF$_{\text{FedLESAM}}$), we replace the standard FedAvg-style aggregation step with the DPPF update defined in Equation~\ref{eq:single_line}. 

Concretely, after $\tau$ local updates are performed according to the underlying FL solver, the worker parameters $\{x_m\}_{m=1}^M$ are used to compute the global average $x_A$.  The DPPF pull–push transformation is then applied to the worker deviations before aggregation.  All other components of SCAFFOLD and FedLESAM (e.g., control variates or sharpness-aware perturbations) remain unchanged. Thus, DPPF acts purely at the aggregation level and does not modify the internal optimization dynamics of the base FL solvers.

For FedLESAM, we tune the perturbation coefficient $\rho \in \{0.1, 0.05, 0.001, 0.0005\}$ and select $\rho = 0.001$ based on validation performance. For all DPPF variants, we fix the pull coefficient to $\alpha = 0.9$ and sweep only the push strength $\lambda$. For DPPF$_{\text{SCAFFOLD}}$, we search $\lambda \in \{0.9, 1.8, 2.7, 3.6\}$, corresponding to target valley-width ratios $\lambda/\alpha \in \{1, 2, 3, 4\}$. We find $\lambda = 1.8$ to perform best for ResNet-18 and $\lambda = 3.6$ for PyramidNet. For DPPF$_{\text{FedLESAM}}$, larger push strengths yield diminishing returns. 
Since this variant combines two flatness-promoting mechanisms, we perform a more conservative search over $\lambda \in \{0.1, 0.3, 0.45, 0.6\}$, corresponding to $\lambda/\alpha \in \{1/9, 1/3, 1/2, 2/3\}$. 
We find $\lambda = 0.6$ to perform best across all settings.

\section{Additional Results and Ablation Studies}

\subsection{Ablation Study on the Second Term}
\label{appendix:subsec:mv_second_term}
In Section~\ref{sec:dppf}, we present the update rule that arises from the Simplified MV ($R$) term in the objective. However, in practice we execute the simplified update rule by only keeping the first term. Let us consider the full, original update with both terms. The full update expression is as follows (as proven in Section~\ref{appendix:subsec:mv_update_rule}):
\begin{equation*}
    \frac{\partial R}{\partial x_m} = - \frac{\lambda}{M^2} \left ( M\frac{d_m}{\|d_m\|}  - \sum_{j =1}^{M} \frac{ d_j }{\|d_j\|} \right ).
\end{equation*}

where $d_m = x_m - x_A$. In practice, we drop the second term in the parentheses because when the workers are symmetrically spread around the average variable, its value is close to $0$. In such a case, the overall update can be approximated as:
\begin{equation*}
    - \frac{\lambda}{M^2} \left ( M\frac{d_m}{\|d_m\|}  - \sum_{j=1}^{M} \frac{ d_j }{\|d_j\|} \right ) \approx - \frac{ \lambda }{M} \frac{d_m}{\|d_m\|}
\end{equation*}
We also empirically check if this is the case. Let $T_1 = - \frac{\lambda}{M} \frac{d_m}{\|d_m\|}$ and $T_2 =  \frac{\lambda}{M^2} \sum_{j=1}^{M} \frac{ d_j }{\|d_j\|}$ hence the overall update can be expressed as $T_1 + T_2$. We run the experiment with $M=4$ workers and plot how the Euclidean norms of $T_1$, $T_2$, and $T_1 + T_2$ change during the training. To check the validity of our claim empirically, we also scale the norm of $T_1$. The results in Figure~\ref{appendix:fig:second_term} reveal that indeed the simplified expression $-\lambda \frac{1}{M} \frac{d_m}{\|d_m\|}$ is a good proxy to the actual update. Although some fluctuations are not captured, as perfect symmetry of workers around the average variable is not always ensured, we did not observe any change in the final performance. Besides, the simplified update is more communication-efficient as the calculation of the second term requires either an additional communication round among the workers, or a costlier communication to retrieve the model copies from all the workers.
\begin{figure}[H]
    \centering
    \includegraphics[width=0.45\columnwidth]{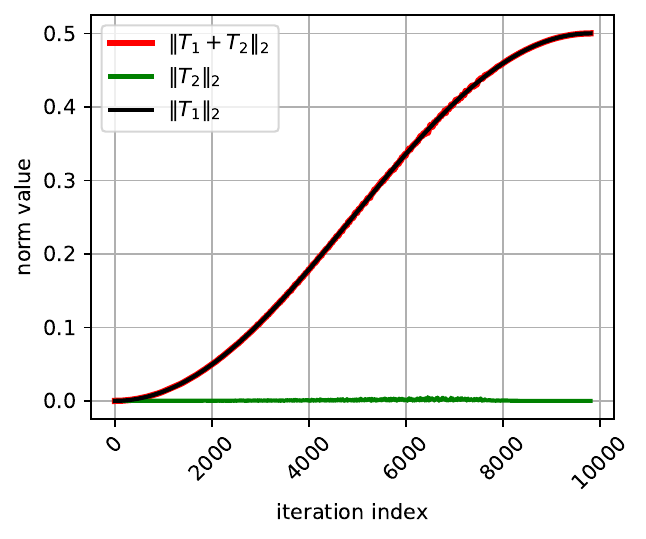} % 
    \caption{Ablation on the presence of the second term}
    \label{appendix:fig:second_term}
\end{figure}

\subsection{DPPF's Hyperparameter Sensitivity and Results Supporting Theorem 2}
\label{appendix:subsec:hyperparam_sens}

In this section, we investigate how sensitive the DPPF is to the hyperparameter selections. Particularly, DPPF has two hyperparameters: the pull $\alpha$ and the push $\lambda$ force strengths. The theoretical analysis reveals that the final valley width found by DPPF is governed by the ratio of the push and pull forces, i.e., $\lambda/\alpha$. We consider training the PyramidNet(270,110) on CIFAR-100 for 400 epochs, a model with more than enough capacity that can potentially suffer from overfitting. In our first sensitivity analysis, we fix $\alpha=0.5$ and try different values of $\lambda$ for DPPF training. In particular, we use $\lambda \in \{ 0.1,0.25, 0.5, 1.0, 2.5, 5, 7.5, 10\}$ so that the DPPF finds valleys with different width. We share the final test errors (averaged across 3 seeds) and the valley width side-by-side in Figure~\ref{appendix:hypersens:fixed_alpha}.

\begin{figure}[h]
  \centering
  \begin{subfigure}[t]{0.5\textwidth}
    \includegraphics[width=\linewidth]{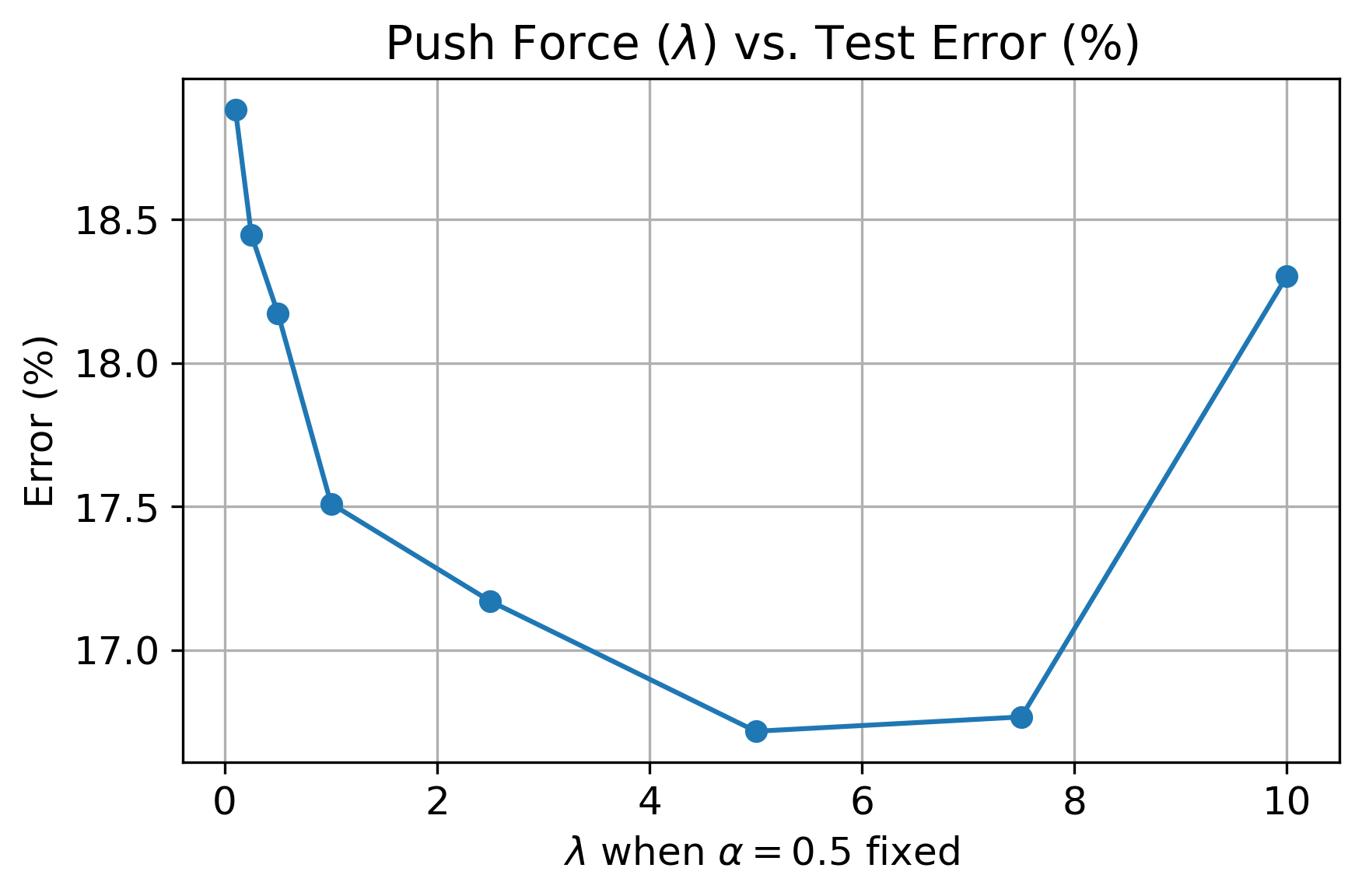}
    \caption{Test errors (\%)}
    \label{appendix:hypersens:fixed_alpha_error}
  \end{subfigure}
  \hfill
  \begin{subfigure}[t]{0.4\textwidth}
    \includegraphics[width=\linewidth]{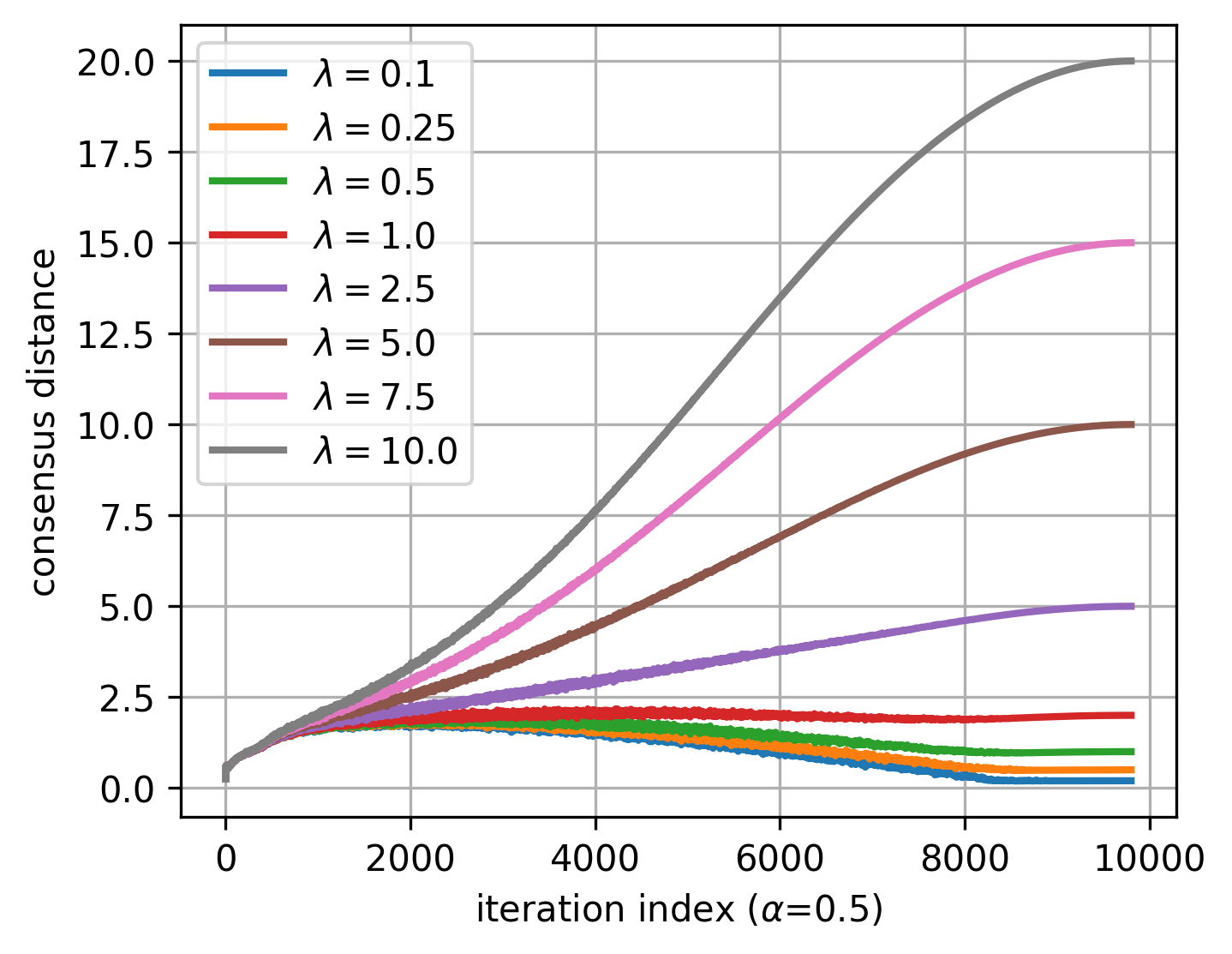}
    \caption{Valley width}
    \label{appendix:hypersens:fixed_alpha_valley}
  \end{subfigure}
  \caption{PyramidNet training on CIFAR-100 dataset with 4 workers.}
  \label{appendix:hypersens:fixed_alpha}
\end{figure}

As shown in Figure~\ref{appendix:hypersens:fixed_alpha}, while both extremely narrow valleys ($\lambda < 1$) and overly wide basins ($\lambda > 8$) lead to suboptimal generalization, we observe that across a broad range of intermediate values ($1 \leq \lambda \leq 8$), the push force introduced by DPPF—when the pull strength is fixed at $\alpha = 0.5$—consistently improves test performance. For reference, standard DDP-SGD (i.e., synchronous gradient averaging) achieves a test error of $19.18\%$ after training PyramidNet on CIFAR-100 for 400 epochs. Moreover, the trend in Figure~\ref{appendix:hypersens:fixed_alpha} aligns with the claim made in Theorem~\ref{main:thm:pac_bayes}, which states that, under mild technical conditions, increasing the valley width leads to better generalization—an effect that is clearly visible up to approximately $\lambda = 5$. In Figure~\ref{appendix:hypersens:fixed_alpha2}, we present statistics on how the norm of the average variable ($x_A$) evolves during training, along with the ratio between the final valley width and the norm of the average variable ($\|x_A\|_2$), for different values of $\lambda$ while keeping $\alpha = 0.5$ fixed again.

\begin{figure}[h]
  \centering
  \begin{subfigure}[t]{0.46\textwidth}
    \includegraphics[width=\linewidth]{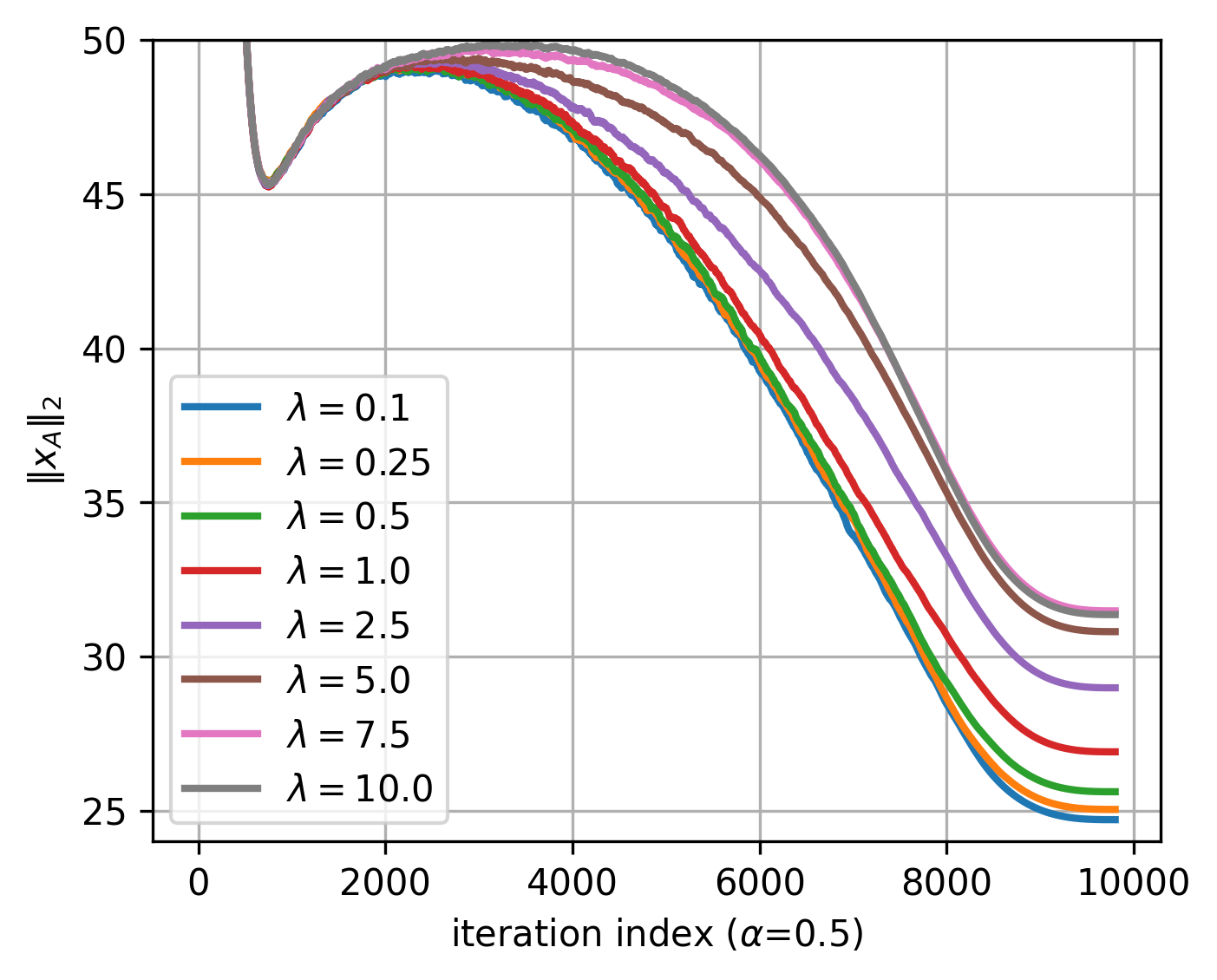}
    \caption{Average model norm}
    \label{appendix:hypersens:fixed_alpha_norm}
  \end{subfigure}
  \hfill
  \begin{subfigure}[t]{0.48\textwidth}
    \includegraphics[width=\linewidth]{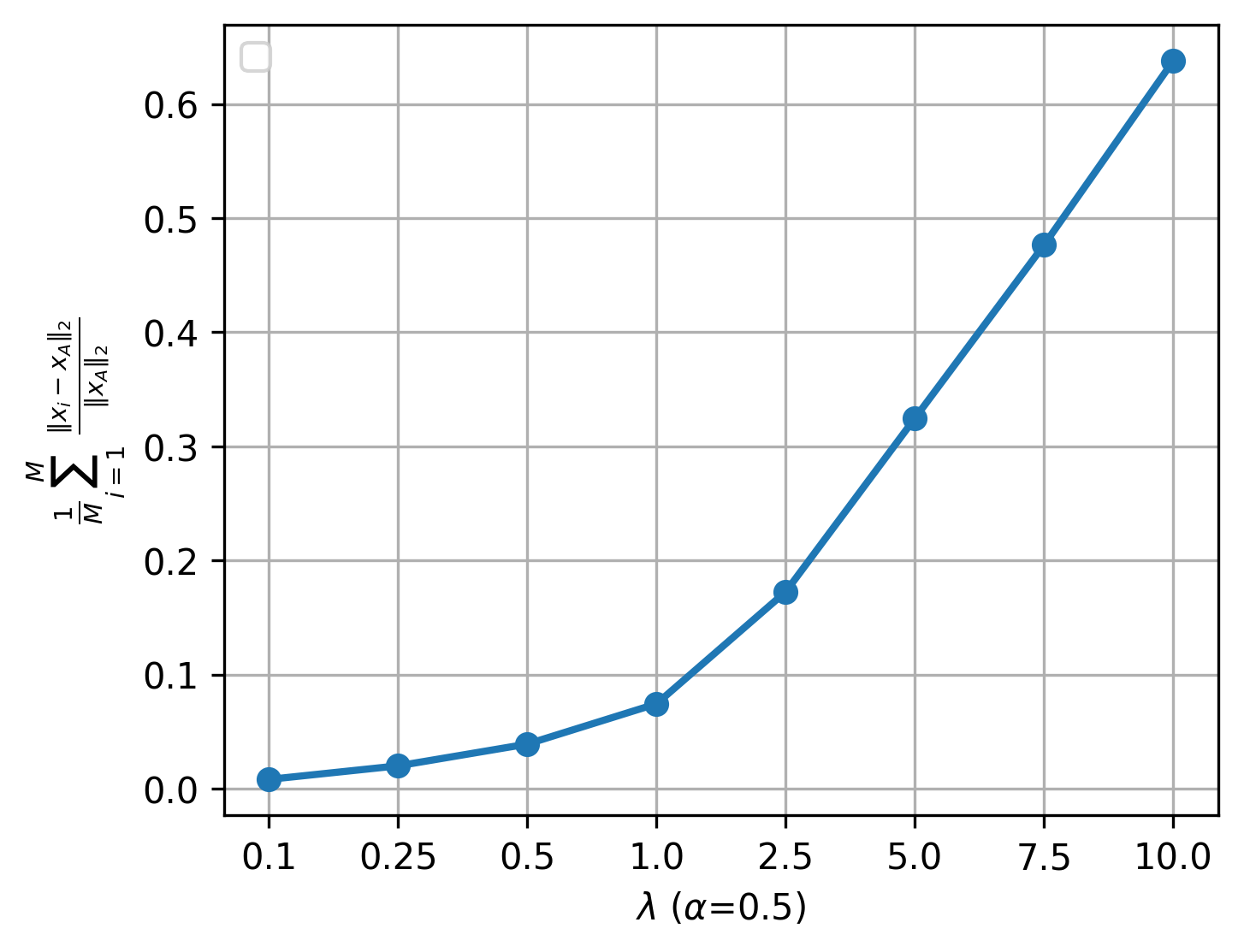}
    \caption{Valley width over norm}
    \label{appendix:hypersens:fixed_alpha_valley_over_norm}
  \end{subfigure}
  \caption{PyramidNet training on CIFAR-100 dataset with 4 workers.}
  \label{appendix:hypersens:fixed_alpha2}
\end{figure}

Figure~\ref{appendix:hypersens:fixed_alpha_norm} illustrates how the norm of the average variable ($|x_A|2$) evolves throughout training when different push force strengths are applied. We observe that stronger push forces lead to higher final norms of the average variable. This is consistent with the assumption made in Theorem~\ref{main:thm:pac_bayes}. In Figure~\ref{appendix:hypersens:fixed_alpha_valley_over_norm}, we report the ratio between the valley width, defined as $\frac{1}{M}\sum{i=1}^M |x_i - x_A|_2$, and the norm of the average variable $|x_A|_2$, measured at the end of training. This ratio grows approximately exponentially with increasing $\lambda$, and appears to diverge for large $\lambda$ values.

Recall from Figure~\ref{appendix:hypersens:fixed_alpha_error} that the generalization improvements brought by the push force begin to saturate or degrade beyond $\lambda = 5$. This suggests the existence of a critical $\lambda$—and potentially a critical valley-width-to-norm ratio—up to which consistent generalization benefits are observed. While DPPF yields substantial generalization improvements over a wide range of $\lambda$ values (for fixed $\alpha$), future work could aim to identify this critical threshold, potentially enabling a hyperparameter-free variant of DPPF. One promising direction could involve estimating a lower bound on the Lipschitz constant of the DNN for a given dataset, which may provide guidance on appropriate values for $\lambda$.

Previously, we examined how the performance improvements of DPPF vary with different values of $\lambda$ while keeping the pull force fixed at $\alpha = 0.5$. Now, we fix the final valley width by preserving the ratio $\lambda / \alpha$, and explore different $(\lambda, \alpha)$ pairs to investigate how test error is affected. Before presenting our results, we report the performance achieved by the standard gradient-averaging scheme (DDP-SGD) across various experimental settings as a reference, as shown in Table~\ref{appendix:table:hyperparam_ddp}.

\begin{table}[h]
\renewcommand{\arraystretch}{1.15}
\caption{Test errors attained by DDP SGD Training}
\label{appendix:table:hyperparam_ddp}
\centering
\begin{tabular}{c|c|c}
\begin{tabular}[c]{@{}c@{}}ResNet-18\\ CIFAR-10\\ 4 Workers\\ 400 Epochs\end{tabular} & \begin{tabular}[c]{@{}c@{}}PyramidNet\\ CIFAR-100\\ 4 Workers\\ 400 Epochs\end{tabular} & \begin{tabular}[c]{@{}c@{}}ResNet-50\\ ImageNet\\ 4 Workers\\ 200 Epochs\end{tabular} \\ \hline
   $4.33_{\pm 0.08}$     &   $19.18_{\pm 0.10}$   &     $23.83_{\pm 0.17}$    
\end{tabular}
\end{table}

Unless otherwise stated, we maintain consistent experimental settings throughout our analysis. For CIFAR-10 and CIFAR-100 experiments, we vary the communication period $\tau \in \{4, 8, 16\}$ and use $(\lambda, \alpha) \in \{(0.1, 0.5), (0.5, 2.5), (0.9, 4.5)\}$ to ensure that the final valley width—governed by the ratio $\lambda / \alpha$—remains fixed at $5$. Additionally, we include experiments with a shorter training schedule of 200 epochs, in contrast to the default 400 epochs. For ImageNet experiments, we use $(\lambda, \alpha) \in \{(0.1, 1.0), (0.5, 5.0), (0.9, 9.0)\}$ to keep the valley size fixed at $10$ and only provide 200 epoch training, the default recipe. Figures~\ref{appendix:hypersens:resnet18_c10}, \ref{appendix:hypersens:pyramidnet_c100}, and~\ref{appendix:hypersens:resnet50_imagenet} present the hyperparameter sensitivity analysis in the form of heatmaps.

\begin{figure}[h]
  \centering
  \begin{subfigure}[t]{0.4\textwidth}
    \includegraphics[width=\linewidth]{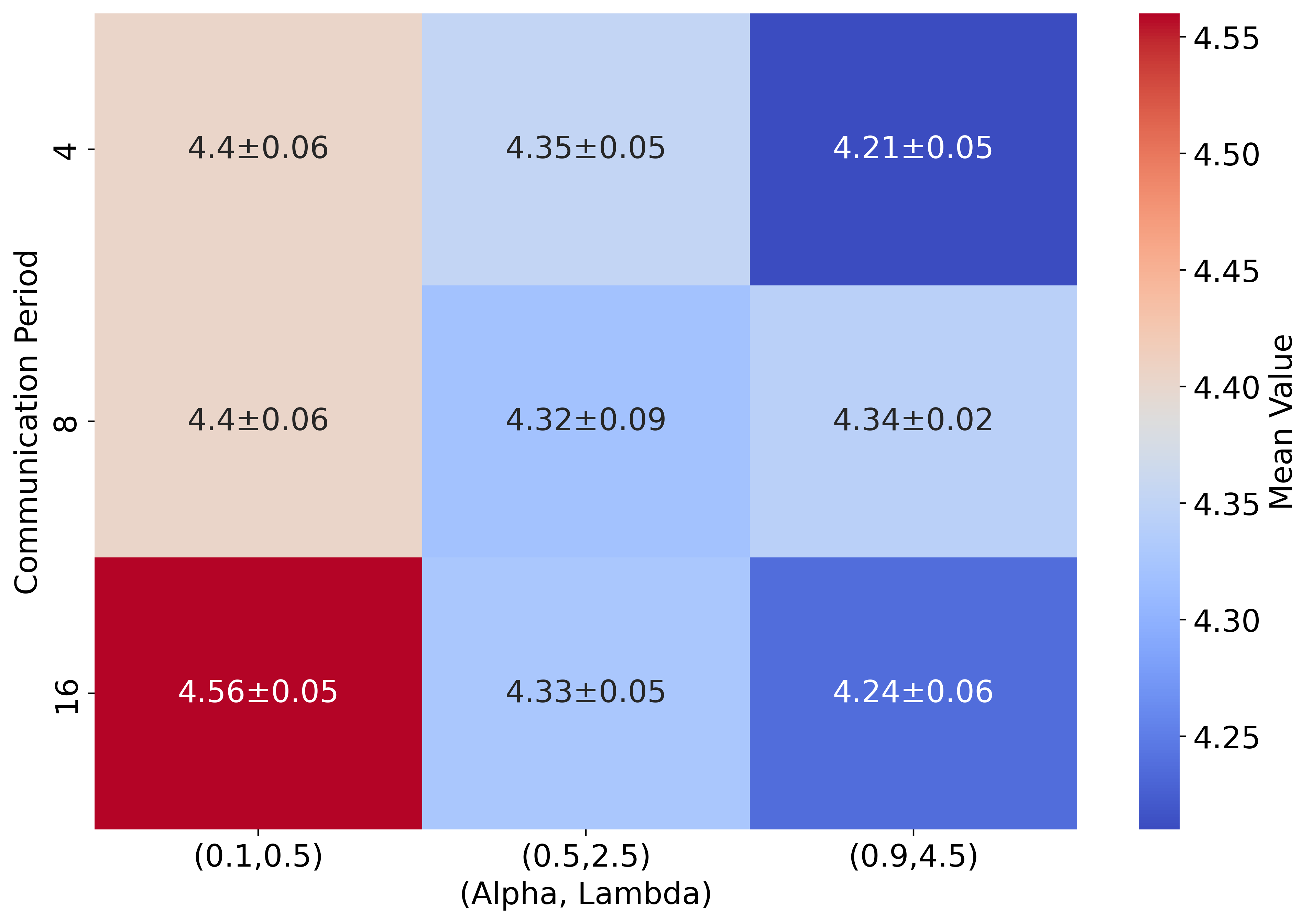}
    \caption{Training for 200 epochs}
    \label{appendix:hypersens:200epochs_rn18_c10}
  \end{subfigure}
  \hspace{5mm}
  \begin{subfigure}[t]{0.4\textwidth}
    \includegraphics[width=\linewidth]{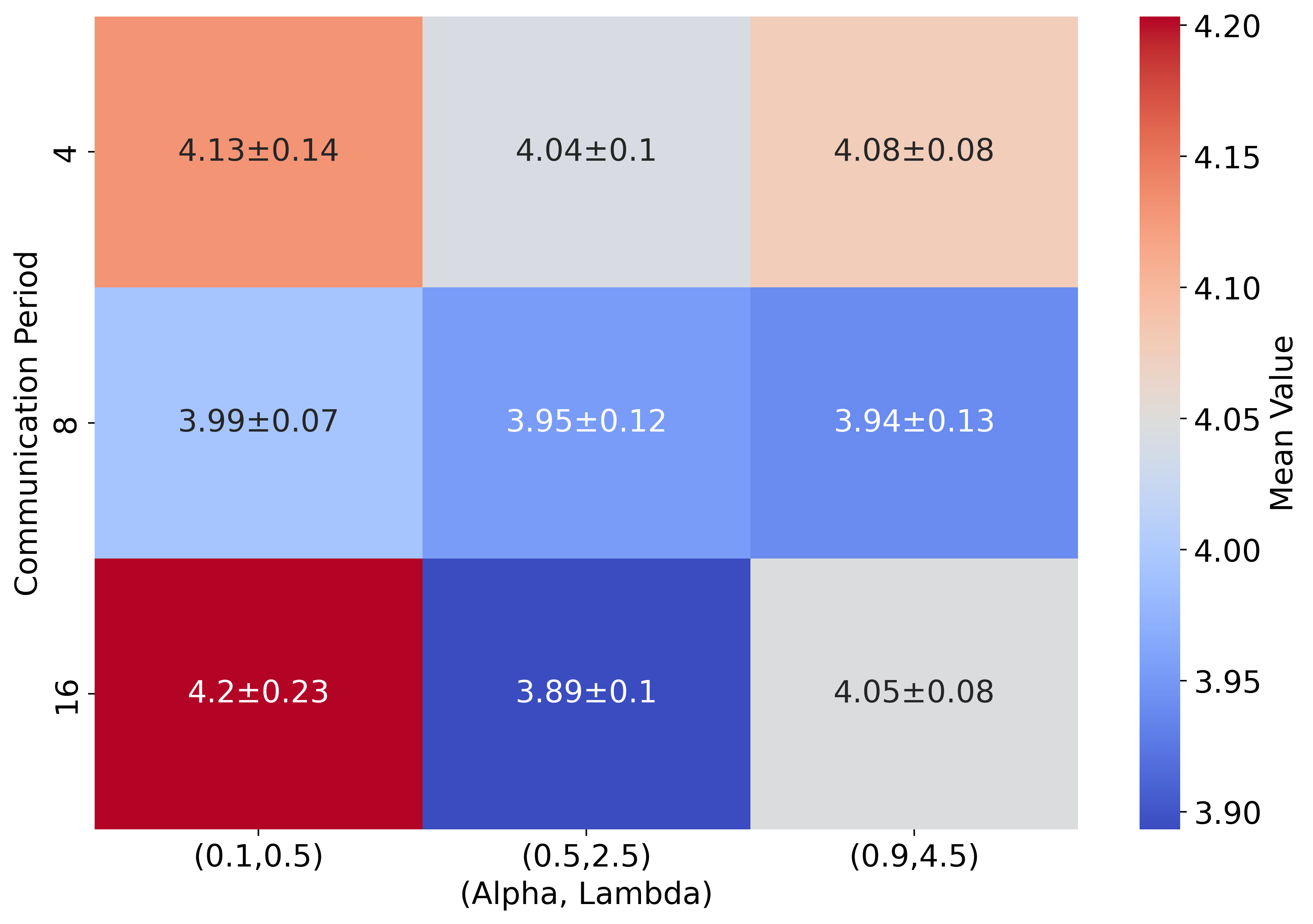}
    \caption{Training for 400 epochs}
    \label{appendix:hypersens:400epochs_rn18_c10}
  \end{subfigure}
  \caption{ResNet-18 training on CIFAR-10 dataset with 4 workers.}
  \label{appendix:hypersens:resnet18_c10}
\end{figure}

\begin{figure}[H]
  \centering
  \begin{subfigure}[t]{0.4\textwidth}
    \includegraphics[width=\linewidth]{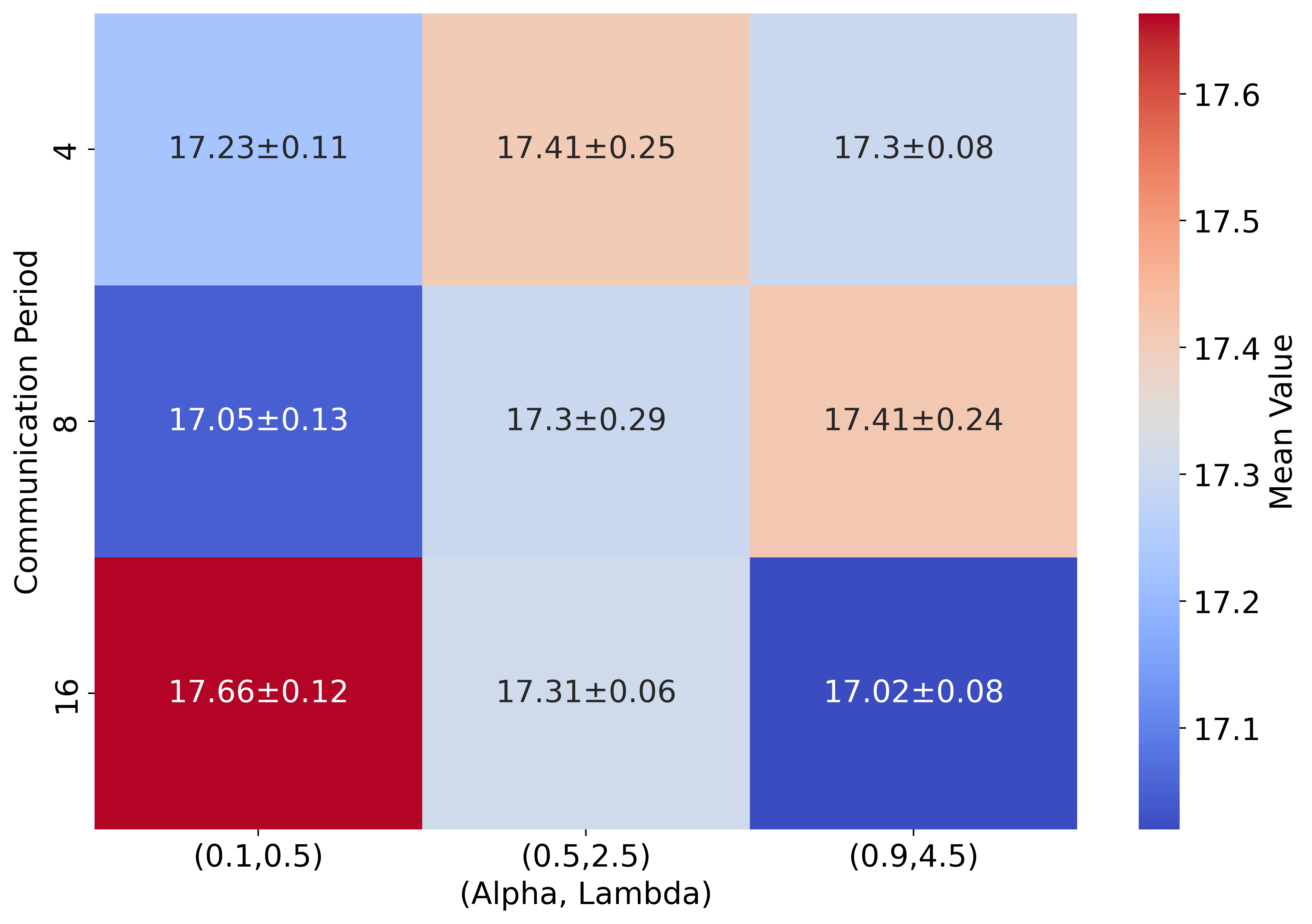}
    \caption{Training for 200 epochs}
    \label{appendix:hypersens:200epochs_pyramidnet_c100}
  \end{subfigure}
  \hspace{5mm}
  \begin{subfigure}[t]{0.4\textwidth}
    \includegraphics[width=\linewidth]{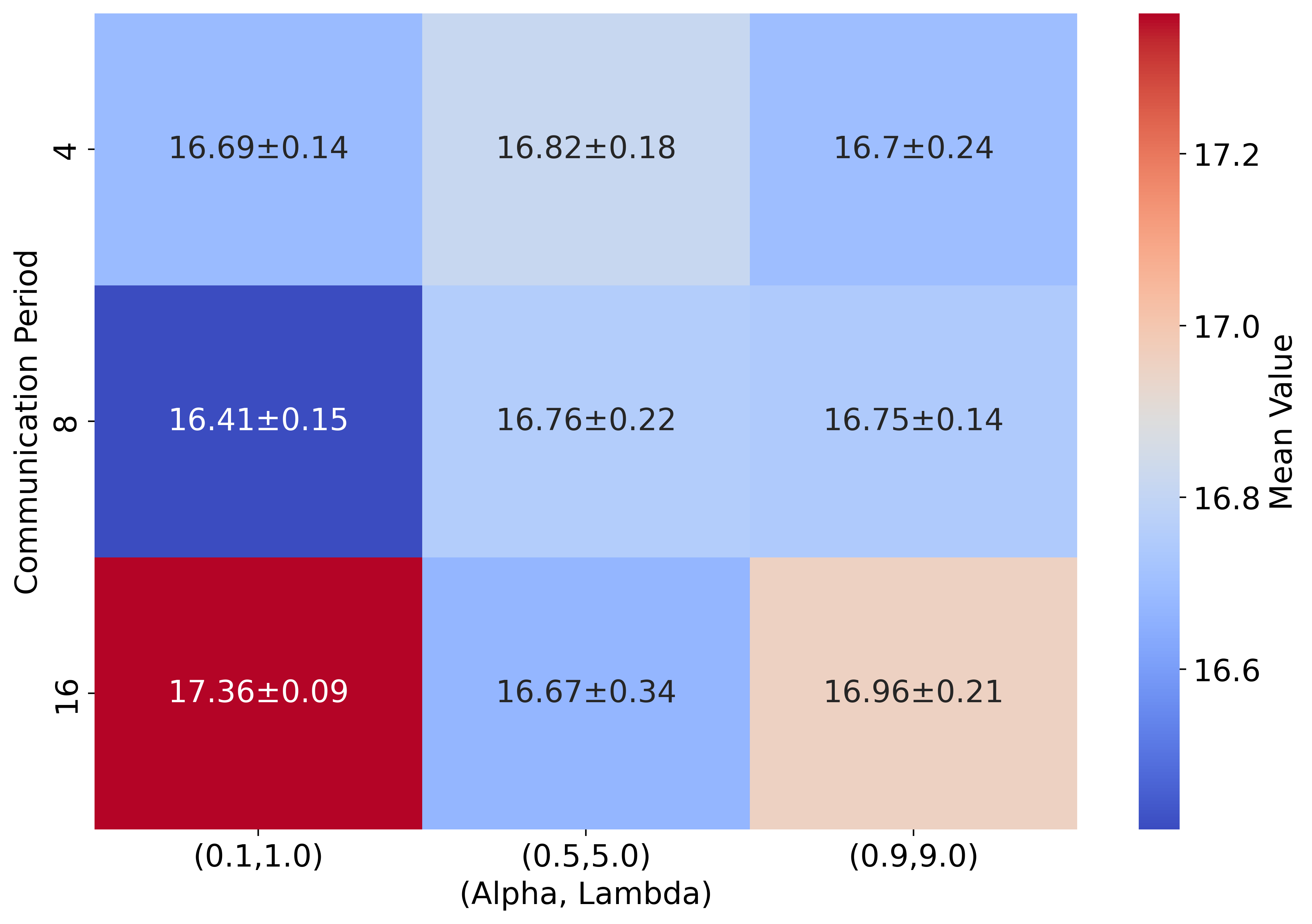}
    \caption{Training for 400 epochs}
    \label{appendix:hypersens:400epochs_pyramidnet_c100}
  \end{subfigure}
  \caption{PyramidNet training on CIFAR-100 dataset with 4 workers.}
  \label{appendix:hypersens:pyramidnet_c100}
\end{figure}

\textbf{ResNet-18, CIFAR-10:} In Figure~\ref{appendix:hypersens:resnet18_c10}, we observe that DPPF maintains a consistent level of performance across different communication periods and $(\lambda, \alpha)$ pairs. However, some trends emerge when comparing training durations: with shorter training (200 epochs), configurations with larger $\alpha$ values tend to yield better test error. In contrast, for longer training (400 epochs), intermediate values of $\alpha$ (e.g., $(0.5, 2.5)$) perform slightly better, suggesting that the optimal balance between push and pull may shift depending on the training duration.

\textbf{PyramidNet, CIFAR-100:} In Figure~\ref{appendix:hypersens:pyramidnet_c100}, we observe that shorter training with 200 epochs does not have a clear trend, whereas the final performance remains stable across different configurations. More importantly, even with only 200 epochs, DPPF consistently outperforms 400-epoch DDP training across all combinations of communication period and pull-push strengths by a significant margin. In the 400-epoch setting, the final performance remains relatively stable across different hyperparameter choices, with the exception of a few outliers under $\tau = 16$. This favorable reduction in sensitivity is likely due to the longer training duration and the over-parameterized nature of the PyramidNet model for CIFAR-100 data, which together enable robust convergence across a broader range of configurations.

\begin{figure}[H]
    \centering
    \includegraphics[width=0.4\linewidth]{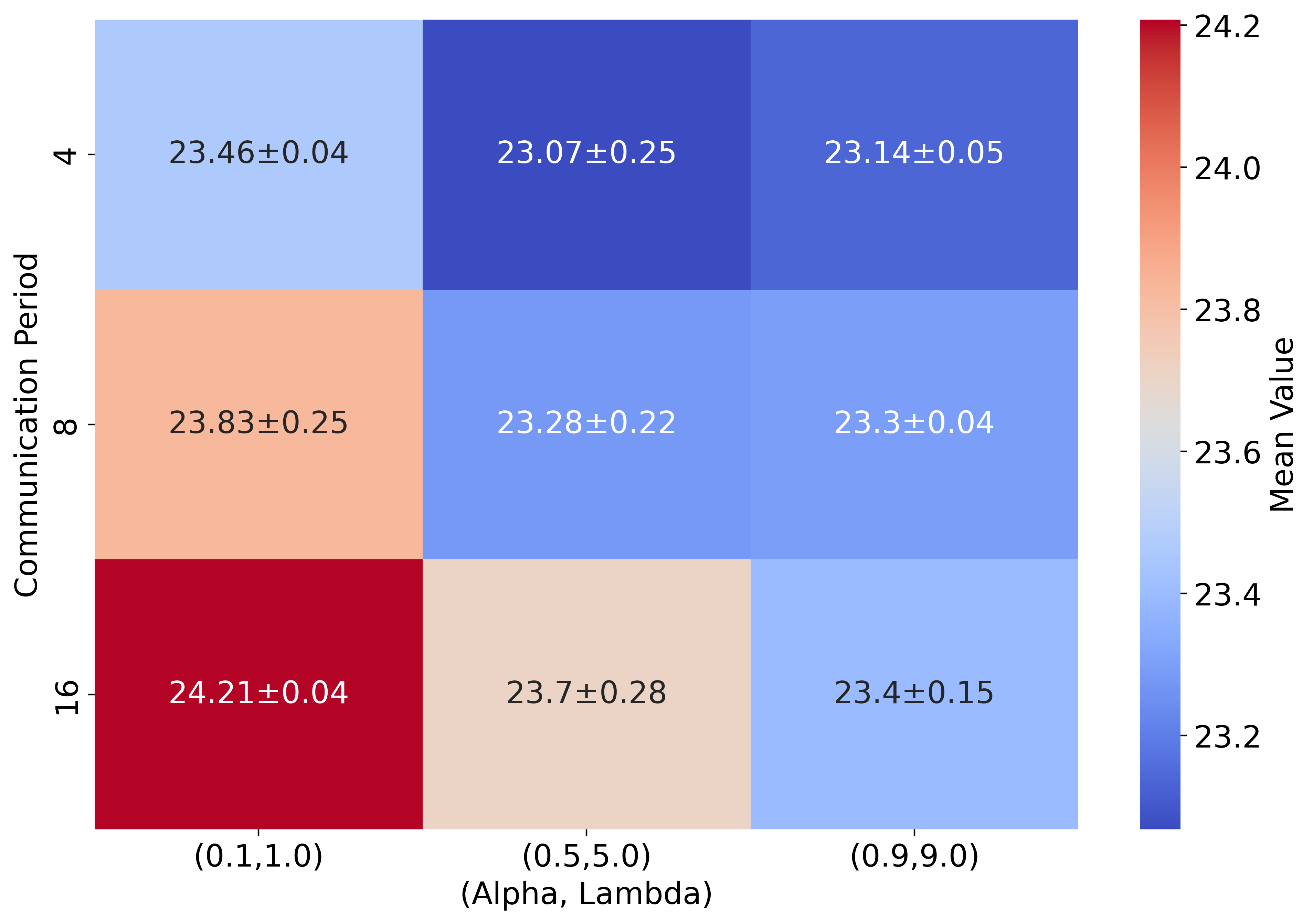}
    \caption{ResNet-50 training on ImageNet dataset with 4 workers}
    \label{appendix:hypersens:resnet50_imagenet}
\end{figure}

\textbf{ResNet-50, ImageNet:} In Figure~\ref{appendix:hypersens:resnet50_imagenet}, a clear trend emerges: this setting benefits from more frequent communication and stronger pull forces (i.e., larger $\alpha$). This suggests that tighter synchronization helps stabilize training for large-scale datasets like ImageNet, possibly due to the increased complexity and depth of ResNet-50. Despite this, 8 out of the 9 DPPF configurations either match or significantly outperform the baseline DDP-SGD in terms of test error, all while offering improved communication efficiency. These results highlight the robustness and practical value of DPPF, even under varied hyperparameter settings. It would be valuable to extend this analysis to larger models on ImageNet, such as ResNet-101, ResNet-152, or Vision Transformers \citep{dosovitskiy2020image}, to assess whether the observed sensitivity trends and performance gains scale with model size and architecture.

\subsection{Is DPPF Essentially an On-the-Fly SWA?}
\label{appendix:subsec:swa_comparison}
Due to the arrangement of workers in the loss landscape driven by DPPF's pushing force, one might wonder if DPPF is emulating Stochastic Weight Averaging (SWA) during training. In SWA \citep{izmailovaveraging}, it has been observed that the minima found by SGD often lie at the edges of the valley, rather than at its center. To address this, the authors propose continuing training after convergence, using a cyclical learning rate to explore different edges of the valley. The SGD solutions obtained are then averaged to locate the center of the valley (SWA solution) where more robust, better solutions lie \citep{izmailovaveraging}. Since all edge solutions and the SWA solution lie in the same basin, the trajectory between any of the two solutions does not encounter any loss barriers which is qualitatively verified in \citep{izmailovaveraging}. 

To address the question of whether DPPF\textsubscript{SimpleAVG} acts as an on-the-fly SWA for SimpleAVG, we check whether the last observation holds in our case. For this purpose, we plot how the loss and error (\%) along the trajectory between SimpleAVG and  DPPF\textsubscript{SimpleAVG} changes. Let $x_{sa}$ denote the solution from SimpleAVG and $x_{dppf}$ the solution obtained by DPPF\textsubscript{SimpleAVG}. We take a convex combination among these two solutions which can be expressed as $x_{C} = \alpha x_{dppf} + (1-\alpha) x_{sa}$ for $\alpha \in [0,1]$. For all $x_C$, we record the training loss, training error, test loss, and test error. The plots in \cref{appendix:fig:swa_check} show that there is a large loss barrier between the solutions of SimpleAVG and DPPF\textsubscript{SimpleAVG}. This implies that DPPF does not simply emulate SWA, and it encourages the recovery of separate, wider valleys with good-quality solutions.

\begin{figure}[h]
    \centering
    \begin{subfigure}{0.45\columnwidth} % Adjust width as needed
        \centering
        \includegraphics[width=\columnwidth]{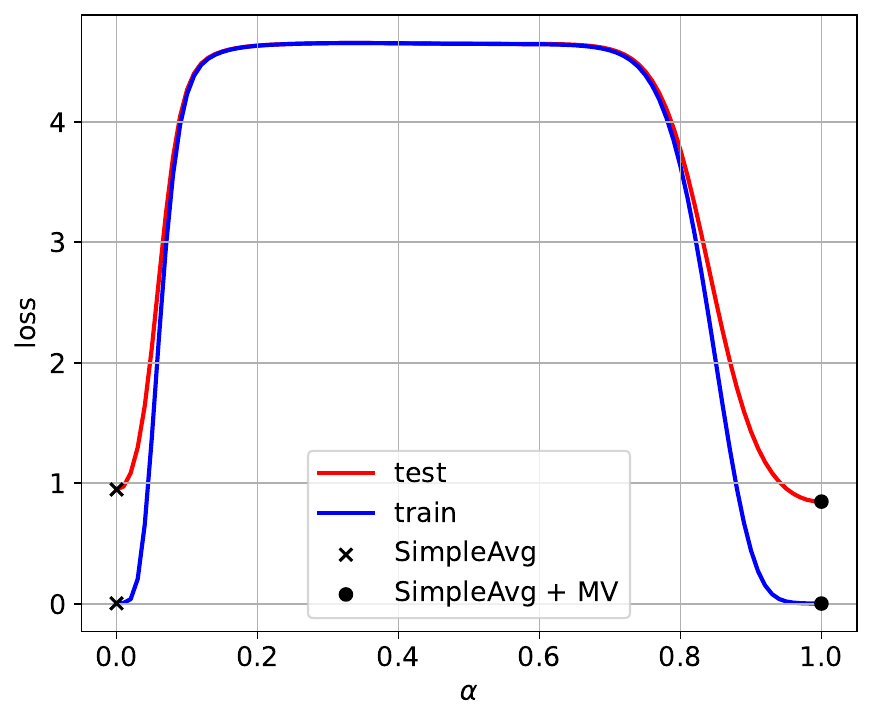}
        \caption{$\alpha$ vs. loss}
    \end{subfigure}
    \hfill
    \begin{subfigure}{0.45\columnwidth} % Adjust width as needed
        \centering
        \includegraphics[width=\columnwidth]{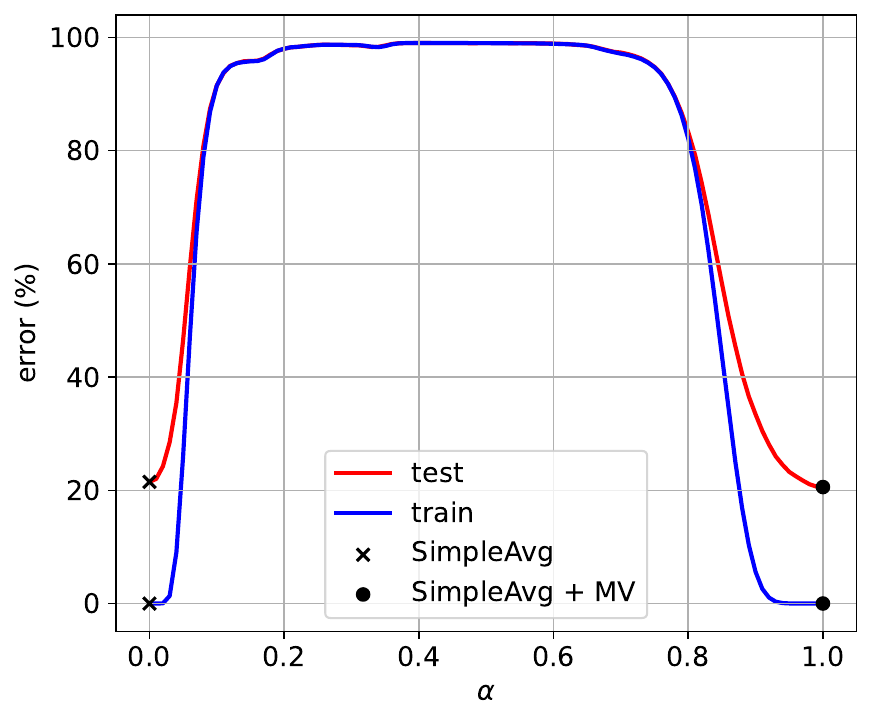}
        \caption{$\alpha$ vs. error (\%)}
    \end{subfigure}
    \caption{Change in training loss, training error, test loss, and test error along the trajectory between SimpleAVG and DPPF\textsubscript{SimpleAVG} solutions.}
    \label{appendix:fig:swa_check}
\end{figure}

\section{Full Proofs of Theoretical Analysis}
\label{appendix:sec:all_theory}

\subsection{Derivation of the Update Rule by Minimizing the Relaxed Inv. MV Term}
\label{appendix:subsec:mv_update_rule}

Recall that the relaxed Inv. MV measure is mathematically expressed as follows:

\begin{equation*}
R = - \frac{1}{M} \sum_{i=1}^M \| x_{i} - x_A\|_2 \hspace{2mm} \text{where} \hspace{2mm}  x_A = \frac{1}{M} \sum_{i=1}^M x_{i} 
\end{equation*}

Now, we want to derive the gradient of the term with respect to worker $m$, i.e. $\frac{\partial R}{\partial x_m}$. Let us also write $d_i = x_i - x_A$ so we can also write:

\begin{equation*}
R = - \frac{1}{M} \sum_{i=1}^M \sqrt{ d_i^Td_i } \hspace{2mm} \text{where} \hspace{2mm} d_i = x_i - \frac{1}{M} \sum_{j=1}^M x_{j}
\end{equation*}

We have two different cases to consider to find $\frac{\partial R}{\partial x_m}$. This is because the calculation of $x_A$ includes all the model parameters hence, $x_m$ is present in all terms of $R$'s summation due to $x_A$. Besides, $i=m$ includes the copy of $x_m$ itself inside the norm. These two cases require separate treatment to derive the update rule accurately.

\textbf{ 1) $i = m$:} We start by using the chain rule:

\begin{equation*}
\frac{\partial R}{\partial x_i} = \frac{\partial R}{\partial \sqrt{d_i^Td_i}}  \frac{\partial \sqrt{d_i^Td_i}}{\partial d_i} \frac{\partial d_i}{\partial x_i}
\end{equation*}

The individual terms in the chain rule are equal to the following:
\begin{equation*}
\frac{\partial R}{\partial \sqrt{d_i^Td_i}} = \frac{-1}{M}, \hspace{2mm}  \frac{\partial \sqrt{d_i^Td_i}}{\partial d_i} = \frac{d_i} {\sqrt{d_i^Td_i}}, \hspace{2mm} \frac{\partial d_i}{\partial x_i} = \frac{M-1}{M}
\end{equation*}

Plugging in the expression of the individual partial derivatives in the chain rule and re-writing $d_i = x_i - x_A$ we obtain the following:

\begin{equation}
\label{appendix:eq:derivation_1}
\frac{\partial R}{\partial \sqrt{d_i^Td_i}}  \frac{\partial \sqrt{d_i^Td_i}}{\partial d_i} \frac{\partial d_i}{\partial x_m} = - \frac{M-1}{M^2} \frac{x_m - x_A}{\| x_m - x_A \|_2}
\end{equation}

This concludes the part of the update that arises from the case $m=i$.

\textbf{ 2) $i \neq m$:} We again start by using the chain rule:

\begin{equation*}
\frac{\partial R}{\partial x_i} = \frac{\partial R}{\partial \sqrt{d_i^Td_i}}  \frac{\partial \sqrt{d_i^Td_i}}{\partial d_i} \frac{\partial d_i}{\partial x_i}
\end{equation*}

Again we separately express the partial derivatives in the chain rule:
\begin{equation*}
\frac{\partial R}{\partial \sqrt{d_i^Td_i}} = \frac{-1}{M}, \hspace{2mm}  \frac{\partial \sqrt{d_i^Td_i}}{\partial d_i} = \frac{d_i} {\sqrt{d_i^Td_i}}, \hspace{2mm} \frac{\partial d_i}{\partial x_i} = \frac{-1}{M}
\end{equation*}

Notice that this time $ \frac{\partial d_i}{\partial x_i}$ is different since $x_m$ only participates in $d_i$ due to the presence of $x_A$. Putting everything together we get the following for all $i \neq m$:

\begin{equation}
\label{appendix:eq:derivation_2}
\frac{\partial R}{\partial \sqrt{d_i^Td_i}}  \frac{\partial \sqrt{d_i^Td_i}}{\partial d_i} \frac{\partial d_i}{\partial x_m} = \frac{1}{M^2} \frac{x_i - x_A}{\| x_i - x_A \|_2}
\end{equation}

Finally, by combining \ref{appendix:eq:derivation_1} and \ref{appendix:eq:derivation_2}, we obtain the following overall update:
\begin{equation*}
    \frac{\partial R}{\partial x_m} = - \frac{1}{M^2} \left ( (M-1)\frac{d_m}{\|d_m\|}  - \sum_{j \neq m}^{M} \frac{ d_j }{\|d_j\|} \right ) = - \frac{1}{M^2} \left ( M\frac{d_m}{\|d_m\|}  - \sum_{j=1}^{M} \frac{ d_j }{\|d_j\|} \right ) .
\end{equation*}

This concludes the derivation.

\subsection{Valley Width and Generalization Guarantees}
\label{appendix:subsec:generalization_proof}

Consider a distributed training of a DNN performed with $M$ workers by optimizing the loss function $f$. We define the following notations for communication round $k$ and worker $m$:
\begin{align*}
  x_{m,k}, x_{m,k}^+      &= \text{parameters before and after distributed update},\\
  x_{A,k}, x_{A,k}^+      &= \frac1M\sum_{j=1}^{M}x_{j,k}, \frac1M\sum_{j=1}^{M}x_{j,k}^+\quad\text{average variables before and after distributed update},\\
  \Delta_{m,k} &= x_{A,k}-x_{m,k}\quad\text{gap vector before distributed update},\\
  \Delta_{m,k}^+ &= x_{A,k}^+-x_{m,k}^+\quad\text{gap vector after distributed update},\\
  r_{m,k}      &= \lVert \Delta_{m,k}\rVert,\quad
  u_{m,k}      = \Delta_{m,k}/r_{m,k}\quad(\lVert u_{m,k}\rVert=1),\\
  C            &= 1-\alpha\quad(0<C<1).
\end{align*}

We also assume that the workers observe independent and unbiased gradients ($g$'s) with bounded variance during the training:
\begin{equation*}
    \mathbb{E} \| g_{m,k}^t - \nabla f(x_{m,k}^t) \|^2 \leq \sigma_0^2  \;\;\; \text{for all} \;\;\; m,k,t.
\end{equation*}

\begin{lemma} \label{appendix:lemma:variance} Let $G_{j,k} = \sum_{t=1}^\tau g_{j,k}^t$ be the sum of local gradients calculated by worker $j$ between communication rounds $k$ and $k+1$. Also, $\bar{G}_{k} = \frac{1}{M} \sum_{i=1}^M G_{i,k}$ denotes the average of the local gradients across workers. Based on the bounded variance assumption of the gradients, we have:
\begin{equation*}
\mathbb{E}\bigl[\|G_{m,k}-\bar G_k\|\bigr]
  \le
  \sqrt{\tfrac{M+1}{M}}
  \sqrt{\tau}\sigma_{0}
\end{equation*}
\end{lemma}

\begin{proof} Using the definition that $G_{m,k} = \sum_{t=1}^\tau g_{m,k}^t$, we write the following:
\begin{equation*}
\bar G_k = \frac{1}{M}\sum_{j=1}^{M} G_{j,k},
\qquad
\mathbb{E}\bigl[\|\bar G_k\|^{2}\bigr]
   = \frac{1}{M^{2}}\sum_{j=1}^{M}
          \mathbb{E}\bigl[\|G_{j,k}\|^{2}\bigr]
   = \frac{\tau\sigma_{0}^{2}}{M}.
\end{equation*}

\begin{equation*}
\begin{aligned}
\mathbb{E}\bigl[\|G_{m,k}-\bar G_k\|^{2}\bigr]
  &= \mathbb{E}\bigl[\|G_{m,k}\|^{2}\bigr]
     + \mathbb{E}\bigl[\|\bar G_k\|^{2}\bigr]
     - 2\mathbb{E}\langle G_{m,k},\bar G_k\rangle \\[4pt]
  &= \tau\sigma_{0}^{2}
     + \frac{\tau\sigma_{0}^{2}}{M}
     \quad (\text{cross term }=0) \\[4pt]
  &= \tau\sigma_{0}^{2}\!\Bigl(1+\frac{1}{M}\Bigr).
\end{aligned}
\end{equation*}

Finally, by using the Cauchy-Schwarz, we can bound the first moment as:
\begin{equation*}
\mathbb{E}\bigl[\|G_{m,k}-\bar G_k\|\bigr]
  \le
  \sqrt{\mathbb{E}\bigl[\|G_{m,k}-\bar G_k\|^{2}\bigr]}
  =
  \sqrt{\tfrac{M+1}{M}}
  \sqrt{\tau}\sigma_{0}
\end{equation*}

\end{proof}

\begin{lemma} \label{appendix:lemma:avg_unit_normed} Let us assume that $u_i \in \mathbb{R}^d$ is a column vector drawn independently and uniformly from the unit sphere $\mathbb{S}^{d-1}$ so that the norms are $1$ surely. For $\bar{u}_{k} = \frac{1}{M} \sum_{i=1}^M u_{i,k}$, we can write the following upper bound:
    \begin{equation*}
    \mathbb{E}\left[\left\| \bar{u}_k \right\|_2\right] \leq \frac{1}{\sqrt{M}}
\end{equation*}
\end{lemma}

\begin{proof}
We start by writing:

\begin{align*}
\mathbb{E}\left[\left\| \bar{u}_k \right\|_2^2\right] 
&= \mathbb{E}\left[\bar{u}_k^\top \bar{u}_k \right] \\
&= \frac{1}{M^2} \sum_{i,j=1}^{M} \mathbb{E}\left[u_{i,k}^\top u_{j,k} \right] \\
&= \frac{1}{M^2} \left( \sum_{i=1}^{M} \mathbb{E}\left[\left\| u_{i,k} \right\|_2^2\right] 
+ \sum_{ i \neq j}^{M} \mathbb{E}\left[u_{i,k}^\top u_{j,k} \right] \right).
\end{align*}

Since $\mathbb{E}\left[\left\| u_{i,k} \right\|_2^2\right] = 1$ by construction and the cross terms are $0$ due to $\mathbb{E}[u_{i,k}] = \mathbf{0}$ and independence, we arrive at $\mathbb{E}\left[\left\| \bar{u}_k \right\|_2^2\right] = \frac{1}{M}$. Then, using Cauchy-Schwarz, we get:

\begin{equation*}
    \mathbb{E}\left[\left\| \bar{u}_k \right\|_2\right] \leq \sqrt{\mathbb{E}\left[\left\| \bar{u}_k \right\|_2^2\right]}  = \frac{1}{\sqrt{M}}
\end{equation*}
\end{proof}

\begin{theorem}[\textbf{Valley Width under DPPF}] \label{appendix:theorem:valley_width} Consider a distributed training of a DNN performed by $M$ workers using DPPF with pull and push force strengths $\alpha$ and $\lambda$ respectively. The communication period is $\tau$, and the workers locally run SGD with learning rate $\eta$. Denote the worker index, communication round, and local iteration with $m,k$, and $t$ respectively. We assume bounded variance: $\mathbb{E} \| g_{m,k}^t - \nabla f(x_{m,k}^t) \|^2 \leq \sigma_0^2  \;\;\text{for all} \;\; m,k,t$. Now, define the gap vector $\Delta_{m,k}^+= x_{A,k}^+ - x_{m,k}^+$ as the distance between worker $m$ and the worker average after the communication update $k$. Asymptotically, we obtain:
\[
\boxed{
\lim_{k\to\infty}\mathbb{E}\bigl[\lVert\Delta_{m,k}^{+}\rVert\bigr] = \frac{\lambda}{\alpha} + \mathcal{O}\!\Bigl(\eta\sigma_{0}+\tfrac{1}{\sqrt{M}}\Bigr)}
\]
Furthermore, with $\eta\to0$ and many workers $M \gg 1$, we have  \(
\displaystyle
\lim_{k\to\infty}\mathbb{E}\bigl[\lVert\Delta_{m,k}^{+}\rVert\bigr]
      =\frac{\lambda}{\alpha}.
\)
\end{theorem}

\begin{proof}
    We start by expressing how the gap changes after the distributed pull-push update.

\begin{align}
x_{m,k}^+ &= x_{m,k} + \Delta_{m,k} \left( \alpha - \lambda \frac{1}{\|\Delta_{m,k}\|} \right) \label{equation:gap_dynamics:1} \\
x_{A,k}^+ &= \frac{1}{M} \sum_{i=1}^M x_{i,k}^+ 
= \underbrace{\frac{1}{M} \sum_{i=1}^M x_{i,k}}_{x_{A,k}}+ \alpha \underbrace{\frac{1}{M} \sum_{i=1}^M \Delta_{i,k}}_{0} - \frac{\lambda}{M} \sum_{i=1}^M u_{i,k} \label{equation:gap_dynamics:2}
\end{align}

By subtracting Equation~\ref{equation:gap_dynamics:1} from Equation~\ref{equation:gap_dynamics:2}, we can write:

\begin{equation}
    \Delta_{m,k}^+ = (1 - \alpha)\Delta_{m,k} + \lambda u_{m,k} - \frac{\lambda}{M} \sum_{i=1}^M u_{i,k} \label{equation:gap_dynamics:3}
\end{equation}

However, we want to extract the recurrence between how the post-update gap changes. To this end, we need to relate $\Delta_{m,k-1}^+$ term to $\Delta_{m,k}^+$. Let us consider the local gradient steps between communication rounds $k-1$ and $k$.

\begin{equation}
    x_{m,k} = x_{m,k-1}^+ - \eta G_{m,k-1} \;\;\; \text{where} \;\;\; G_{m,k-1} = \sum_{t=1}^\tau g_{m,k-1}^t \label{equation:gap_dynamics:4}
\end{equation}

Here $g_{m,k-1}^t$ is the gradient update that worker $m$ takes at $t^\text{th}$ local iteration after communication round $k-1$. And we use $G_{m,k-1}$ to denote the cumulative local updates. Let us also express:

\begin{equation}
x_{A,k} = \frac{1}{M} \sum_{i=1}^M x_{i,k} = \underbrace{\frac{1}{M} \sum_{i=1}^M x_{i,k-1}^+}_{x_{A,k-1}^+} - \frac{\eta}{M} \sum_{i=1}^M G_{i,k-1} \label{equation:gap_dynamics:5}
\end{equation}

Let us all define $\bar{G}_{k-1} = \frac{1}{M} \sum_{i=1}^M G_{i,k-1}$ to denote the average of the local gradients across workers. By subtracting Equation~\ref{equation:gap_dynamics:4} from Equation~\ref{equation:gap_dynamics:5}, we arrive at:

\begin{align}
x_{A,k} - x_{m,k} &= x_{A,k-1}^+ - x_{m,k-1}^+ - \eta \left( \bar{G}_{k-1} - G_{m,k-1} \right) \\
\Delta_{m,k} &= \Delta_{m,k-1}^+ - \eta Z_{m,k-1} \;\;\;  \text{where} \;\;\;  Z_{m,k-1} = \bar{G}_{k-1} - G_{m,k-1} \label{equation:gap_dynamics:6}
\end{align}

By plugging Equation~\ref{equation:gap_dynamics:6} into Equation~\ref{equation:gap_dynamics:3}, we can obtain the following recurrence relation:

\begin{equation}
 \Delta_{m,k}^+ = (1 - \alpha)\Delta_{m,k-1}^+ -\eta (1-\alpha)Z_{m,k-1}+ \lambda u_{m,k} - \lambda \bar{u}_{k}  \label{equation:gap_dynamics:7}
\end{equation}

where we expressed $\frac{1}{M} \sum_{i=1}^M u_{i,k}$ as $\bar{u}_{k}$. We now start from the recurrence:

\begin{equation}
\Delta_{m,k}^+ = C \Delta_{m,k-1}^+ - \eta C Z_{m,k-1} + \lambda u_{m,k} - \lambda \bar{u}_{k}
\end{equation}
Taking norms and applying the triangle inequality:
\begin{align}
\|\Delta_{m,k}^+\| 
&\leq C \|\Delta_{m,k-1}^+\| + \eta C \|Z_{m,k-1}\| + \lambda \|u_{m,k} - \bar{u}_k\| \\
&\leq C \|\Delta_{m,k-1}^+\| + \eta C \|Z_{m,k-1}\| + \lambda (1 + \|\bar{u}_k\|) \label{eq:gapnorm_bound}
\end{align}
where we used that $\|u_{m,k}\| = 1$ by construction. For ease of notation, let us define $r_{k+1} = \mathbb{E}\!\bigl[\lVert\Delta_{m,k+1}^{+}\rVert\bigr]$. We express $\mathbb{E}Z_{m,k}$ using Lemma~\ref{appendix:lemma:variance}. Additionally, we treat $u_i \in \mathbb{R}^d$ as a column vector drawn uniformly from the unit sphere $\mathbb{S}^{d-1}$ so that the norms are $1$ surely. Also, $u_i$'s are drawn independently, a valid assumption as the independent stochastic noise determines the location of $x_{i,k}$. By invoking Lemma~\ref{appendix:lemma:avg_unit_normed} on $\|\bar{u}_k\|$, we write the following:

\begin{align}
r_{k+1}
&=\mathbb{E}\bigl[\lVert\Delta_{m,k+1}^{+}\rVert\bigr]  \\[2pt]
&\le
Cr_{k}+
\eta C\underbrace{\mathbb{E}\lVert Z_{m,k}\rVert}_{\le\sqrt{\tau}\sigma_{0}\sqrt{\tfrac{M+1}{M}}}
+
\lambda\!\left(1+\underbrace{\mathbb{E}\lVert\bar u_{k+1}\rVert}_{\le1/\sqrt{M}}\right)
\notag \\[2pt]
&=Cr_{k}+\beta+\gamma,
\label{eq:one‑step}
\end{align}
where
\[
\beta =\eta C\sqrt{\tau}\sigma_{0}\sqrt{\tfrac{M+1}{M}},
\qquad
\gamma =\lambda\!\left(1+\tfrac{1}{\sqrt{M}}\right)
\]

Iterating \eqref{eq:one‑step} for $k$ steps yields
\begin{align}
r_{k}
&\le
C^{k}r_{0}+
(\beta+\gamma)\sum_{j=0}^{k-1}C^{j}
\notag\\[2pt]
&=
C^{k}r_{0}+
\frac{\beta+\gamma}{1-C}\bigl(1-C^{k}\bigr)
=
C^{k}r_{0}+\frac{\beta+\gamma}{\alpha}\bigl(1-C^{k}\bigr).
\label{eq:rk‑general}
\end{align}

Because $r_0=0$ and $C^{k}\!\to0$, taking the limit in \eqref{eq:rk‑general} with respect to $k$ gives
\begin{equation}
\limsup_{k\to\infty} r_{k}
\le
\frac{\beta+\gamma}{\alpha}
=
\frac{\eta(1-\alpha)\sqrt{\tau}\sigma_{0}\sqrt{\tfrac{M+1}{M}}}{\alpha}
+
\frac{\lambda}{\alpha}\!\left(1+\tfrac{1}{\sqrt{M}}\right).
\label{eq:limsup}
\end{equation}

Consequently, with diminishing learning rate $\eta\to0$ and many workers $M \gg 1$, we obtain
\[
\boxed{
\lim_{k\to\infty}\mathbb{E}\bigl[\lVert\Delta_{m,k}^{+}\rVert\bigr]
=
\frac{\lambda}{\alpha}
+
\mathcal{O}\!\Bigl(\eta\sigma_{0}+\tfrac{1}{\sqrt{M}}\Bigr)
}
\]
which completes the proof.
\end{proof}
% \paragraph{Tightness of the bound.}
% We write the exact recursion as $r_{k+1}=Cr_{k}+\lambda+\xi_{k}$ with $\mathbb{E}[\xi_{k}]=0$ and $|\xi_{k}|\le\eta_k C\sqrt{\tau}\sigma_0\sqrt{(M+1)/M}+\lambda/\sqrt{M}$.
% Choosing a diminishing step size $\eta_k\!\to\!0$ or letting
% $M\!\to\!\infty$ makes $\xi_{k}\!\xrightarrow{k\to\infty}\!0$.
% Since $|r_{k+1}-\lambda/\alpha|
%       \le C^{k+1}|r_{0}-\lambda/\alpha|
%       +\sum_{j=0}^{k}C^{j}|\xi_{k-j}|$,
% both terms on the right vanish and the sequence converges:
% \[
% \boxed{\displaystyle
% \lim_{k\to\infty}\mathbb{E}\bigl[\lVert\Delta_{m,k}^{+}\rVert\bigr]
%       =\frac{\lambda}{\alpha}.}
% \]
% Thus the upper bound is tight.
% \end{proof}

\begin{theorem}[Monotone PAC-Bayes Gap Tightening with Valley Radius]
\label{thm:valley_width_pac_bayes}
Consider a geometric grid for candidate valley sizes governed by the DPPF algorithm's pull ($\alpha_j$) and push ($\lambda_j$) strengths: $\mathcal G=\{r_j=r_{\min}(1+\gamma)^j\}_{j=0}^J$, where each $r_j=\frac{\lambda_j}{\alpha_j}$. For every $r_j$ assume the spherical‑Gaussian prior $P_{r_j}=\mathcal N(0,r_j\sigma_0^2I_d)$ and let the training algorithm return the posterior  $Q_{r_j}=\mathcal N(\mu_{r_j},c_{r_j}r_j\sigma_0^2I_d)$ over model parameters, where $c_{r_j}\ge1$ is a data-dependent scalar. Assume there are constants $D_0\!>\!0$ and $0\le\beta<1$ such that $\lVert\mu_{r_j}\rVert_2^2\le D_0\,r_j^{\beta}$ for every $r_j\in\mathcal G$. Then with probability $1-\delta$ over the draw of the sample set $S$ with $|S| = n$, for all $r_j\in\mathcal G$, we can write:
\[
\mathbb E_{x\sim Q_{r_j}}\![L_{\mathcal D}(x)] \le
\mathbb E_{x\sim Q_{r_j}}\![L_S(x)] +
\underbrace{\sqrt{\frac{\,
      \tfrac{d}{2}(c_{r_j}-1-\log c_{r_j})
      +\frac{D_0}{2\sigma_0^2 r_j^{1-\beta}}
      +\log\frac{nJ}{\delta}}
      {2(n-1)}}}_{\displaystyle\textstyle \text{gap}(r_j)}.
\]
because $1-\beta>0$, $\text{gap}(r_{j+1})<\text{gap}(r_j)$ for every consecutive pair in $\mathcal G$.
\end{theorem}

\begin{proof} We base our starting bound on the analysis carried out in \citep{chatterji2019intriguing} and \citep{SAM}. Following their footsteps, we use the PAC-Bayes bound derived for DNNs by \citep{dziugaite2017computing} by building upon \citep{mcallester1999pac}. Then, for any prior distribution $P$ over parameters in $d$ dimensional space, and for any posterior distribution $Q$, the following generalization guarantee holds with probability at least $1-\delta$ over the random draw of the training set $S$ with $n$ samples:
\begin{equation}
      \mathbb{E}_{x\sim Q}\![ L_{\mathcal D}(x)] \le \mathbb{E}_{x\sim Q}\![L_{S}(x)]
  + \sqrt{\frac{ KL(Q || P) + \log \frac{n}{\delta} }{2(n-1)}}.
  \label{appendix:generic_pac}
\end{equation}

If we assume that prior $P$ and posterior $Q$ are isotropic Gaussians, i.e. \( P \sim \mathcal{N}(\mu_P, \sigma_P^2 I_d) \) and \( Q \sim \mathcal{N}(\mu_Q, \sigma_Q^2 I_d) \) we can simply express the $KL$ divergence as follows:
\[
   KL(Q || P) = \frac{d}{2} \left ( \frac{\sigma_Q^2}{\sigma_P^2} -1 + \log \frac{\sigma_P^2}{\sigma_Q^2} \right ) + \frac{\| \mu_Q - \mu_P \|_2^2}{2\sigma_P^2} 
\]

Let $r$ be the valley width which asymptotically converges to the ratio between push and pull force strengths as characterized by Theorem~\ref{appendix:theorem:valley_width}, i.e. $r = \frac{\lambda}{\alpha}$. Since, $\lambda$ and $\alpha$ hyperparameters are set without seeing the data, we can shape the prior variance based on $r$. Let us consider the following isotropic Gaussians for prior and posterior: \( P \sim \mathcal{N}(0, r\sigma_0^2 I_d) \) and \( Q \sim \mathcal{N}(\mu, cr\sigma_0^2 I_d) \) where $c$ is the data-dependent coefficient that impacts the posterior, we can re-write the PAC-Bayes bound in \ref{appendix:generic_pac} as follows:
\begin{equation*}
      \mathbb{E}_{x\sim Q}\![ L_{\mathcal D}(x)] \le \mathbb{E}_{x\sim Q}\![L_{S}(x)] + 
      \sqrt{\frac{ \frac{d}{2} \left ( c -1 - \log c \right ) + 
      \frac{\| \mu \|_2^2}{2r\sigma_0^2} + \log \frac{n}{\delta} }{2(n-1)}}.
  \label{appendix:generic_pac_iso}
\end{equation*}

It may be tempting to conclude that increasing $r$ tightens the generalization gap by $1/\sqrt{r}$ based on the expression on the right-hand side above. However, our current assumptions do not account for how $\|\mu\|_2^2$ changes with increased prior variance. Without additional guarantees on the behavior of $\|\mu_Q\|_2^2$, such a claim would be misleading or incomplete.

To address this incompleteness, we employ the Langford-Caruana grid. We start by declaring a geometric grid of valley widths \(\mathcal G\):
\[
  \mathcal G
  =
  \bigl\{r_j=r_{\min}(1+\gamma)^j\big|j=0,\dots,J\bigr\},
  \qquad
  J=\Bigl\lceil\log_{1+\gamma}\!\frac{r_{\max}}{r_{\min}}\Bigr\rceil,
\]
with \(r_{\min},r_{\max},\gamma>0\) chosen a priori. For each \(r_j\) we attach a Gaussian prior \(P_{r_j}=\mathcal N(0,r_j\sigma_0^2I_d)\). Now, running the algorithm with the target width \(r_j\) yields the following posterior \( Q_{r_j}=\mathcal N\bigl(\mu_{r_j},c_{r_j}r_j\sigma_0^2I_d \bigr ) \) where \(c_{r_j}\ge 1\) is again data‑dependent.

Applying \eqref{appendix:generic_pac} to every \((P_{r_j},Q_{r_j})\) pair from the grid and
union‑bounding over the \(J\) choices as practiced in Langford–Caruana method \citep{langford2001caruna} gives, with probability at least \(1-\delta\):
\begin{equation}
\mathbb E_{x\sim Q_{r_j}}\bigl[L_{\mathcal D}(x)\bigr] \le
\mathbb E_{x\sim Q_{r_j}}\bigl[L_S(x)\bigr] +
\underbrace{\sqrt{\frac{
        \tfrac{d}{2}\bigl(c_{r_j}-1-\log c_{r_j}\bigr)
       +\dfrac{\lVert\mu_{r_j}\rVert_2^2}{2r_j\sigma_0^{2}}
       +\log\tfrac{nJ}{\delta}}
       {2(n-1)}}}_{\text{gap}(r_j)}
\label{eq:grid_bound}
\end{equation}

We assume \textit{bounded-drift} that sets an upper limit on the $L2$ norm of the posterior mean based on a chosen valley width. Assume that for some constants \(D_0\ge 0\) and \(0\le\beta<1\)
\begin{equation}
  \|\mu_{r}\|_2^{2}\leq D_0r^{\beta},
  \qquad\text{for all }r\in[r_{\min},r_{\max}].
  \label{eq:bounded_drift}
\end{equation}

Substituting \eqref{eq:bounded_drift} into the numerator of the $\text{gap}(r_j)$ term in \eqref{eq:grid_bound} yields
\[
  \text{gap}(r_j) = \sqrt{\frac{
       \tfrac{d}{2}\bigl(c_{r_j}-1-\log c_{r_j}\bigr)
      +\dfrac{D_0}{2r_j^{1-\beta}\sigma_0^{2}}
      +\log\tfrac{nJ}{\delta}}
      {2(n-1)}}. \]
      
Because \(1-\beta>0\), the term \(D_0/(r_j^{1-\beta})\) is strictly decreasing in~\(r_j\); all other terms are independent of \(r_j\). Hence
\(\text{gap}(r_j)\) is monotonically decreasing along the grid $\mathcal{G}$:
\[
  r_{j+1}>r_j
  \Longrightarrow
  \text{gap}(r_{j+1})<\text{gap}(r_j).
\]

Additionally, since \eqref{eq:grid_bound} holds \emph{simultaneously} for every \(r_j\), we can safely pick
\[
   r^\star =\arg\min_{r_j\in\mathcal G} \text{gap}(r_j)
\]
after training (or equivalently, choose the model with the smallest validation loss).  The PAC‑Bayes guarantee remains valid and satisfies
\(B(r^\star)\le B(r_0)\).

\end{proof}

% For any $r_j\in\mathcal R$ the KL divergence between $Q_{r_j}$ and $P_{r_j}$
% is
% \[
%   \operatorname{KL}\!\bigl(Q_{r_j}\Vert P_{r_j}\bigr)
%   =\frac12\Bigl[
%         k\!\bigl(c_{r_j}-1-\log c_{r_j}\bigr)
%         +\frac{\|\mu_{r_j}\|_2^{2}}{r_j\sigma_0^{2}}
%     \Bigr].
% \]
% Insert this expression into McAllester’s PAC‑Bayes inequality and apply
% the Langford–Caruana union bound over the finite grid
% $\mathcal R$, which adds the term $\log J$ inside the square root,
% yielding~\eqref{eq:gap_rj}.

% Under the bounded‑drift assumption~\eqref{eq:bounded_drift} we have
% \[
%   \operatorname{KL}\bigl(Q_{r}\Vert P_{r}\bigr)
%   \le
%   \tfrac12k\!\bigl(c-1-\log c\bigr)
%   +\dfrac{D_0}{2\sigma_0^{2}r^{1-\beta}}.
% \]
% Differentiating the right‑hand side of~\eqref{eq:gap_rj} with respect to
% $r$ and using $0\le\beta<1$ gives the negative derivative displayed in
% the theorem statement, hence $\mathrm{gap}(r)$ is strictly decreasing in the valley selection that satisfies the bounded drift condition..
% \end{proof}

\subsection{Maximizing the Worker Consensus Distance in a Valley with Radius C}
\label{appendix:subsec:worker_arrangement}

Consider \( M \) points placed on a circle of radius \( C \) in the 2D plane. Let each point \( P_i \) be parameterized by an angle \( \theta_i \), so that:
\[
P_i = (C \cos \theta_i, C \sin \theta_i), \quad i = 1, 2, \ldots, M
\]

We want to maximize the mean distance of these points to the average variable, \( P_A = (x_A, y_A) \), where the coordinates of the average variable can be expressed as follows:
\[
x_A= \frac{1}{M} \sum_{i=1}^M C \cos \theta_i, \quad y_A = \frac{1}{M} \sum_{i=1}^M C \sin \theta_i
\]

The Euclidean distance \( d_i \) from each point \( P_i \) to the average \( P_a \) can be written as:
\[
d_i = \sqrt{(C \cos \theta_i - x_A)^2 + (C \sin \theta_i - y_A)^2}
\]

Recall that Simplified MV, is expressed as follows:
\[
R = \frac{1}{M} \sum_{i=1}^M d_i.
\]

Directly optimizing the above objective is harder due to the square root in $d_i$. Hence, we simplify the problem and maximize the following objective:
\[
R = \sum_{i=1}^M d_i^2.
\]

We plug in the expression for $d_i$ and re-write $R$:
\[
R = \sum_{i=1}^M \left( (C \cos \theta_i - x_A)^2 + (C \sin \theta_i - y_A)^2 \right)
\]

If we expand each term in the summation, we obtain:
\begin{equation*}
R = \sum_{i=1}^M  [ C^2 \cos^2 \theta_i  + C^2 \sin^2 \theta_i - 2C \cos \theta_i x_A  - 2C \sin \theta_i y_A +x_A^2 +y_A^2 ]
\end{equation*}

Now, using the identity \( \cos^2 \theta_i + \sin^2 \theta_i = 1 \), we can get:
\[
M C^2 - 2C \left(x_A \sum_{i=1}^M \cos \theta_i +y_A \sum_{i=1}^M \sin \theta_i \right) + M (x_A^2 +y_A^2)
\]

Since \(x_A = \frac{C}{M} \sum_{i=1}^M \cos \theta_i \) and \(y_A = \frac{C}{M} \sum_{i=1}^M \sin \theta_i \), we can rewrite the expression above as:
\[
M C^2 - 2C \left(x_A \cdot \frac{Mx_A}{C} +y_A \cdot \frac{My_A}{C} \right) + M (x_A^2 +y_A^2)
\]

Further simplification yields:
\[
S = M C^2 - M (x_A^2 +y_A^2)
\]

This now tells us that maximizing \( R \) is equivalent to minimizing \(x_A^2 +y_A^2 \), which can be written as follows:
\[
x_A^2 +y_A^2 = \left( \frac{C}{M} \sum_{i=1}^M \cos \theta_i \right)^2 + \left( \frac{C}{M} \sum_{i=1}^M \sin \theta_i \right)^2
\]

To minimize \(x_A^2 +y_A^2 \), we require:
\[
\sum_{i=1}^M \cos \theta_i = 0 \quad \text{and} \quad \sum_{i=1}^M \sin \theta_i = 0,
\]

Although this does not have a unique solution, one straightforward solution is the symmetric distribution of the points on the circle so that the above expressions can be made equal to 0. This can be mathematically expressed as follows:

\[
\theta_i = \theta_0 + \frac{2 \pi (i - 1)}{M}, \quad i = 1, 2, \ldots, M
\]

Here \( \theta_0 \) is any fixed angle used as a reference for the symmetric arrangement. Also observe that, in this configuration, the maximum mean distance is $C$.

\subsection{Convergence in Non-Convex Setting}

\textbf{Formulation.} Let us have $M$ workers collaboratively process non-exclusive data shards to find a DNN model that attains the smallest loss on the training data: 
\[
\min \limits_{x \in \mathcal{R}^d} f(x) \stackrel{\Delta}{=} \frac{1}{M} \sum_{i=1}^M f_i(x)
\]

We define each $f_i(x)$ as $f_i(x) \stackrel{\Delta}{=} \mathbb{E}_{\xi_i \sim \mathcal{D}_i} \left [ F_i (x; \xi_i) \right ]$ where $\mathcal{D}_i$ is the portion of the data seen by worker $i$. We assume that each worker can locally observe unbiased, independent stochastic gradients similar to \citep{yu2019parallel}. $g_i^t = \nabla F_i(x_i^{t-1}; \xi_i^t)$ and $\mathbb{E}_{\xi_i^t \sim \mathcal{D}_i} \left [ g_i^t|\xi^{t-1} \right ] = \nabla f_i(x_i^{t-1})$ where $t$ denotes the iteration index. We assume the following technical conditions for the non-convex convergence rate analysis:

\begin{itemize}
    \item We assume each function $f_i(x)$ to be L-smooth: 
    \[ \|\nabla f_i(a) - \nabla f_i(b)\| \leq L \|a-b\| \]
    \item Bounded variance: 
    \[\mathbb{E}_{\xi_i \sim \mathcal{D}_i} \left [ \| \nabla F_i(x; \xi_i) - \nabla f_i(x) \|^2  \right ] \leq \sigma^2 \]
    \item Bounded domain at any iteration $t$: 
    \[ \mathbb{E} \left [ \| x_i^t - x_A^t \|^2  \right ]  \leq \Delta^2 \] 
    where $x_A^t$ is the average of the workers at iteration $t$, i.e. $x_A^t = \frac{1}{M} \sum_{i=1}^M x_i^t$.
\end{itemize}

Recall that the optimization objective we have is expressed as follows:
\begin{equation*}
\sum_{m=1}^{M} f_i (x_i) + \frac{\alpha}{2} \| x_i - x_C \|^2 - \frac{\lambda}{M} \sum_{i=1}^M \| x_{i} - x_A\|_2
\end{equation*}

We consider a general update rule that corresponds to optimizing the objective above in a stochastic way.
\begin{equation*}
x_i^t = x_t^{t-1} - \eta g_i^t - \left ( \alpha  (x_i^{t-1} - x_A^{t-1}) - \lambda \frac{x_i^{t-1} - x_A^{t-1}}{\|x_i^{t-1} - x_A^{t-1}\|} \right )
\end{equation*}

\textbf{Proof.} We start by writing the following expression from $L$-smoothness assumption:
\begin{align}
 \mathbb{E} \left [ f(x_i^t) \right ] \leq \mathbb{E} & \left [ f(x_i^{t-1}) \right ] \label{L_smoothness} \\
 + &\mathbb{E} \left [ \langle \nabla f(x_i^{t-1}, x_i^{t} - x_i^{t-1}) \rangle \right ] \label{term_1}\\
 + &\frac{L}{2} \mathbb{E} \left[ \| x_i^{t} - x_i^{t-1} \|^2 \right] \label{term_2}
\end{align}

We individually tackle the terms on the right-hand side (RHS) of the inequality. We start with \ref{term_2}:
\begin{align}
\mathbb{E} & \left [ \| x_i^{t} - x_i^{t-1} \|^2 \right ]  \\
=\mathbb{E} & \left [ \| - \eta g_i^t - \eta \alpha (x_i^{t-1} - x_A^{t-1}) + \eta \lambda \frac{x_i^{t-1} - x_A^{t-1}}{\|x_i^{t-1} - x_A^{t-1}\|} \|^2 \right ] \notag \\
\leq  3  & \mathbb{E} \left [ \|  - \eta g_i^t \|^2  \right ] \label{term2_1} \\
+& 3\mathbb{E} \left [ \| - \eta  \alpha (x_i^{t-1} - x_A^{t-1}) \|^2  \right ] \label{term2_2}\\ 
+& 3 \mathbb{E} \left [ \| + \eta  \lambda \frac{x_i^{t-1} - x_A^{t-1}}{\|x_i^{t-1} - x_A^{t-1}\|} \|^2  \right ] \label{term2_3}
\end{align}

Where the inequality follows from $\|\sum_{i=1}^n a_i\|^2 \leq n\sum_{i=1}^n \|a_i\|^2$ for $n=3$. We then individually consider the three terms on the RHS. We start by \ref{term2_1}:
\begin{align}
& \mathbb{E} \left [ \| - \eta g_i^t \|^2  \right ] \notag \\
& = \eta^2 \mathbb{E} \left [ \|  g_i^t - \nabla f_i(x_i^{t-1}) + \nabla f_i(x_i^{t-1}) \|^2  \right ] \notag \\
\leq &\eta^2 \left (  \mathbb{E} \left [ \|  g_i^t - \nabla f_i(x_i^{t-1})\|^2 \right ] + \mathbb{E} \|  \nabla f_i(x_i^{t-1}) \|^2 \right) \\
\leq &\eta^2 \left ( \sigma^2 + \mathbb{E} \|  \nabla f_i(x_i^{t-1}) \|^2 \right) 
\end{align}

Here the first inequality is obtained using $\mathbb{E} \left [ \|A\|^2  \right ] = \mathbb{E} \left [ \|A - \mathbb{E} \left [ A \right ]  \|^2  \right ] + \| \mathbb{E} \left [ A \right ]\|^2$ and the second one is reached by using the bounded variance assumption. For \ref{term2_2} and \ref{term2_3}, we have:

\begin{align*}
    \mathbb{E} & \left [ \| - \eta  \alpha (x_i^{t-1} - x_A^{t-1}) \|^2  \right ] \leq \eta^2 \alpha^2 \Delta^2 \hspace{3mm} \text{and} \\
    \mathbb{E} &\left [ \| + \eta  \lambda \frac{x_i^{t-1} - x_A^{t-1}}{\|x_i^{t-1} - x_A^{t-1}\|} \|^2  \right ] \leq \eta^2 \lambda^2
\end{align*}

By combining all these inequalities, we can derive an upper bound for \ref{term_2} as follows:

\begin{align}
\frac{L}{2}\mathbb{E} & \left [ \| x_i^{t} - x_i^{t-1} \|^2 \right ]  \label{term2_upper} \\
\leq  \frac{3L}{2}  & \eta^2 \left ( \sigma^2 + \mathbb{E} \|  \nabla f_i(x_i^{t-1}) \|^2 \right) \notag \\
+& \frac{3L}{2}  \eta^2 \alpha^2 \Delta^2 + \frac{3L}{2} \eta^2 \lambda^2 \notag
\end{align}

We now derive an upper bound for \ref{term_1}.

\begin{align}
     \mathbb{E}  \langle  \nabla f(x_i^{t-1}, &x_i^{t} - x_i^{t-1}) \rangle\\
     = \mathbb{E} &[- \eta\nabla  f(x_i^{t-1})^T (  g_i^t) ] \\
     + \mathbb{E} &[-\eta\alpha \nabla f(x_i^{t-1})^T   (x_i^{t-1} - x_A^{t-1}) \\
      + \mathbb{E} &[\eta\lambda \nabla f(x_i^{t-1})^T\frac{x_i^{t-1} - x_A^{t-1}}{\|x_i^{t-1} - x_A^{t-1}\|} ] 
\end{align}

For the first term on RHS, we can write the following by using $\mathbb{E}_{\xi_i^t \sim \mathcal{D}_i} \left [ g_i^t|\xi^{t-1} \right ] = \nabla f_i(x_i^{t-1})$.

\begin{equation*}
     \mathbb{E} [- \eta\nabla  f(x_i^{t-1})^T (  g_i^t) ]  = -\eta \mathbb{E} \|  \nabla f_i(x_i^{t-1}) \|^2
\end{equation*}

Then using Young's inequality, the upper bounds for the other two terms on the RHS can be written as follows:

\begin{align*}
    \mathbb{E} [-\eta\alpha \nabla &f(x_i^{t-1})^T   (x_i^{t-1} - x_A^{t-1})] \\
    \leq & \frac{\eta \alpha}{2} \mathbb{E} [\| \nabla f(x_i^{t-1}) \|^2] + \frac{\eta \alpha}{2}  \mathbb{E}[ \| x_i^{t-1} - x_A^{t-1} \|^2] \\
    \leq & \frac{\eta \alpha}{2} \mathbb{E} [\| \nabla f(x_i^{t-1}) \|^2] + \frac{\eta \alpha}{2}\Delta^2
\end{align*}

Similarly for the second term, we have:

\begin{align*}
    \mathbb{E} [\eta\lambda \nabla &f(x_i^{t-1})^T   \frac{x_i^{t-1} - x_A^{t-1}}{\|x_i^{t-1} - x_A^{t-1}\|})] \\
    \leq & \frac{\eta \lambda}{2} \mathbb{E} [\| \nabla f(x_i^{t-1}) \|^2] + \frac{\eta \lambda}{2}  \mathbb{E}[ ||\frac{x_i^{t-1} - x_A^{t-1}}{\|x_i^{t-1} - x_A^{t-1}\|}||^2] \\
    \leq & \frac{\eta \lambda}{2} \mathbb{E} [\| \nabla f(x_i^{t-1}) \|^2] + \frac{\eta \lambda}{2}
\end{align*}

Overall, we can write the following for \ref{term_1}:
\begin{align}
     \mathbb{E}  \langle  \nabla f(x_i^{t-1}, &x_i^{t} - x_i^{t-1}) \rangle \label{term1_upper}\\
     \leq & -\eta \mathbb{E} \|  \nabla f_i(x_i^{t-1}) \|^2 \notag \\
     & + \frac{\eta \alpha}{2} \mathbb{E} [\| \nabla f(x_i^{t-1}) \|^2] + \frac{\eta \alpha}{2}\Delta^2 \notag \\
     & + \frac{\eta \lambda}{2} \mathbb{E} [\| \nabla f(x_i^{t-1}) \|^2] + \frac{\eta \lambda}{2} \notag \\
\end{align}

Now using the upper bounds obtained in \ref{term1_upper} and \ref{term2_upper} we can write the following for \ref{L_smoothness}:

\begin{align*}
    \mathbb{E} &[f(x_i^t)] \leq \mathbb{E} [f(x_i^{t-1})] \\
    & - \eta \left ( 1-\frac{\alpha}{2}-\frac{\lambda}{2} - \frac{3L}{2}\eta \right ) \mathbb{E} [ \| \nabla f(x_i^{t-1}) \|^2] \\
    & + \frac{3L}{2}\eta^2 ( \alpha^2 \Delta^2 + \lambda^2 + \sigma^2) + \frac{\eta}{2}(\alpha \Delta^2 + \lambda)
\end{align*}

Furthermore, if we assume that $\left ( 1-\alpha-\lambda - 3L\eta \right ) > 0$, we can also write:

\begin{align*}
    \frac{\eta}{2}\mathbb{E} & [ \| \nabla f(x_i^{t-1}) \|^2] \leq\mathbb{E} [f(x_i^{t-1})] - \mathbb{E} [f(x_i^t)]  \\
    & + \frac{3L}{2} \eta^2 ( \alpha^2 \Delta^2 + \lambda^2 + \sigma^2) + \frac{\eta}{2}(\alpha \Delta^2 + \lambda)
\end{align*}

Finally, by multiplying both sides with $\frac{2}{\eta}$, summing the inequality from $t=0, 1, ..., T$ and taking the average of iterations and using $f(x^*) \leq f(x^T)$ where $x^*$ is the minimum of $f$, we reach the following:

\begin{align*}
    \frac{1}{T}\sum_{t=0}^T \mathbb{E} & [ \| \nabla f(x_i^{t-1}) \|^2] \leq \frac{2(f(x_0) - f(x^*))}{\eta T}  \\
    & + 3L \eta ( \alpha^2 \Delta^2 + \lambda^2 + \sigma^2) + \alpha \Delta^2 + \lambda
\end{align*}

This characterizes the convergence rate of a single worker in a non-convex setting by bounding the expected value of gradient norms averaged over iterations.

%% file: supplement_landscapes.tex
\section{More Loss and Error Landscape Visualizations}
\label{appendix:subsec:landscape_vis}
Here we present more visualizations by including train loss, test loss, train error and test error landscapes obtained after training ResNet-18 models on CIFAR-10/100 datasets with SimpleAvg and DPPF\textsubscript{SimpleAvg} when the underlying local optimizer is SGD. The visualization settings—including whether the train or test set is used, whether loss or error is shown, the dataset, axis limits, and step sizes—are specified in the captions. Observe that in all the plots below, the optima recovered by DPPF\textsubscript{SimpleAvg} are substantially wider than those of vanilla SimpleAvg, and the corresponding test error and loss values are consistently lower.

We also present our visualization technique here, as we do not directly inherit a visualization method from the literature. To project the worker positions on a 2D plane with loss/error contours, we perform Singular Value Decomposition (SVD) among the distance vectors calculated between workers and the average variable. Then, we take the two most representative components (unit vectors) that capture the most variance. Using these vectors, we scan a 2D grid whose endpoints and resolution are specified, starting from the average variable $x_A$, hence the $x_A$ is always at the origin of the grid.  We record the test and training loss and error statistics for each grid point, plot the resulting contours, and finally project the worker positions onto this plane. This procedure is described in Algorithm ~\ref{appendix:alg:visualization}.

\begin{algorithm}[H]
    \caption{Landscape Visualization}
    \label{appendix:alg:visualization}
    \Input{trained model parameters from $M$ workers $x_1, x_2, ..., x_M$; grid limit $L$ for defining the grid edges; step size $s$ for the resolution of the grid}

    $x_A \leftarrow \mathbf{0}$
    
    \For{$m = 1$ to $M$;}{
    $x_A \leftarrow x_A + \frac{x_m}{M}$
    }
    
    $\Delta \leftarrow [\hspace{2mm}]$
    
    \For{$m = 1$ to $M$;}{
    $\Delta \leftarrow \Delta  + [(x_m - x_A)]$
    }
    
    $\delta_x, \delta_y \leftarrow$ get useful vectors from SVD($\Delta$)
    
    $\Phi \leftarrow [\hspace{2mm}]$
    
    \For{$i = -L:s:L$;}{
    \For{$j = -L:s:L$;}{
    
    $x \leftarrow x_A + i \delta_x + j\delta_y$ 
    
    $\Phi \leftarrow \Phi + [f(x, D_{train}), f(x, D_{test})]$
    }
    }
    plot($\Phi$)

\end{algorithm}

% \subsection{CIFAR-10 (4 Workers) - 2D Visualizations} 
\label{appendix:subsubsec:c10_landscape}
\vspace{-5mm}
\begin{figure}[H]
    \centering
    \begin{subfigure}{0.34\columnwidth} % Adjust width as needed
        \centering
        \includegraphics[width=\columnwidth]{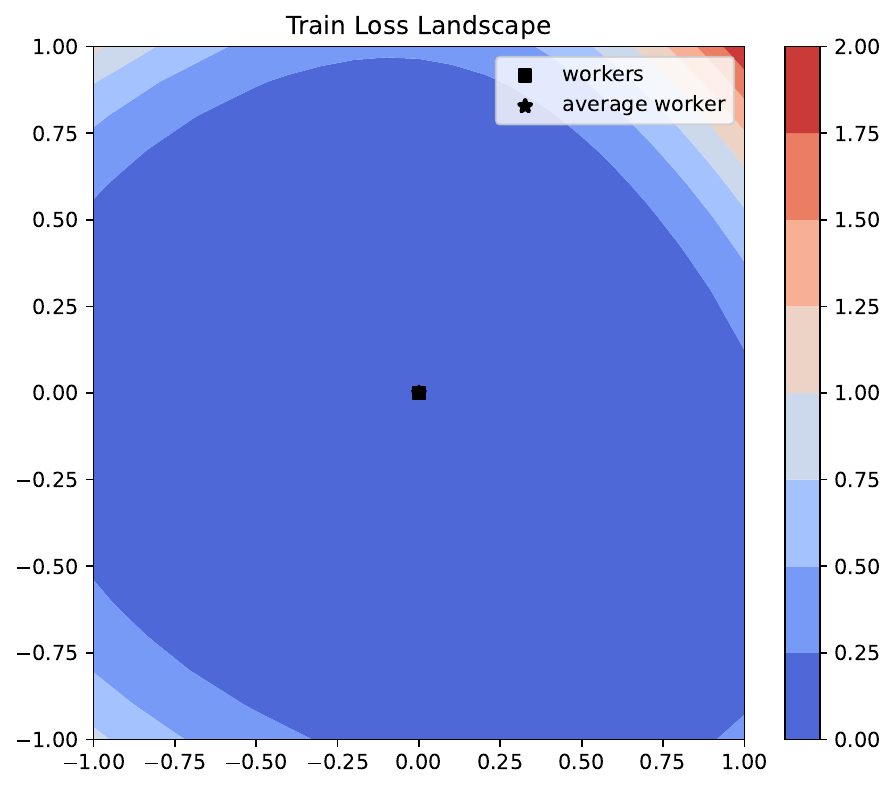}
        \caption{SimpleAvg}
    \end{subfigure}
    \hspace{10mm}
    \begin{subfigure}{0.34\columnwidth} % Adjust width as needed
        \centering
        \includegraphics[width=\columnwidth]{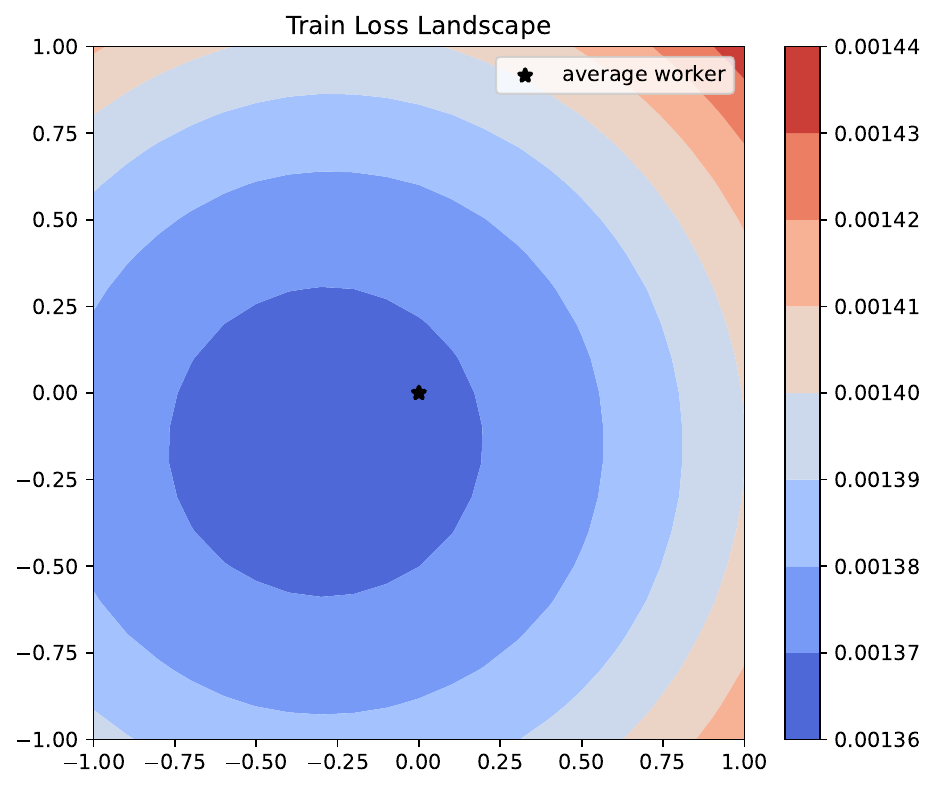}
        \caption{DPPF\textsubscript{SimpleAvg}}
    \end{subfigure}
    %\vspace{-1mm}
    \caption{Training loss landscapes, $\text{lim}=1$, $\text{step}=0.1$, 4 workers, CIFAR-10.}
\end{figure}
\vspace{-5mm}
\begin{figure}[H]
    \centering
    \begin{subfigure}{0.34\columnwidth} % Adjust width as needed
        \centering
        \includegraphics[width=\columnwidth]{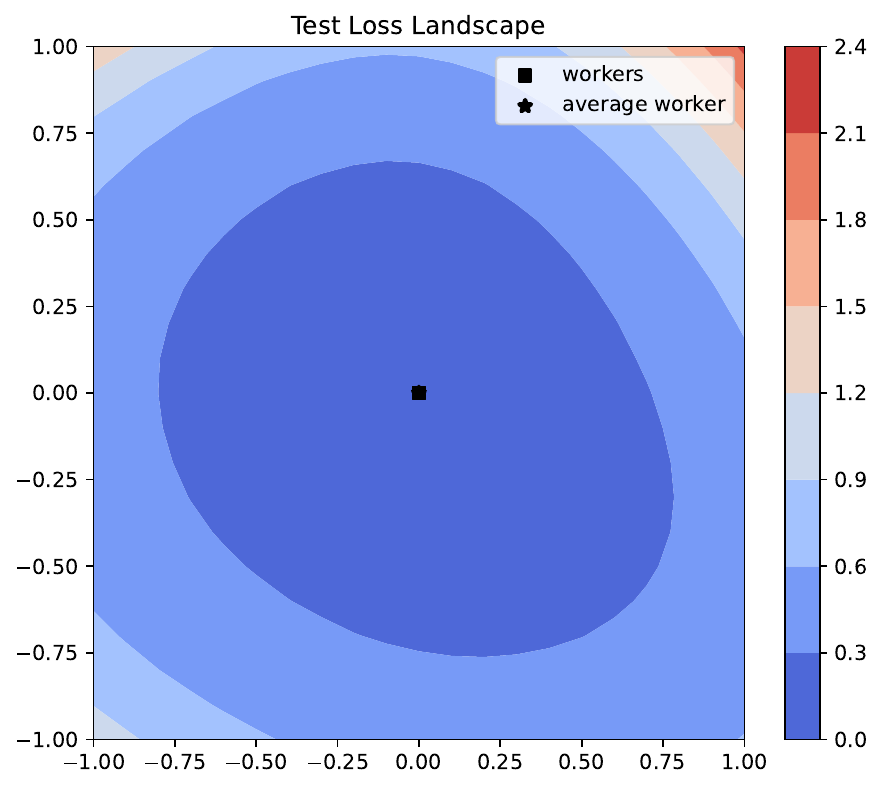}
        \caption{SimpleAvg}
        
    \end{subfigure}
    \hspace{10mm}
    \begin{subfigure}{0.34\columnwidth} % Adjust width as needed
        \centering
        \includegraphics[width=\columnwidth]{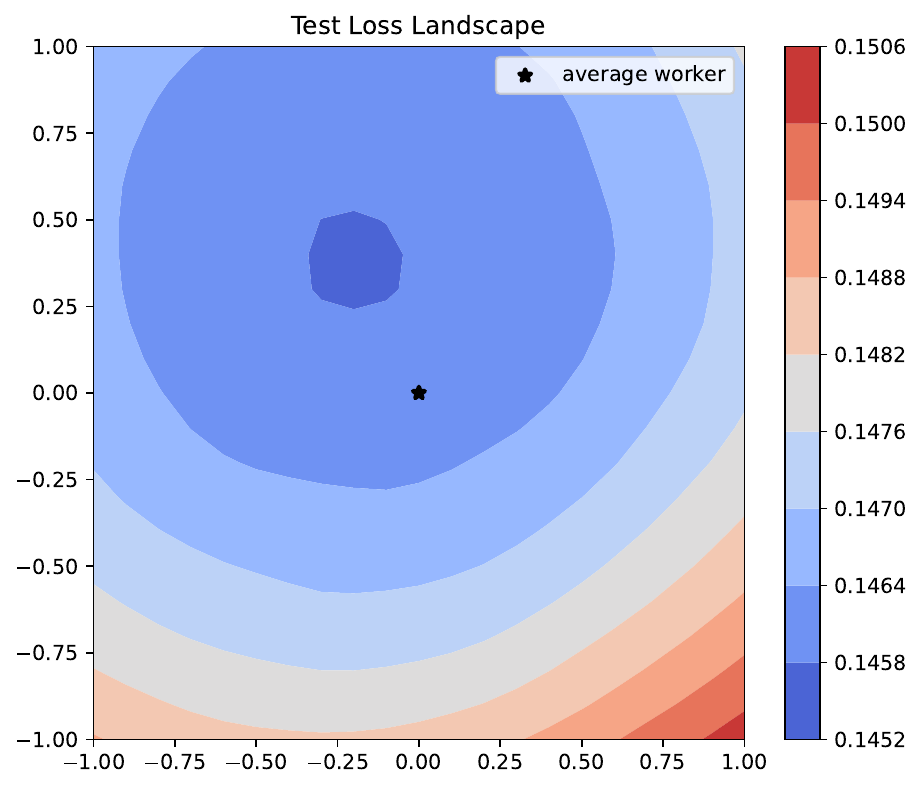}
        \caption{DPPF\textsubscript{SimpleAvg}}
    \end{subfigure}
    %\vspace{-1mm}
    \caption{Test loss landscapes, $\text{lim}=1$, $\text{step}=0.1$, 4 workers, CIFAR-10.}
    
\end{figure}

\begin{figure}[H]
    \centering
    \begin{subfigure}{0.34\columnwidth} % Adjust width as needed
        \centering
        \includegraphics[width=\columnwidth]{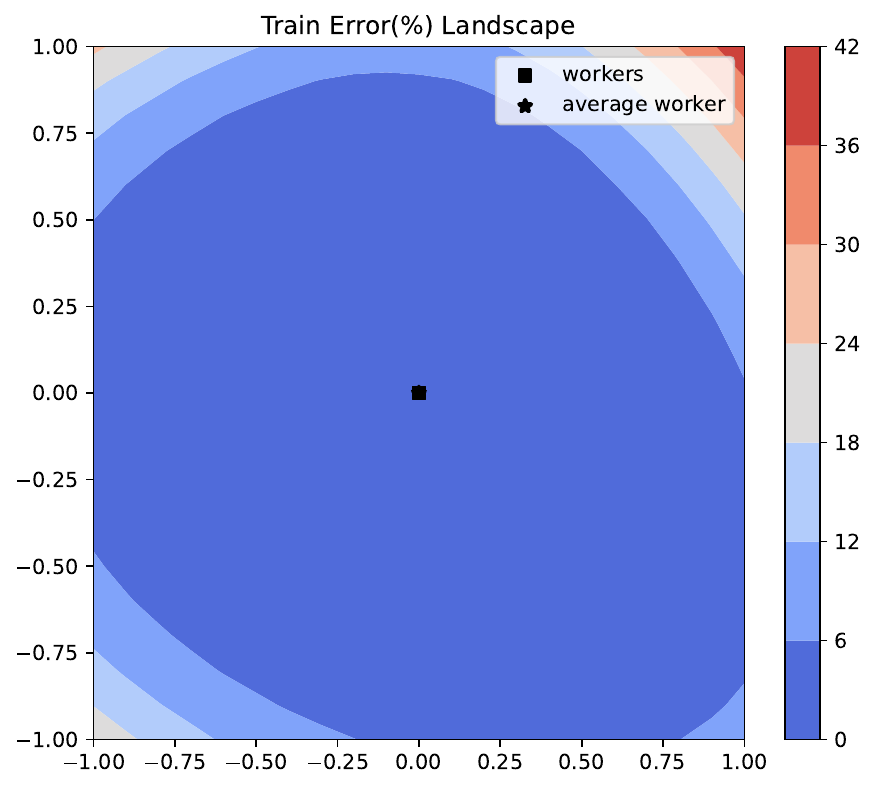}
        \caption{SimpleAvg}
        
    \end{subfigure}
    \hspace{10mm}
    \begin{subfigure}{0.34\columnwidth} % Adjust width as needed
        \centering
        \includegraphics[width=\columnwidth]{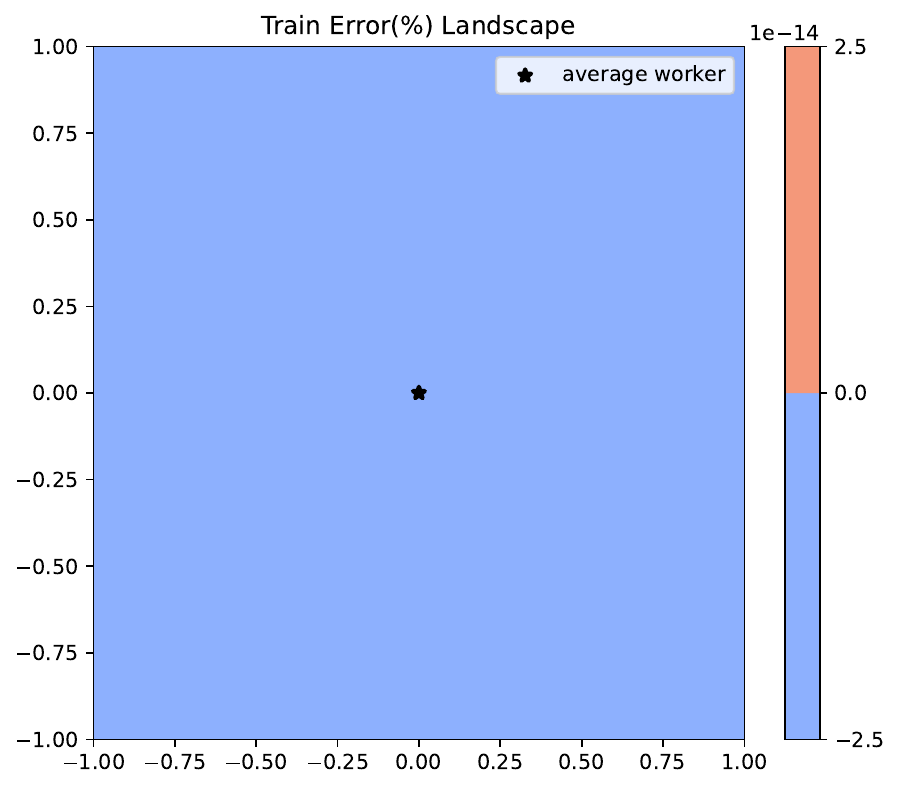}
        \caption{DPPF\textsubscript{SimpleAvg}}
    \end{subfigure}
    %\vspace{-1mm}
    \caption{Training error (\%) landscapes, $\text{lim}=1$, $\text{step}=0.1$, 4 workers, CIFAR-10.}
    
\end{figure}

\begin{figure}[H]
    \centering
    \begin{subfigure}{0.34\columnwidth} % Adjust width as needed
        \centering
        \includegraphics[width=\columnwidth]{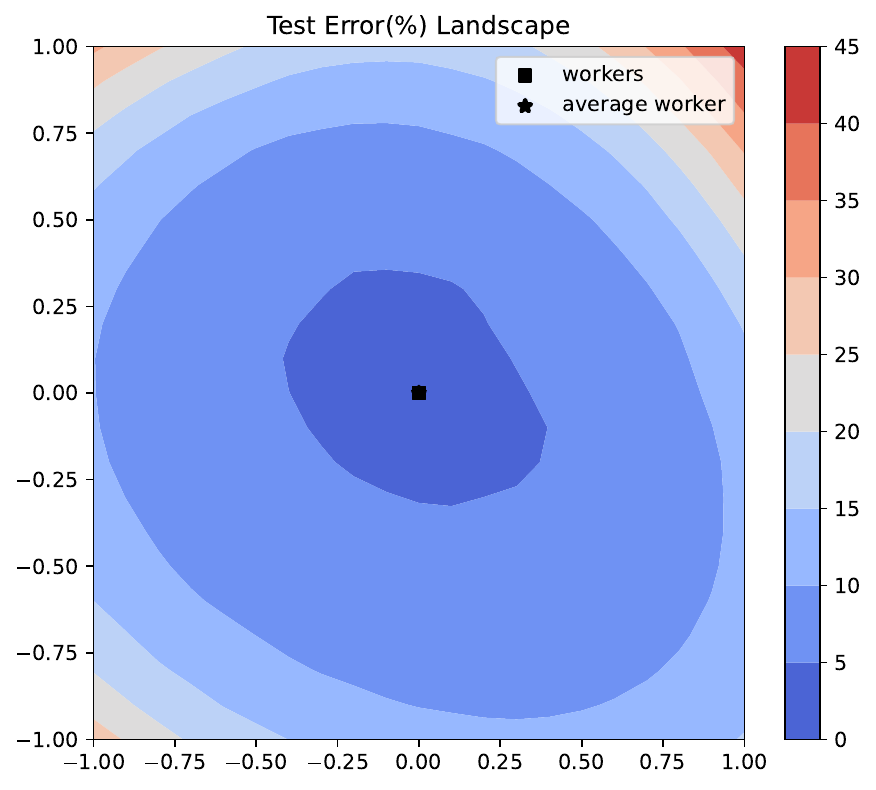}
        \caption{SimpleAvg}
        
    \end{subfigure}
    \hspace{10mm}
    \begin{subfigure}{0.34\columnwidth} % Adjust width as needed
        \centering
        \includegraphics[width=\columnwidth]{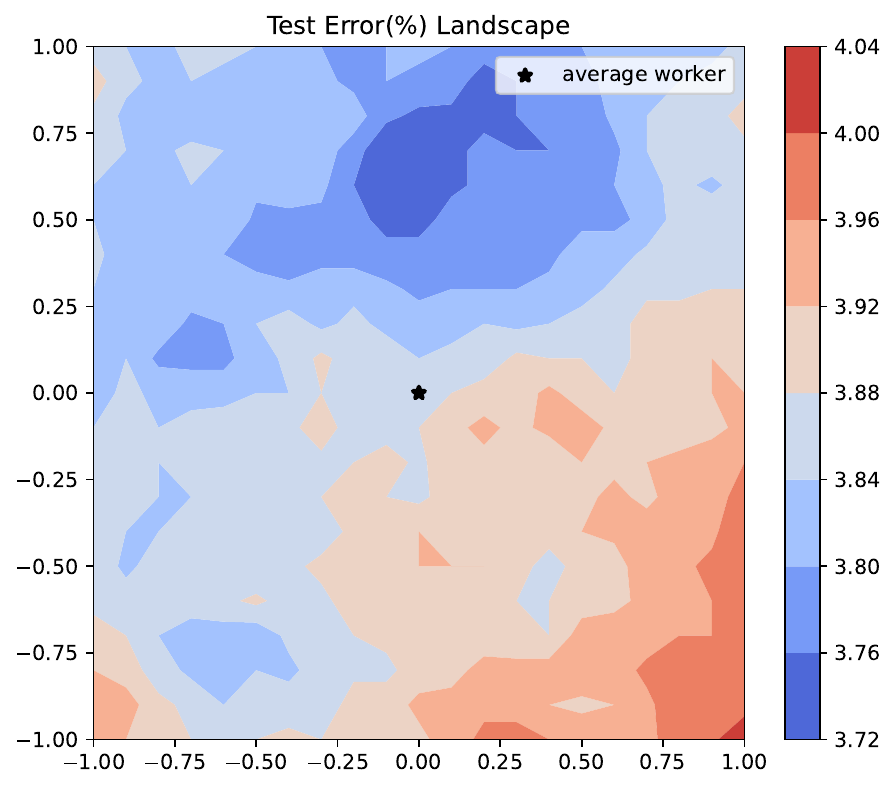}
        \caption{DPPF\textsubscript{SimpleAvg}}
    \end{subfigure}
    %\vspace{-1mm}
    \caption{Test error (\%) landscapes, $\text{lim}=1$, $\text{step}=0.1$, 4 workers, CIFAR-10.}
    
\end{figure}

%%%%%%%%%%%%%%%%%%%%%%%%%%%%%%%%%%%%%%%%%%%%%%%%%%%%%%%%%%%%%%%%%%%%%%%%%%%%%%%%%%%%%%%%%%%%%%%%%%%%%%%%%%%%%
\begin{figure}[H]
    \centering
    \begin{subfigure}{0.34\columnwidth} % Adjust width as needed
        \centering
        \includegraphics[width=\columnwidth]{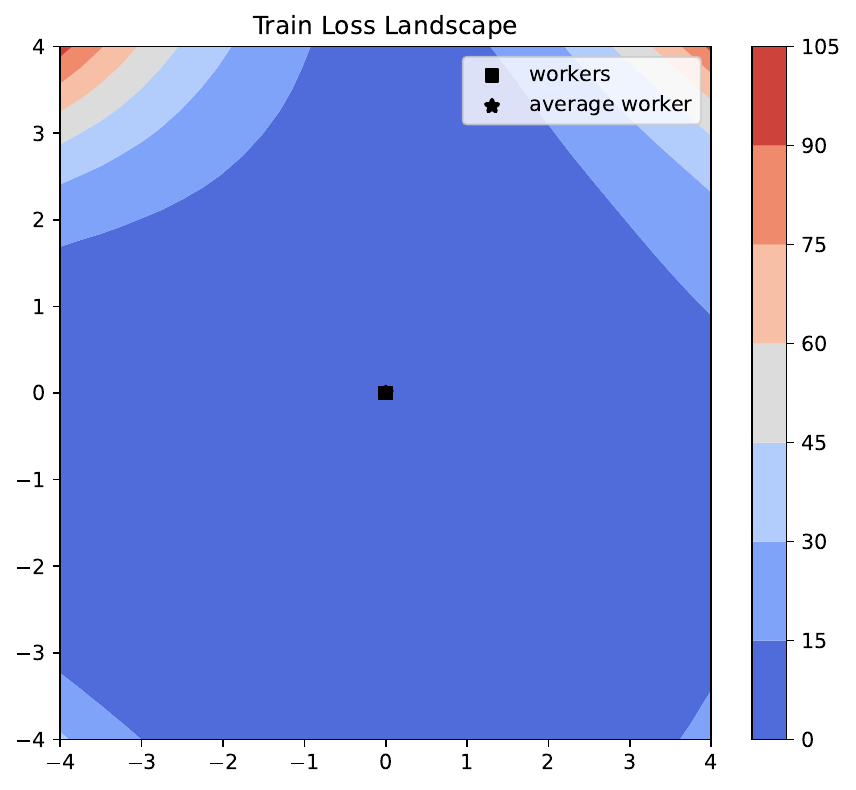}
        \caption{SimpleAvg}
        
    \end{subfigure}
    \hspace{10mm}
    \begin{subfigure}{0.34\columnwidth} % Adjust width as needed
        \centering
        \includegraphics[width=\columnwidth]{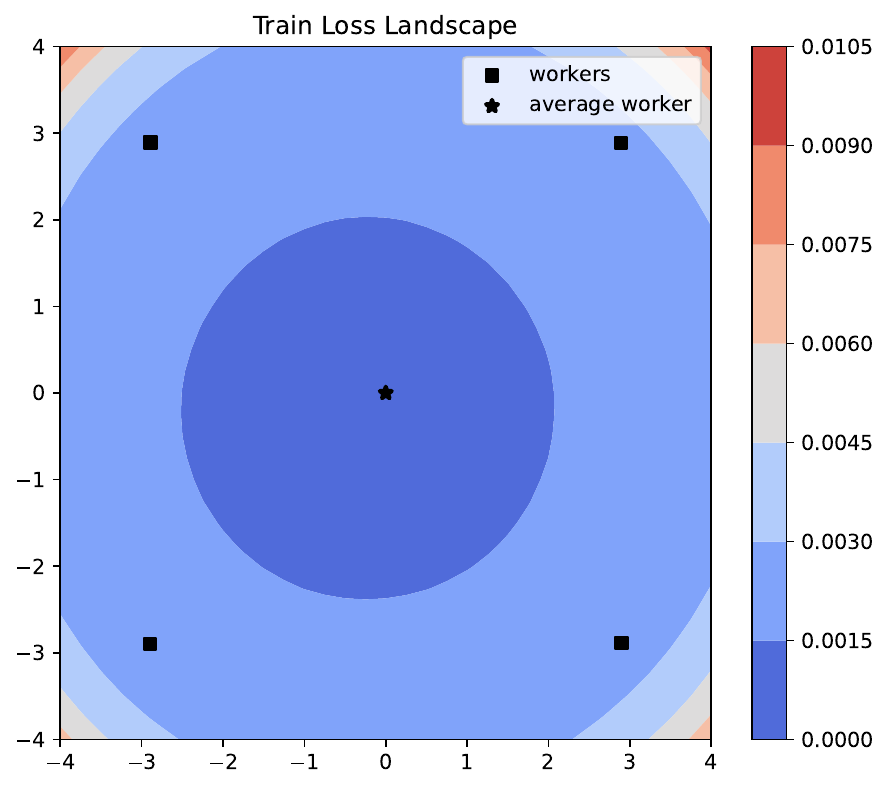}
        \caption{DPPF\textsubscript{SimpleAvg}}
    \end{subfigure}
    %\vspace{-1mm}
    \caption{Training loss landscapes, $\text{lim}=4$, $\text{step}=0.25$, 4 workers, CIFAR-10.}
    
\end{figure}

\begin{figure}[H]
    \centering
    \begin{subfigure}{0.34\columnwidth} % Adjust width as needed
        \centering
        \includegraphics[width=\columnwidth]{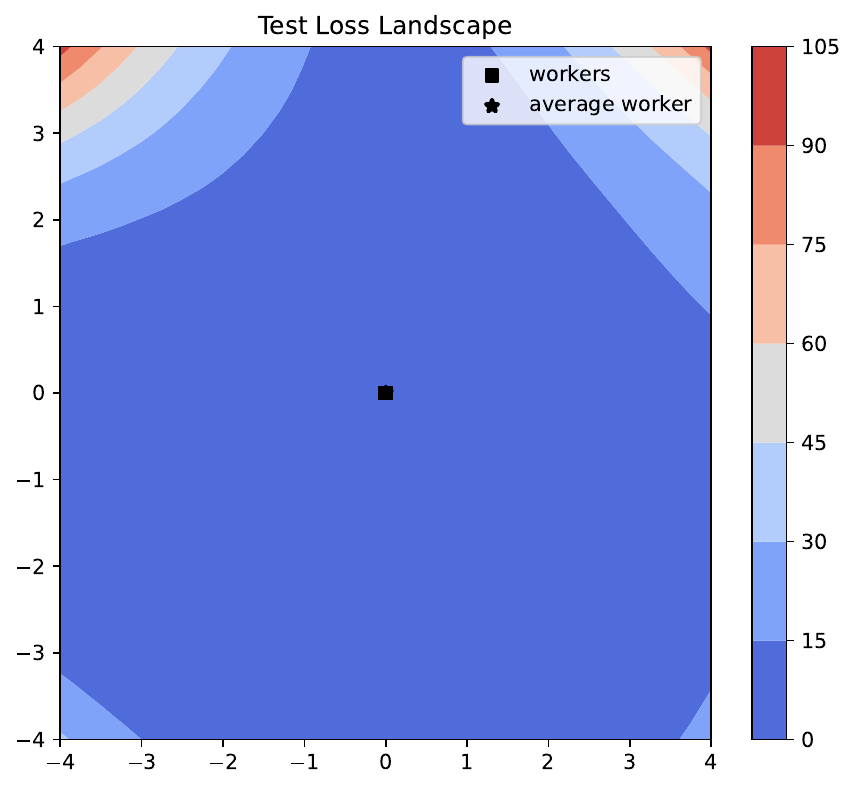}
        \caption{SimpleAvg}
        
    \end{subfigure}
    \hspace{10mm}
    \begin{subfigure}{0.34\columnwidth} % Adjust width as needed
        \centering
        \includegraphics[width=\columnwidth]{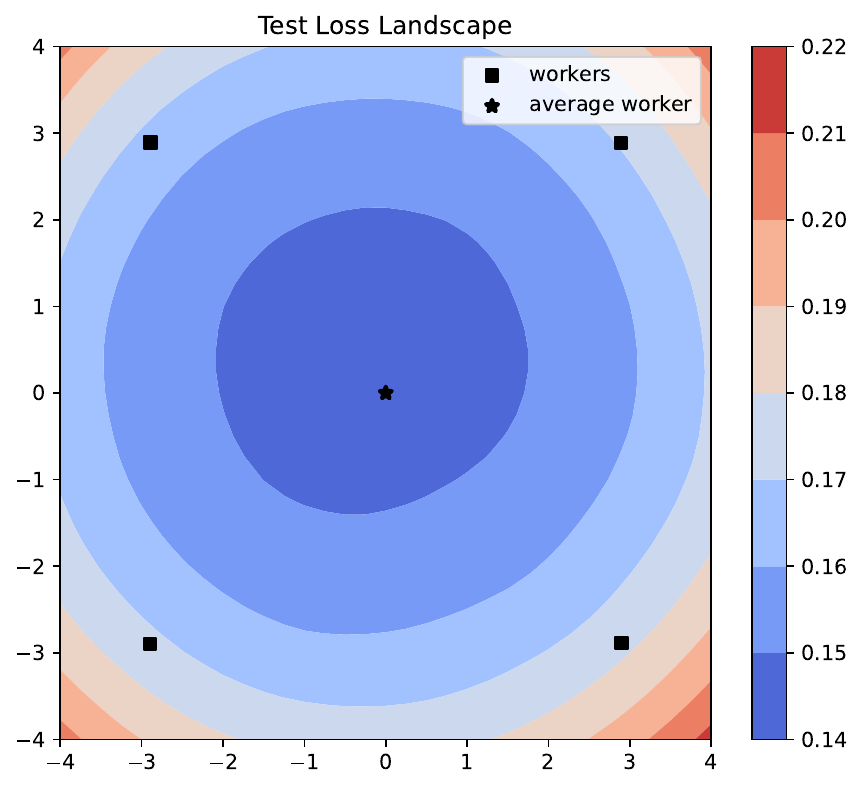}
        \caption{DPPF\textsubscript{SimpleAvg}}
    \end{subfigure}
    %\vspace{-1mm}
    \caption{Test loss landscapes, $\text{lim}=4$, $\text{step}=0.25$, 4 workers, CIFAR-10.}
    
\end{figure}

\begin{figure}[H]
    \centering
    \begin{subfigure}{0.34\columnwidth} % Adjust width as needed
        \centering
        \includegraphics[width=\columnwidth]{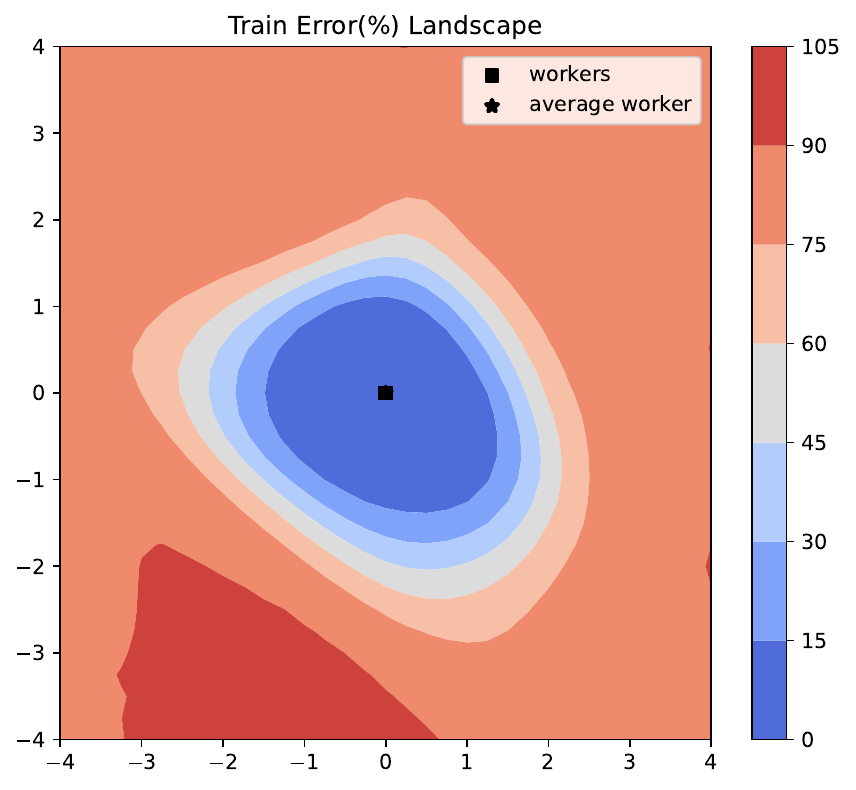}
        \caption{SimpleAvg}
        
    \end{subfigure}
    \hspace{10mm}
    \begin{subfigure}{0.34\columnwidth} % Adjust width as needed
        \centering
        \includegraphics[width=\columnwidth]{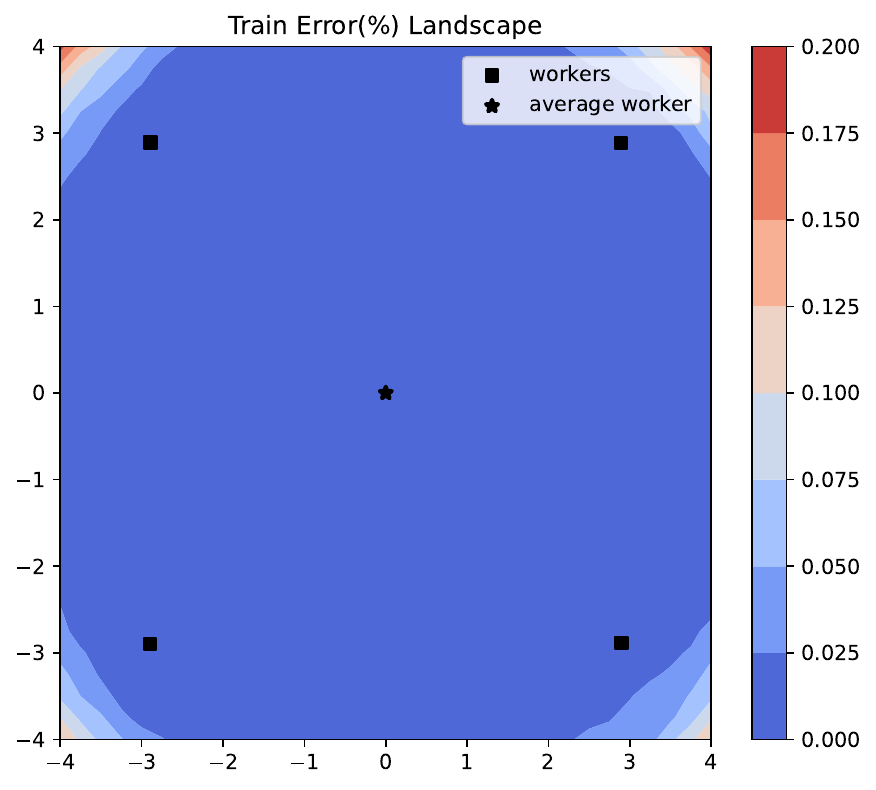}
        \caption{DPPF\textsubscript{SimpleAvg}}
    \end{subfigure}
    %\vspace{-1mm}
    \caption{Training error (\%) landscapes, $\text{lim}=4$, $\text{step}=0.25$, 4 workers, CIFAR-10.}
    
\end{figure}

\begin{figure}[H]
    \centering
    \begin{subfigure}{0.34\columnwidth} % Adjust width as needed
        \centering
        \includegraphics[width=\columnwidth]{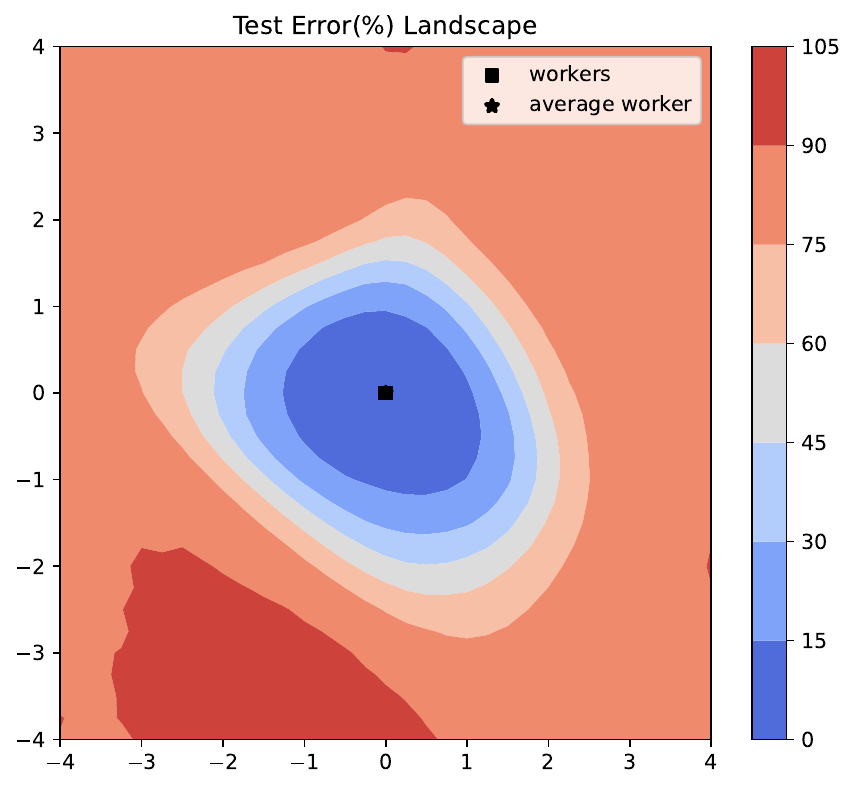}
        \caption{SimpleAvg}
        
    \end{subfigure}
    \hspace{10mm}
    \begin{subfigure}{0.34\columnwidth} % Adjust width as needed
        \centering
        \includegraphics[width=\columnwidth]{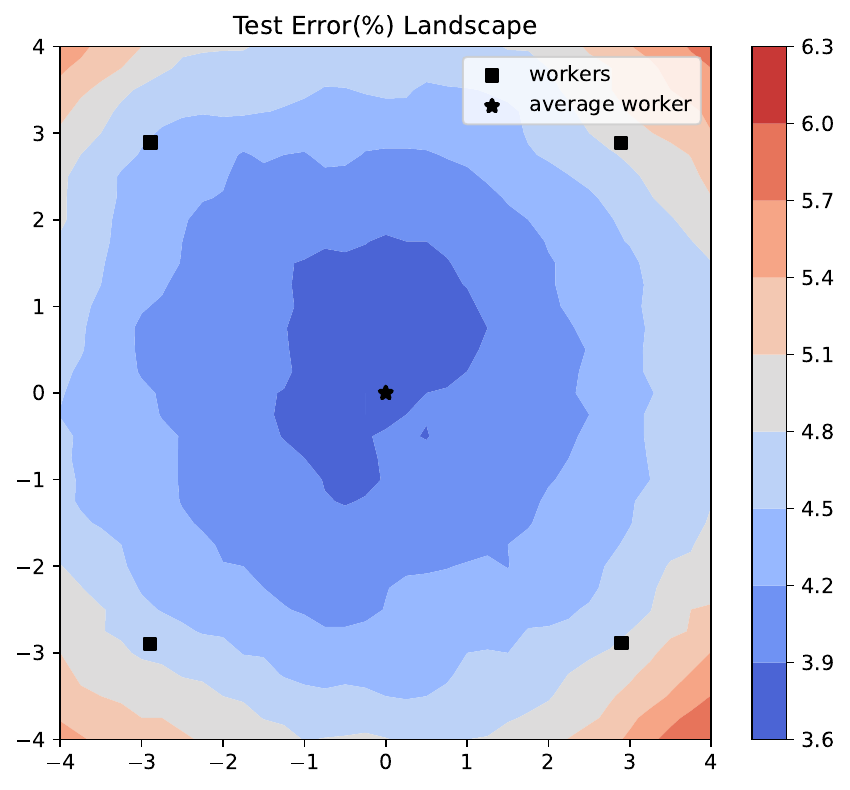}
        \caption{DPPF\textsubscript{SimpleAvg}}
    \end{subfigure}
    %\vspace{-1mm}
    \caption{Test error (\%) landscapes, $\text{lim}=4$, $\text{step}=0.25$, 4 workers, CIFAR-10.}
    
\end{figure}

% \subsection{CIFAR-100 (4 Workers) - 2D Visualizations}
\label{appendix:subsubsec:c100_landscape}

\begin{figure}[H]
    \centering
    \begin{subfigure}{0.34\columnwidth} % Adjust width as needed
        \centering
        \includegraphics[width=\columnwidth]{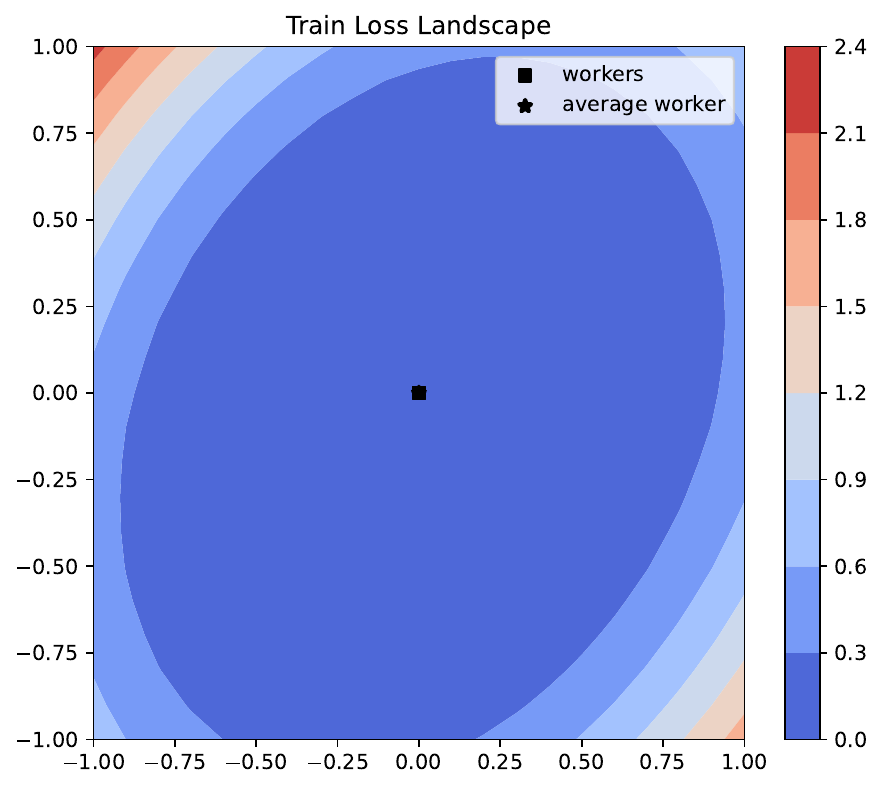}
        \caption{SimpleAvg}
        
    \end{subfigure}
    \hspace{10mm}
    \begin{subfigure}{0.34\columnwidth} % Adjust width as needed
        \centering
        \includegraphics[width=\columnwidth]{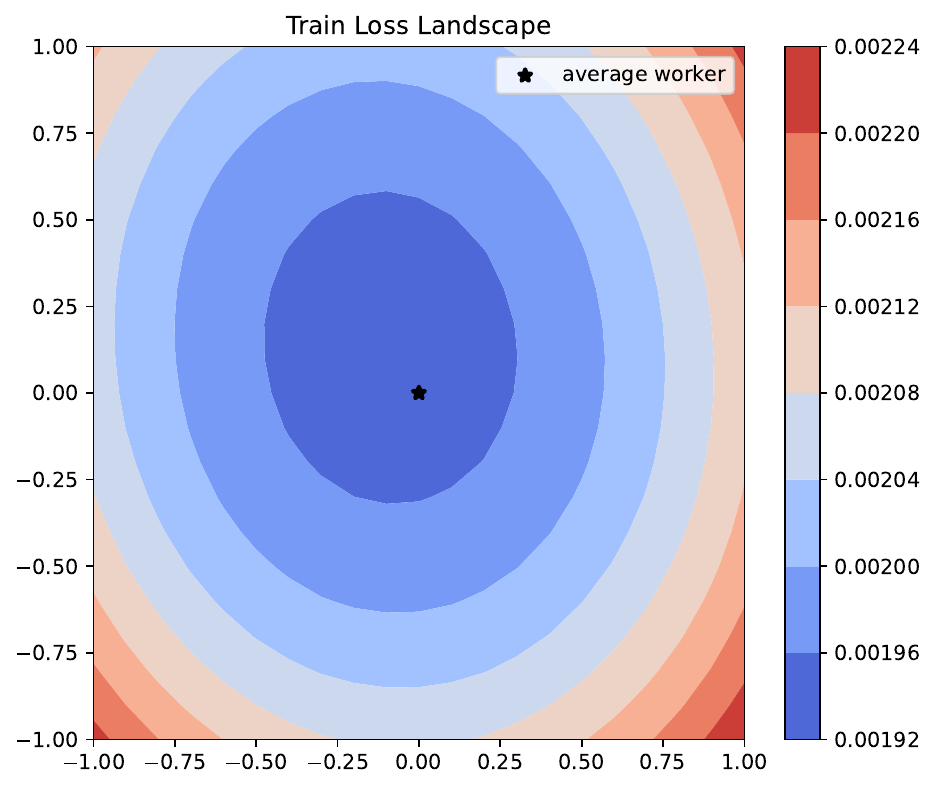}
        \caption{DPPF\textsubscript{SimpleAvg}}
    \end{subfigure}
    %\vspace{-1mm}
    \caption{Training loss landscapes, $\text{lim}=1$, $\text{step}=0.1$, 4 workers, CIFAR-100.}
    
\end{figure}

\begin{figure}[H]
    \centering
    \begin{subfigure}{0.34\columnwidth} % Adjust width as needed
        \centering
        \includegraphics[width=\columnwidth]{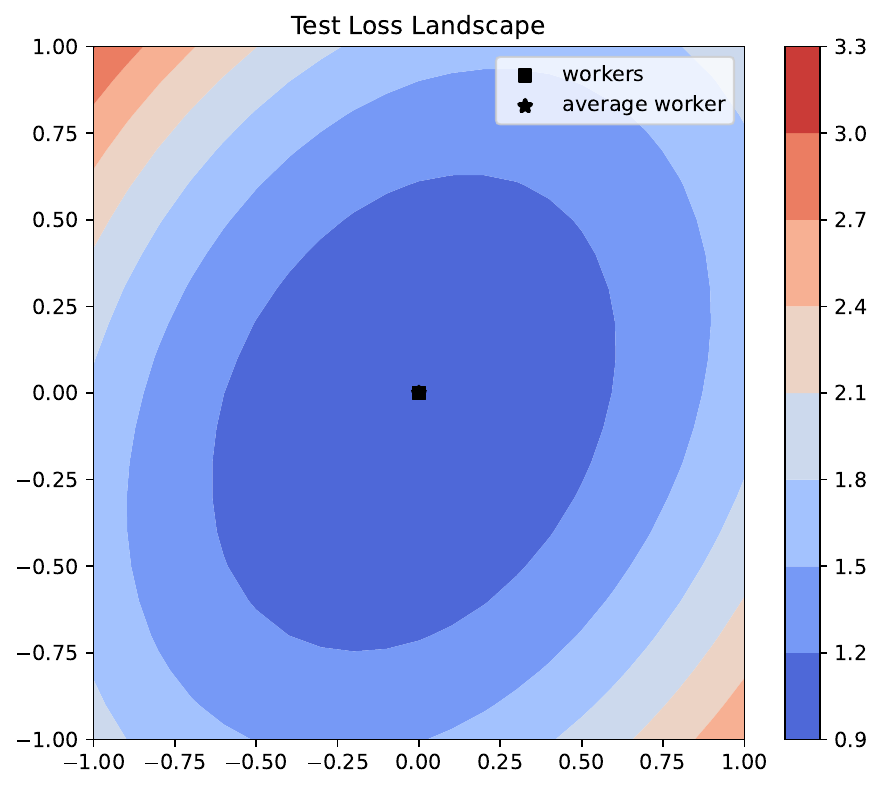}
        \caption{SimpleAvg}
        
    \end{subfigure}
    \hspace{10mm}
    \begin{subfigure}{0.34\columnwidth} % Adjust width as needed
        \centering
        \includegraphics[width=\columnwidth]{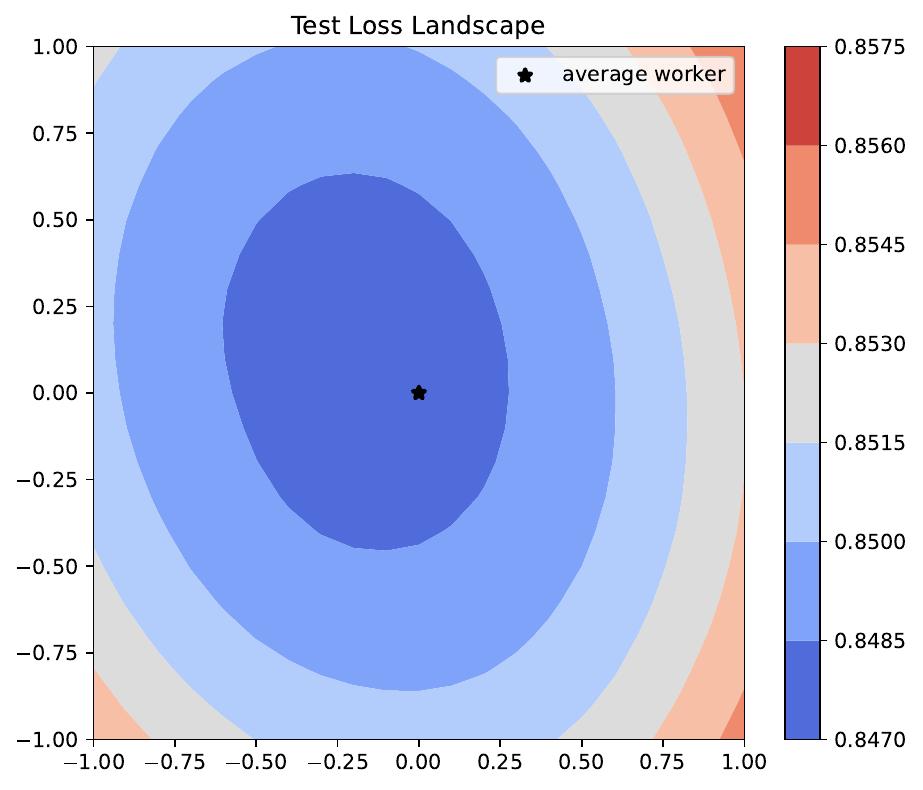}
        \caption{DPPF\textsubscript{SimpleAvg}}
    \end{subfigure}
    %\vspace{-1mm}
    \caption{Test loss landscapes, $\text{lim}=1$, $\text{step}=0.1$, 4 workers, CIFAR-100.}
    
\end{figure}

\begin{figure}[H]
    \centering
    \begin{subfigure}{0.34\columnwidth} % Adjust width as needed
        \centering
        \includegraphics[width=\columnwidth]{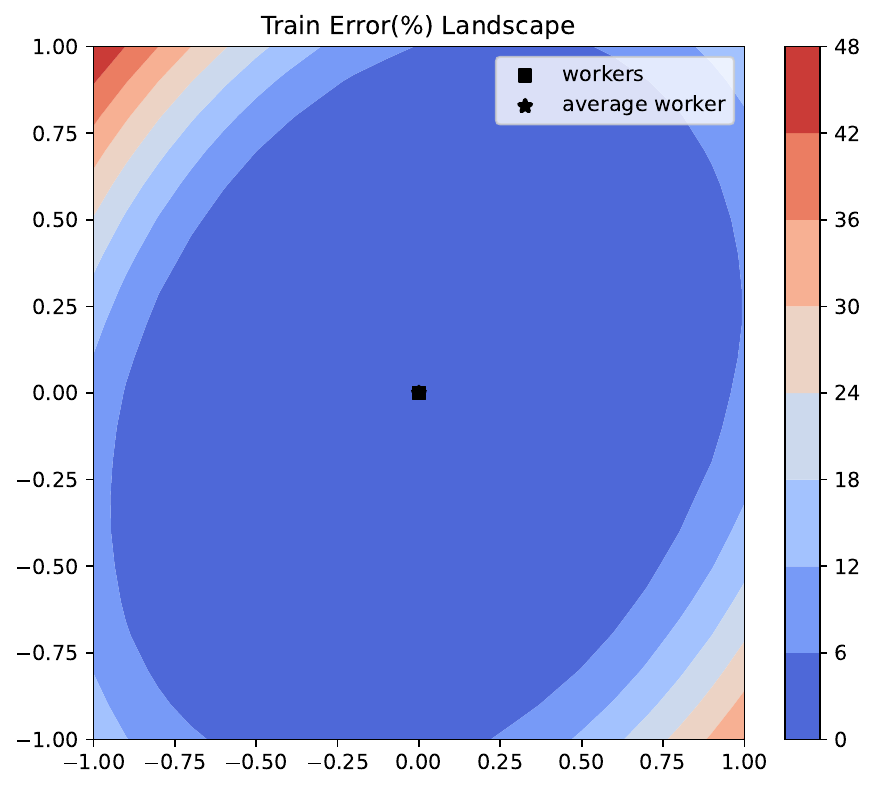}
        \caption{SimpleAvg}
        
    \end{subfigure}
    \hspace{10mm}
    \begin{subfigure}{0.34\columnwidth} % Adjust width as needed
        \centering
        \includegraphics[width=\columnwidth]{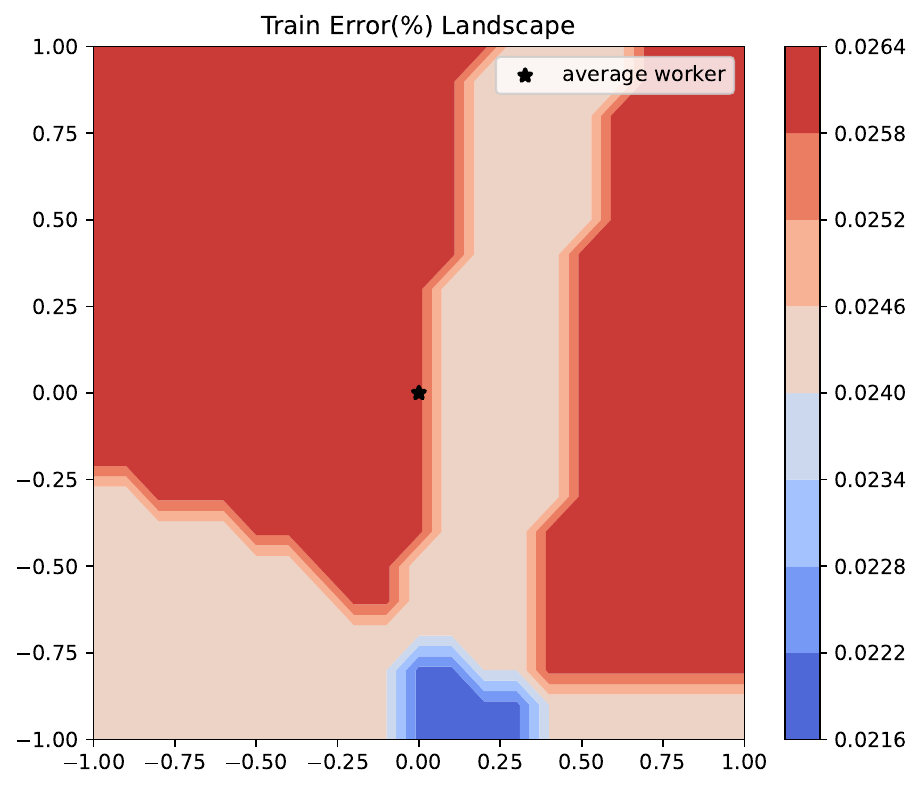}
        \caption{DPPF\textsubscript{SimpleAvg}}
    \end{subfigure}
    %\vspace{-1mm}
    \caption{Training error (\%) landscapes, $\text{lim}=1$, $\text{step}=0.1$, 4 workers, CIFAR-100.}
    
\end{figure}

\begin{figure}[H]
    \centering
    \begin{subfigure}{0.34\columnwidth} % Adjust width as needed
        \centering
        \includegraphics[width=\columnwidth]{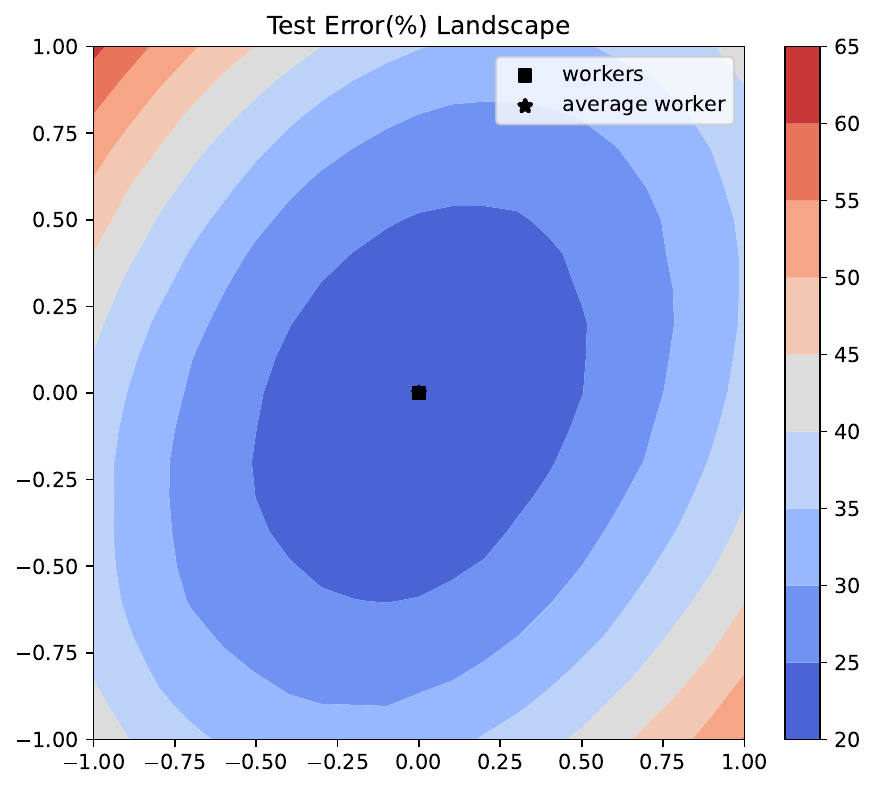}
        \caption{SimpleAvg}
        
    \end{subfigure}
    \hspace{10mm}
    \begin{subfigure}{0.34\columnwidth} % Adjust width as needed
        \centering
        \includegraphics[width=\columnwidth]{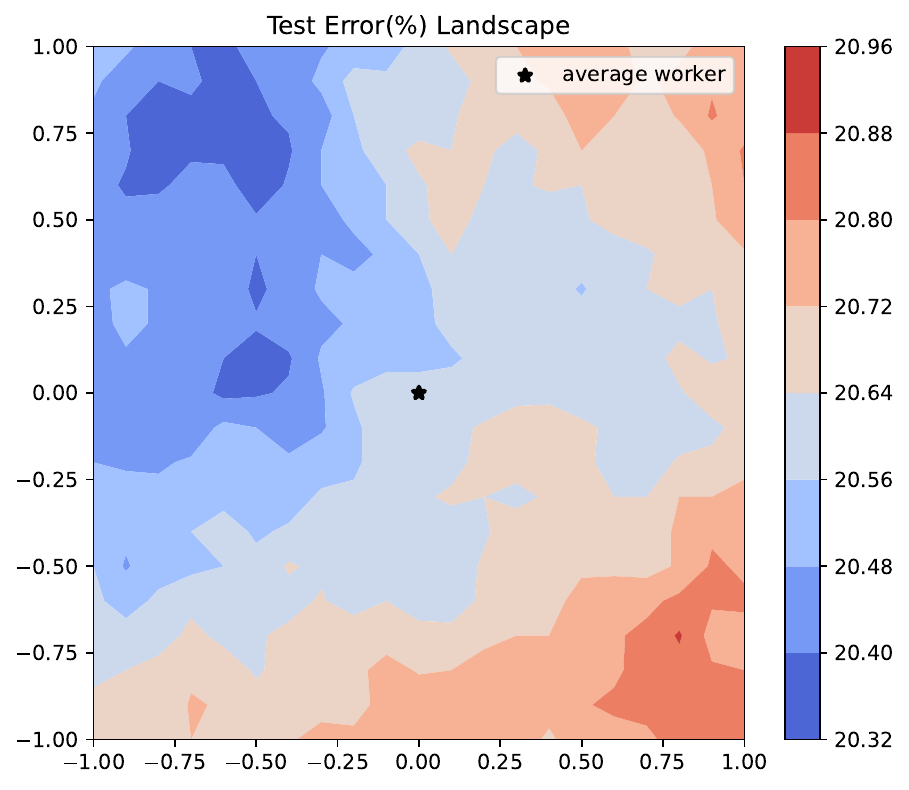}
        \caption{DPPF\textsubscript{SimpleAvg}}
    \end{subfigure}
    %\vspace{-1mm}
    \caption{Test error (\%) landscapes, $\text{lim}=1$, $\text{step}=0.1$, 4 workers, CIFAR-100.}
    
\end{figure}

%%%%%%%%%%%%%%%%%%%%%%%%%%%%%%%%%%%%%%%%%%%%%%%%%%%%%%%%%%%%%%%%%%%%%%%%%%%%%%%%%%%%%%%%%%%%%%%%%%%%%%%%%%%%%
\begin{figure}[H]
    \centering
    \begin{subfigure}{0.34\columnwidth} % Adjust width as needed
        \centering
        \includegraphics[width=\columnwidth]{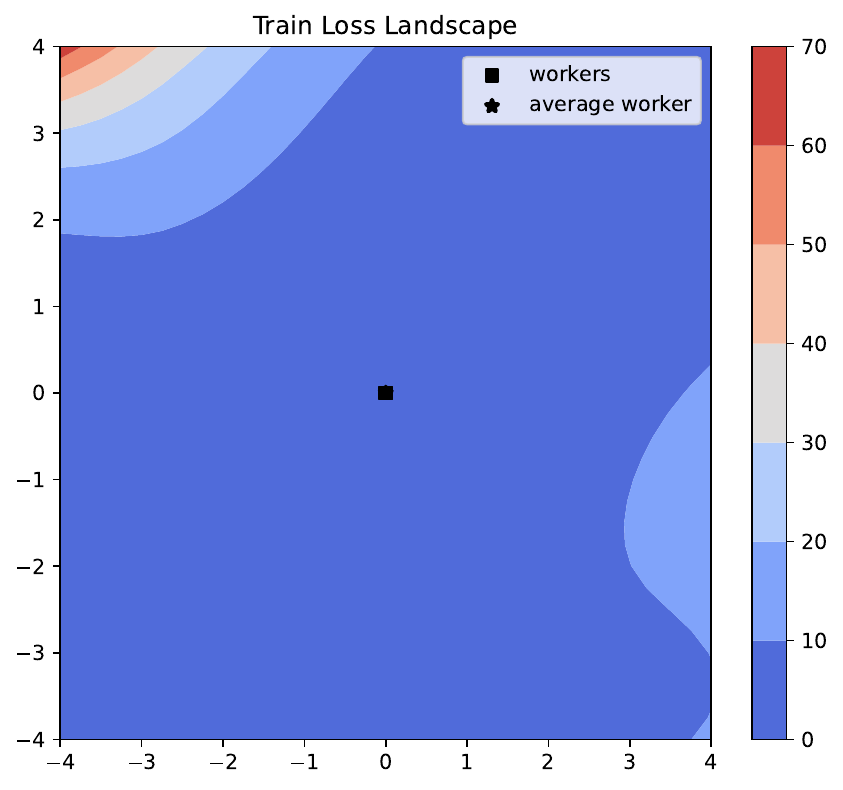}
        \caption{SimpleAvg}
        
    \end{subfigure}
    \hspace{10mm}
    \begin{subfigure}{0.34\columnwidth} % Adjust width as needed
        \centering
        \includegraphics[width=\columnwidth]{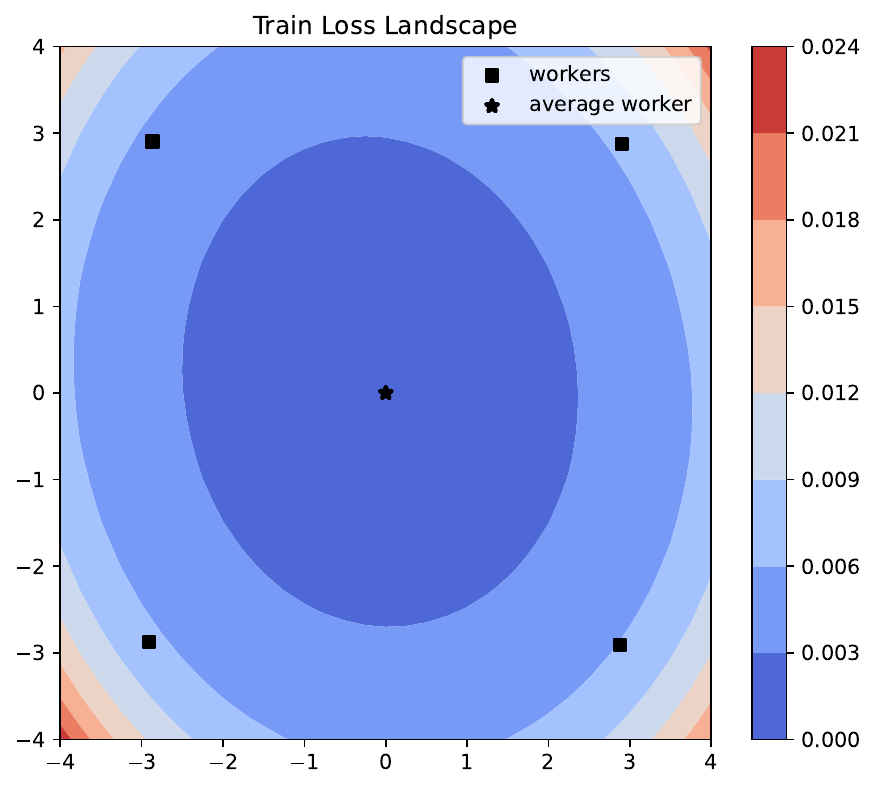}
        \caption{DPPF\textsubscript{SimpleAvg}}
    \end{subfigure}
    %\vspace{-1mm}
    \caption{Training loss landscapes, $\text{lim}=4$, $\text{step}=0.25$, 4 workers, CIFAR-100.}
    
\end{figure}

\begin{figure}[H]
    \centering
    \begin{subfigure}{0.34\columnwidth} % Adjust width as needed
        \centering
        \includegraphics[width=\columnwidth]{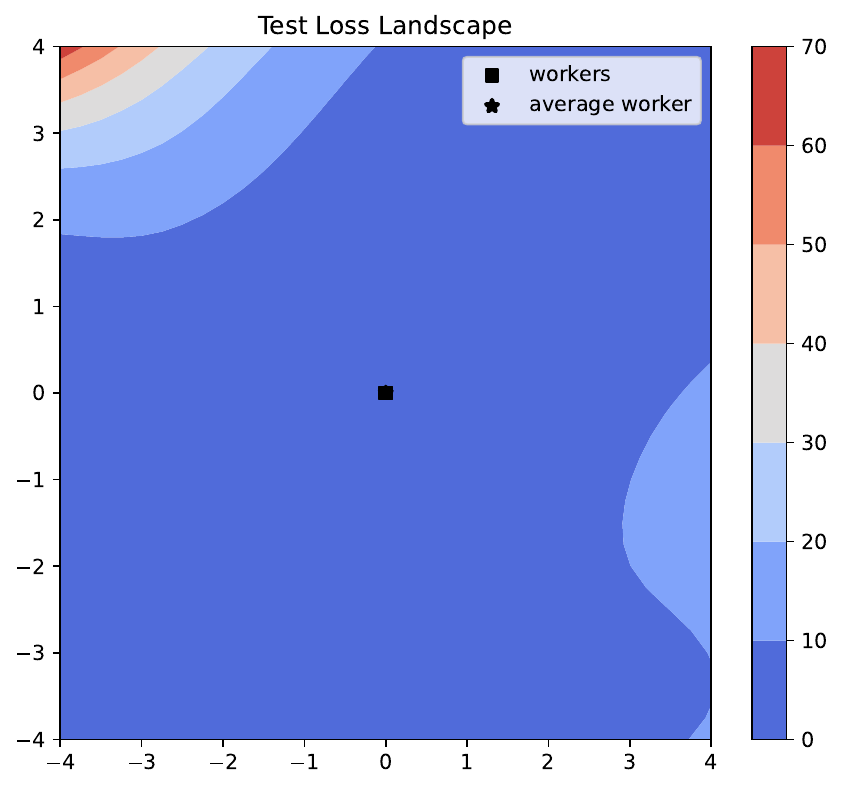}
        \caption{SimpleAvg}
        
    \end{subfigure}
    \hspace{10mm}
    \begin{subfigure}{0.34\columnwidth} % Adjust width as needed
        \centering
        \includegraphics[width=\columnwidth]{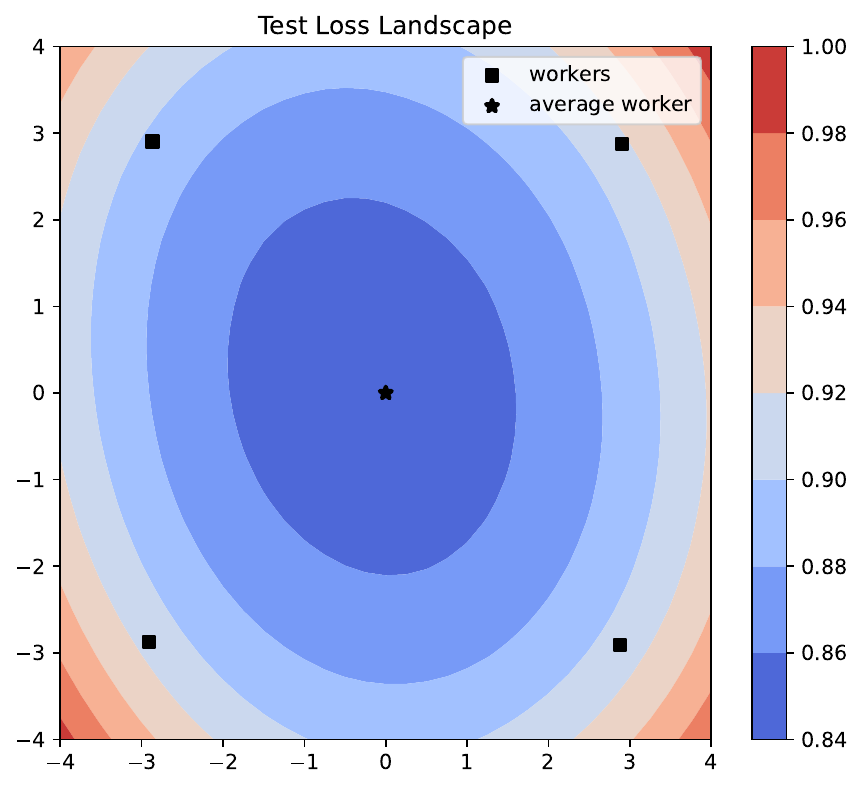}
        \caption{DPPF\textsubscript{SimpleAvg}}
    \end{subfigure}
    %\vspace{-1mm}
    \caption{Test loss landscapes, $\text{lim}=4$, $\text{step}=0.25$, 4 workers, CIFAR-100.}
    
\end{figure}

\begin{figure}[H]
    \centering
    \begin{subfigure}{0.34\columnwidth} % Adjust width as needed
        \centering
        \includegraphics[width=\columnwidth]{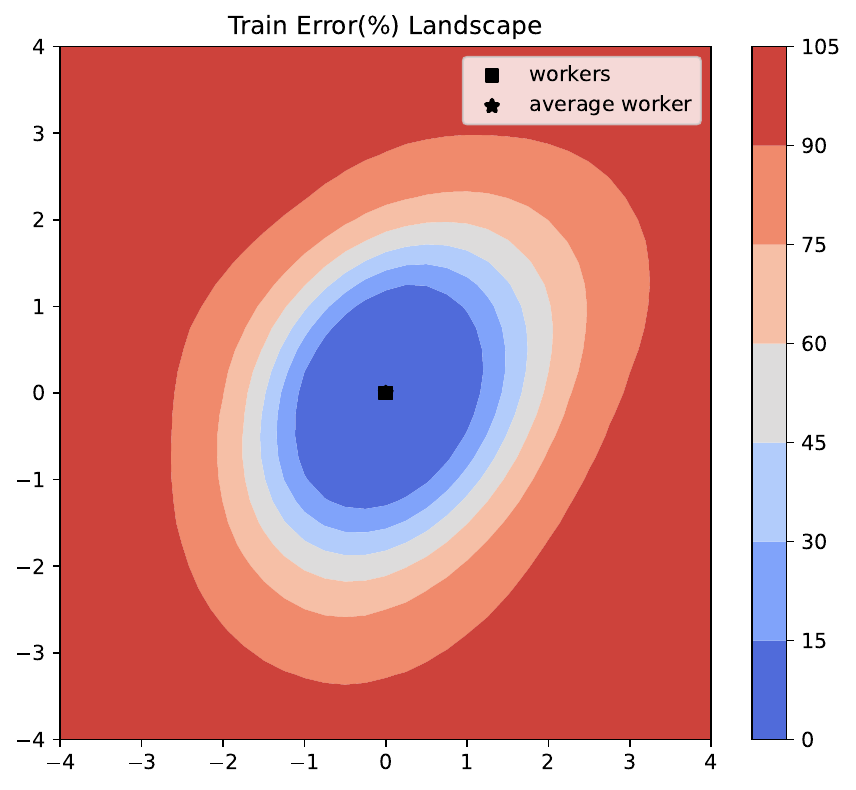}
        \caption{SimpleAvg}
        
    \end{subfigure}
    \hspace{10mm}
    \begin{subfigure}{0.34\columnwidth} % Adjust width as needed
        \centering
        \includegraphics[width=\columnwidth]{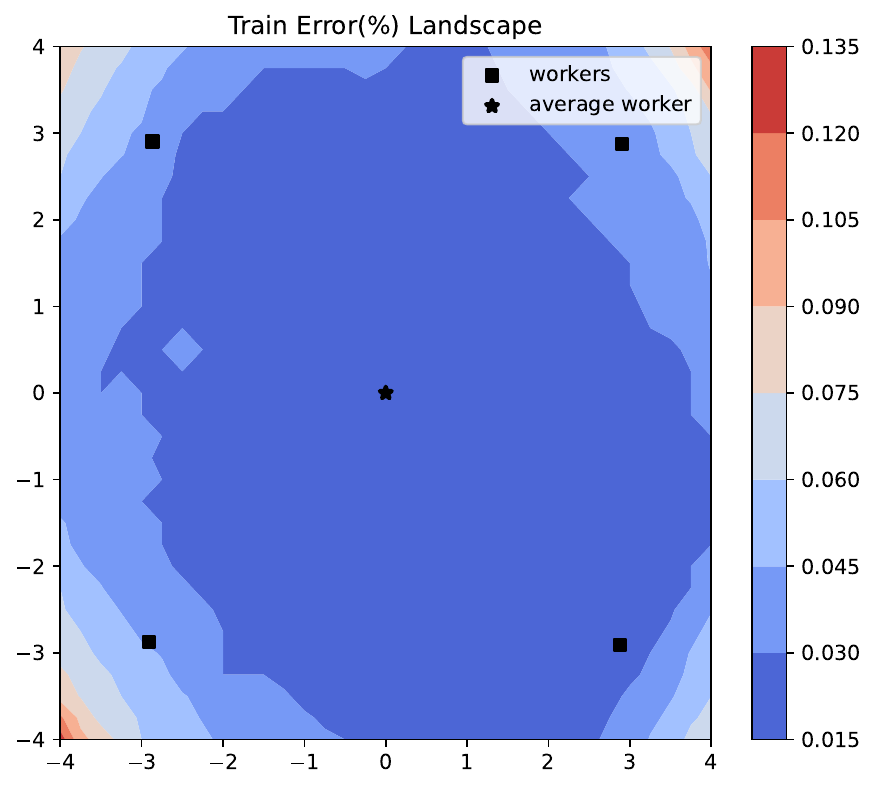}
        \caption{DPPF\textsubscript{SimpleAvg}}
    \end{subfigure}
    %\vspace{-1mm}
    \caption{Training error (\%) landscapes, $\text{lim}=4$, $\text{step}=0.25$, 4 workers, CIFAR-100.}
    
\end{figure}

\begin{figure}[H]
    \centering
    \begin{subfigure}{0.34\columnwidth} % Adjust width as needed
        \centering
        \includegraphics[width=\columnwidth]{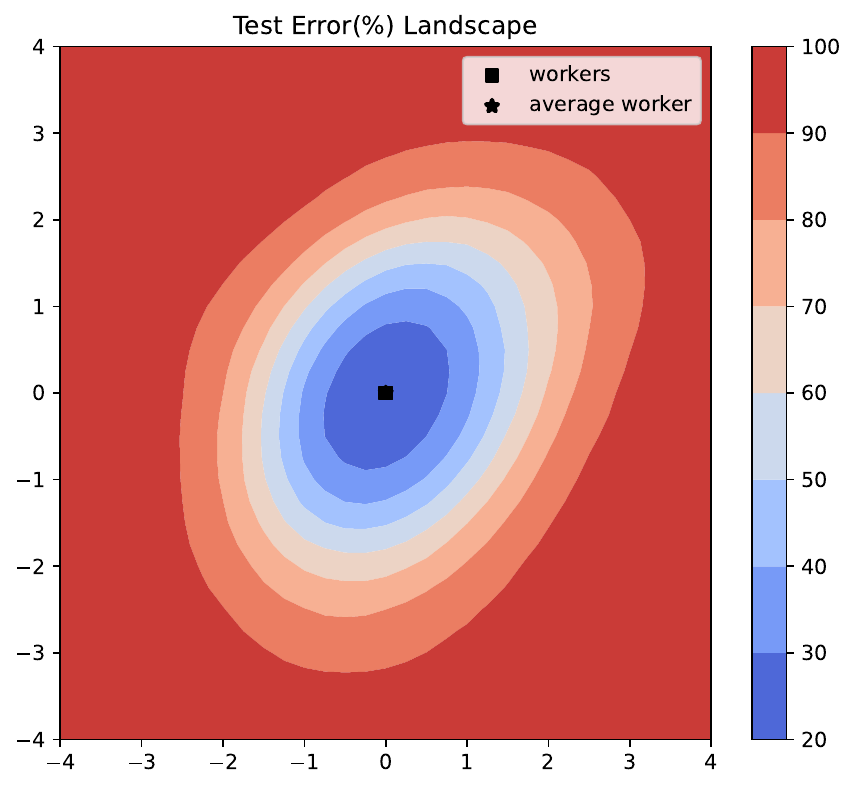}
        \caption{SimpleAvg}
        
    \end{subfigure}
    \hspace{10mm}
    \begin{subfigure}{0.34\columnwidth} % Adjust width as needed
        \centering
        \includegraphics[width=\columnwidth]{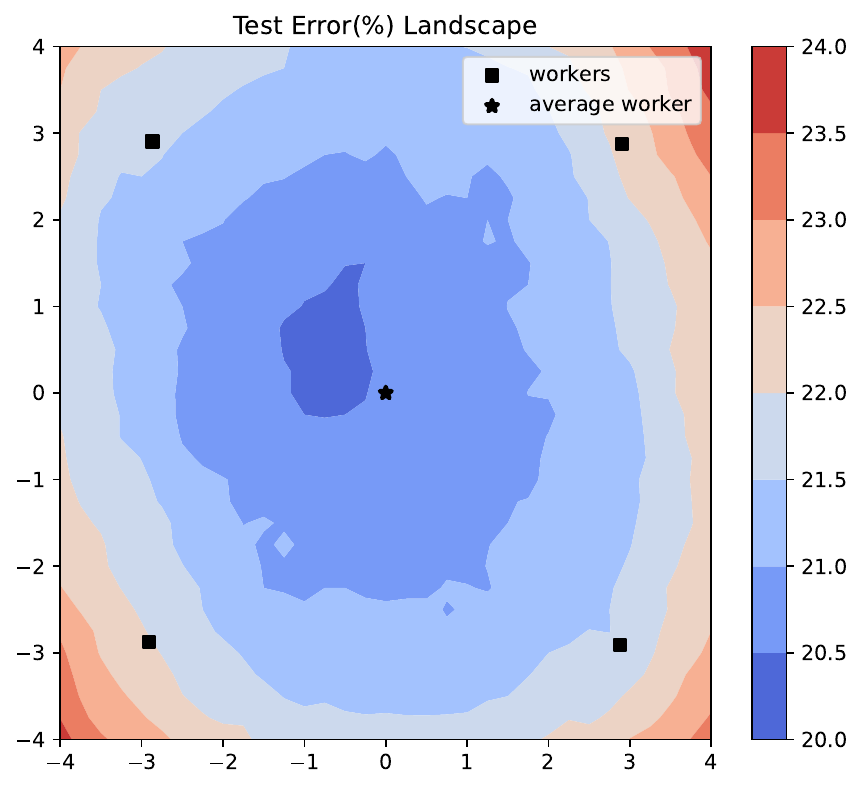}
        \caption{DPPF\textsubscript{SimpleAvg}}
    \end{subfigure}
    %\vspace{-1mm}
    \caption{Test error (\%) landscapes, $\text{lim}=4$, $\text{step}=0.25$, 4 workers, CIFAR-100.}
    
\end{figure}

% \subsection{CIFAR-10 (4 Workers) - 3D Visualizations} 
\label{appendix:subsubsec:c10_landscape_3d}

\begin{figure}[H]
    \centering
    \begin{subfigure}{0.34\columnwidth} % Adjust width as needed
        \centering
        \includegraphics[width=\columnwidth]{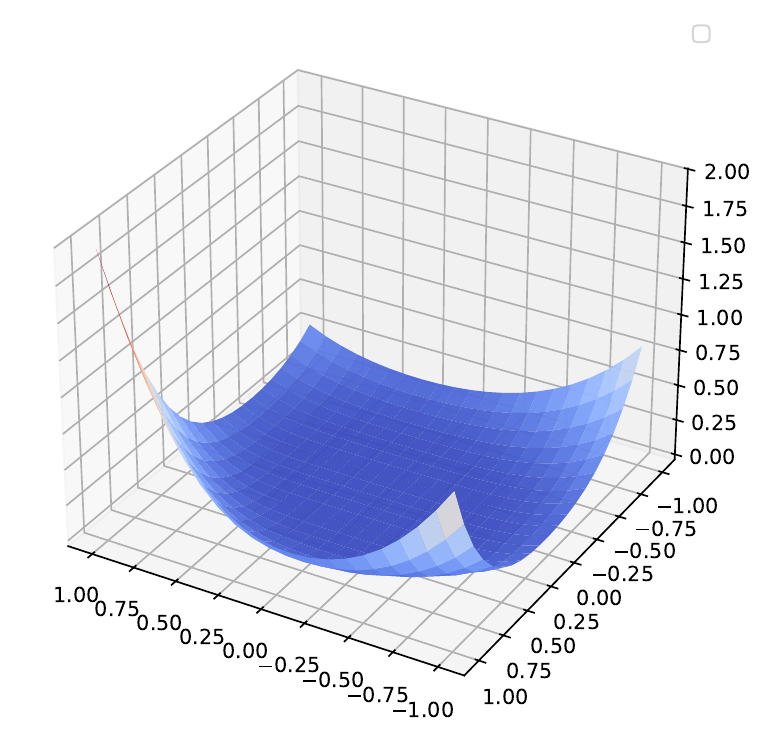}
        \caption{SimpleAvg}
        
    \end{subfigure}
    \hspace{10mm}
    \begin{subfigure}{0.34\columnwidth} % Adjust width as needed
        \centering
        \includegraphics[width=\columnwidth]{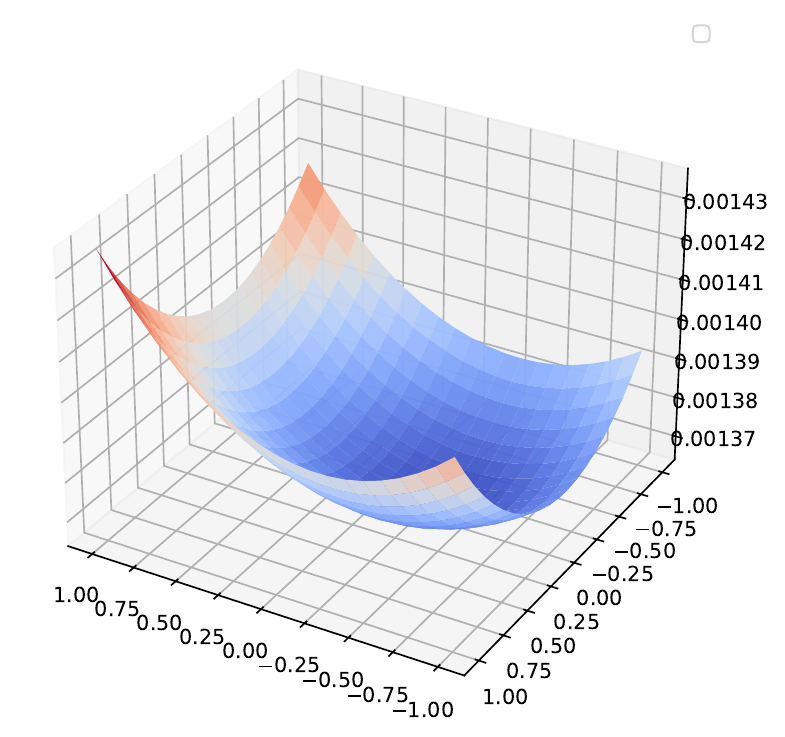}
        \caption{DPPF\textsubscript{SimpleAvg}}
    \end{subfigure}
    %\vspace{-1mm}
    \caption{Training loss landscapes, $\text{lim}=1$, $\text{step}=0.1$, 4 workers, CIFAR-10.}
    
\end{figure}
%\vspace{-10mm}
\begin{figure}[H]
    \centering
    \begin{subfigure}{0.34\columnwidth} % Adjust width as needed
        \centering
        \includegraphics[width=\columnwidth]{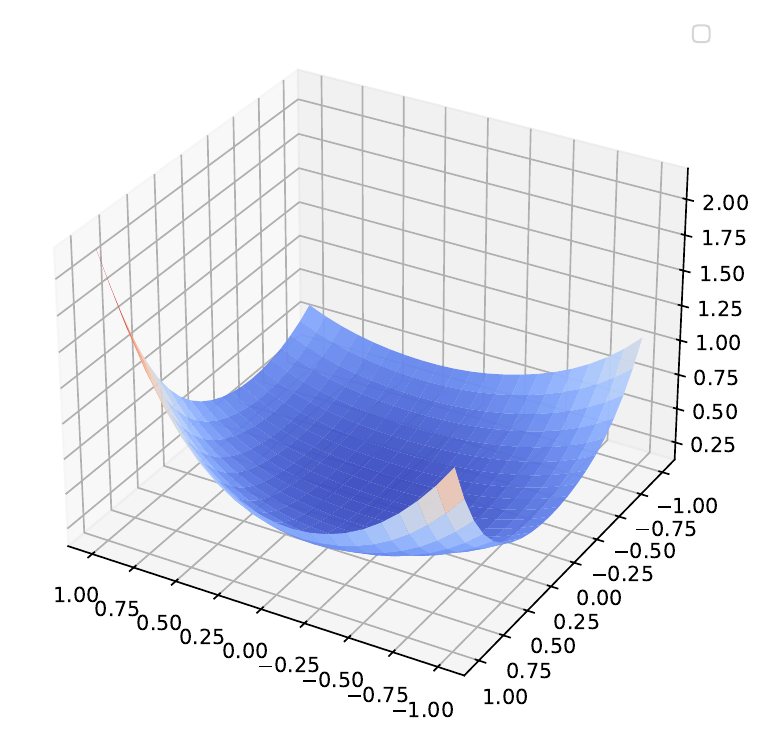}
        \caption{SimpleAvg}
        
    \end{subfigure}
    \hspace{10mm}
    \begin{subfigure}{0.34\columnwidth} % Adjust width as needed
        \centering
        \includegraphics[width=\columnwidth]{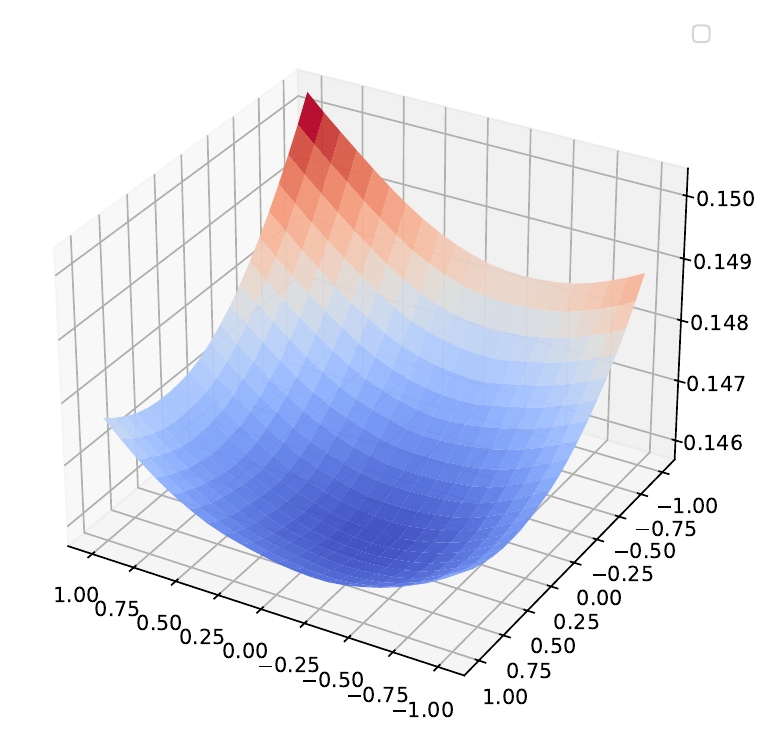}
        \caption{DPPF\textsubscript{SimpleAvg}}
    \end{subfigure}
    %\vspace{-1mm}
    \caption{Test loss landscapes, $\text{lim}=1$, $\text{step}=0.1$, 4 workers, CIFAR-10.}
    
\end{figure}

%\vspace{-10mm}
\begin{figure}[H]
    \centering
    \begin{subfigure}{0.34\columnwidth} % Adjust width as needed
        \centering
        \includegraphics[width=\columnwidth]{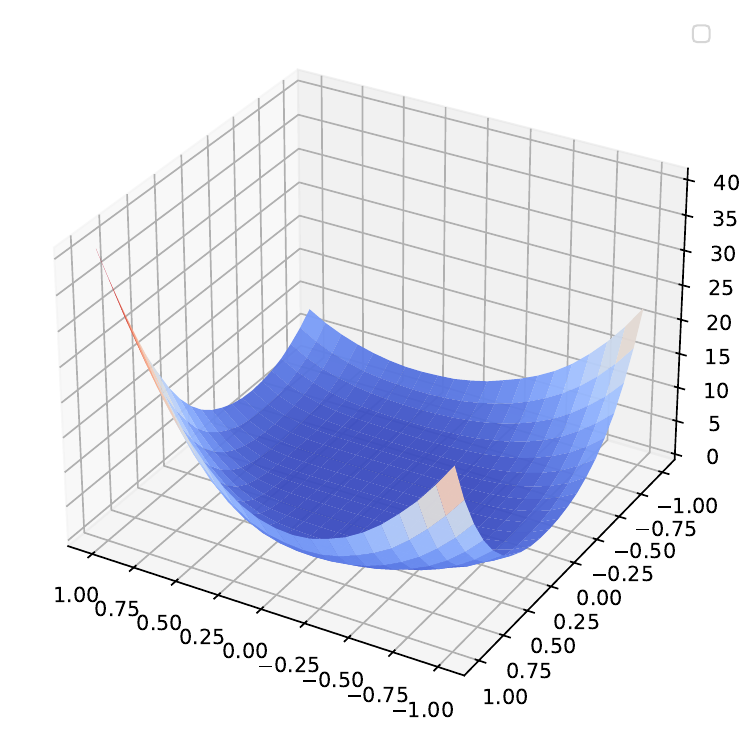}
        \caption{SimpleAvg}
        
    \end{subfigure}
    \hspace{10mm}
    \begin{subfigure}{0.34\columnwidth} % Adjust width as needed
        \centering
        \includegraphics[width=\columnwidth]{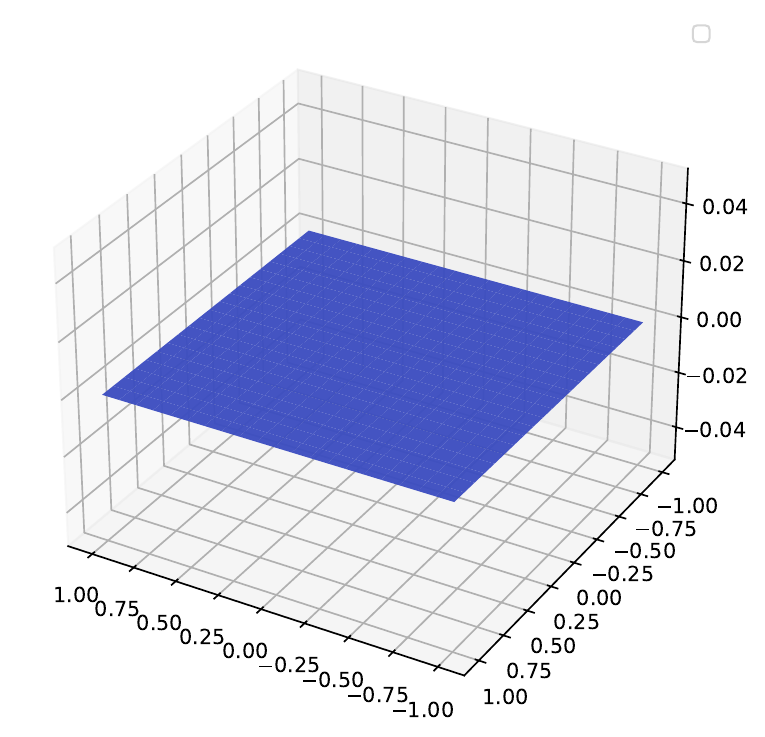}
        \caption{DPPF\textsubscript{SimpleAvg}}
    \end{subfigure}
    %\vspace{-1mm}
    \caption{Training error (\%) landscapes, $\text{lim}=1$, $\text{step}=0.1$, 4 workers, CIFAR-10.}
    
\end{figure}

%\vspace{-10mm}
\begin{figure}[H]
    \centering
    \begin{subfigure}{0.34\columnwidth} % Adjust width as needed
        \centering
        \includegraphics[width=\columnwidth]{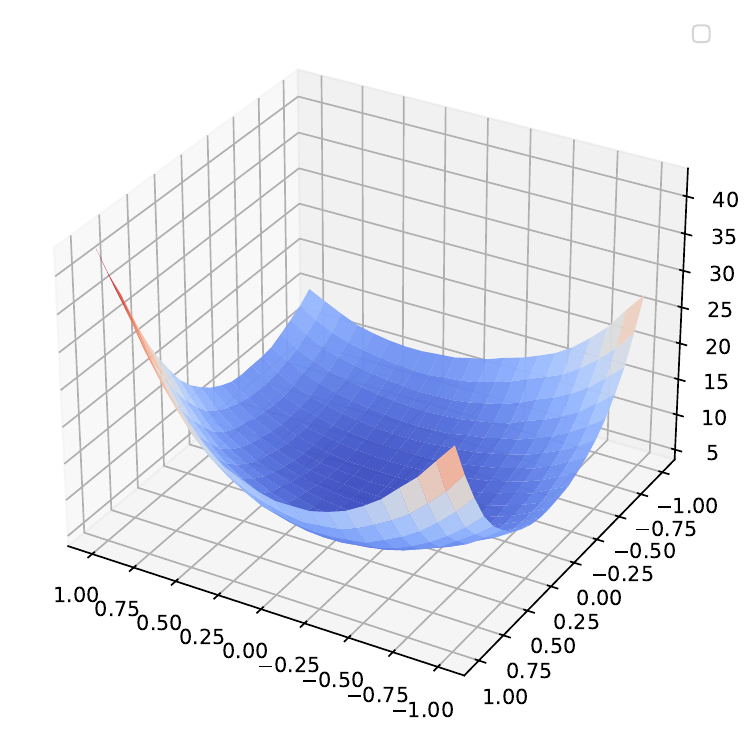}
        \caption{SimpleAvg}
        
    \end{subfigure}
    \hspace{10mm}
    \begin{subfigure}{0.34\columnwidth} % Adjust width as needed
        \centering
        \includegraphics[width=\columnwidth]{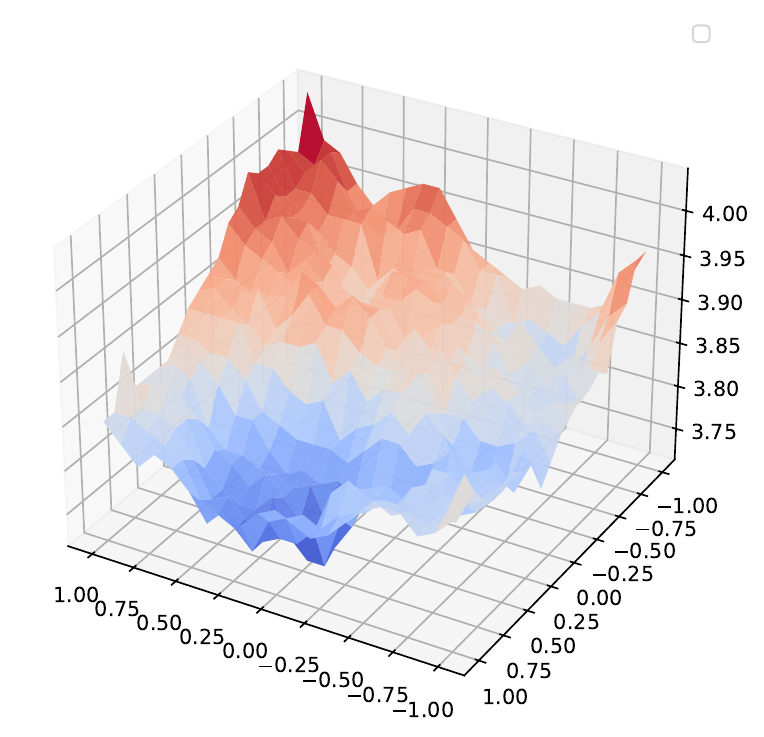}
        \caption{DPPF\textsubscript{SimpleAvg}}
    \end{subfigure}
    %\vspace{-1mm}
    \caption{Test error (\%) landscapes, $\text{lim}=1$, $\text{step}=0.1$, 4 workers, CIFAR-10.}
    
\end{figure}

%\vspace{-10mm}
\begin{figure}[H]
    \centering
    \begin{subfigure}{0.34\columnwidth} % Adjust width as needed
        \centering
        \includegraphics[width=\columnwidth]{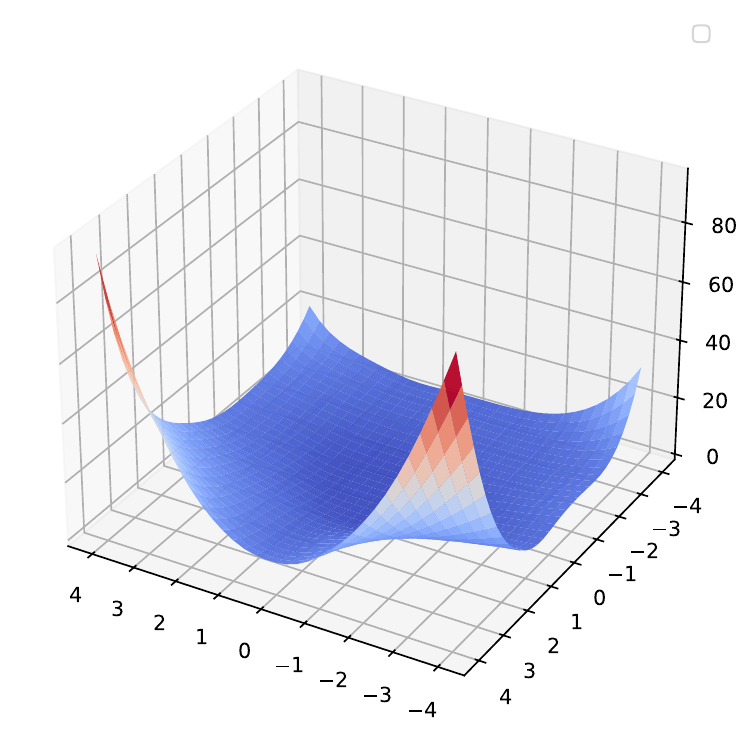}
        \caption{SimpleAvg}
        
    \end{subfigure}
    \hspace{10mm}
    \begin{subfigure}{0.34\columnwidth} % Adjust width as needed
        \centering
        \includegraphics[width=\columnwidth]{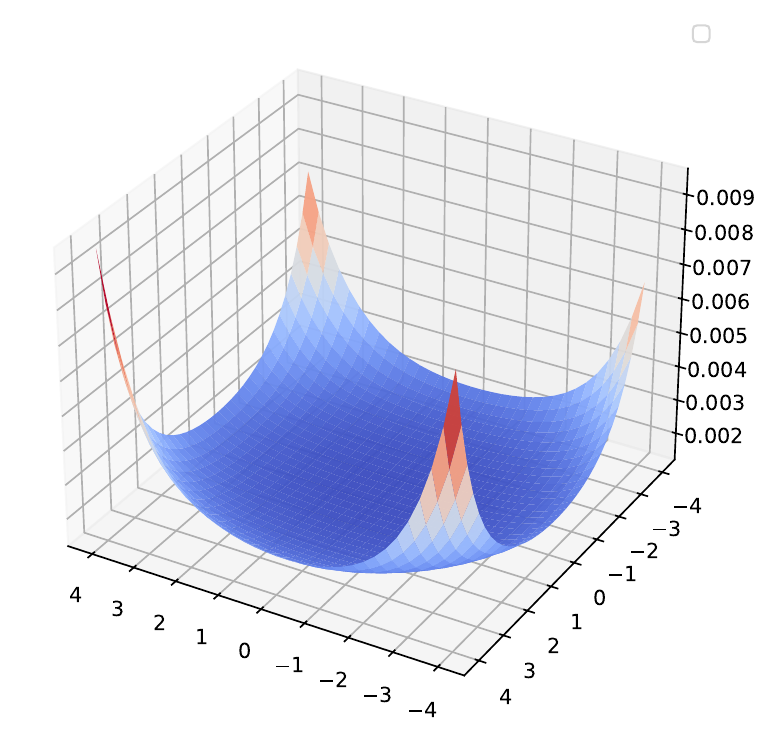}
        \caption{DPPF\textsubscript{SimpleAvg}}
    \end{subfigure}
    %\vspace{-1mm}
    \caption{Training loss landscapes, $\text{lim}=4$, $\text{step}=0.25$, 4 workers, CIFAR-10.}
    
\end{figure}

%\vspace{-10mm}
\begin{figure}[H]
    \centering
    \begin{subfigure}{0.34\columnwidth} % Adjust width as needed
        \centering
        \includegraphics[width=\columnwidth]{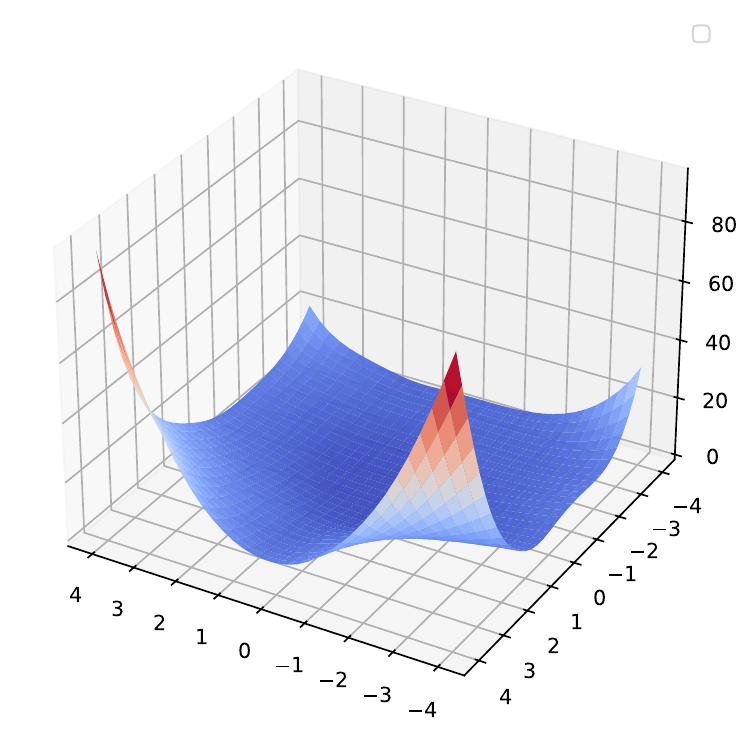}
        \caption{SimpleAvg}
        
    \end{subfigure}
    \hspace{10mm}
    \begin{subfigure}{0.34\columnwidth} % Adjust width as needed
        \centering
        \includegraphics[width=\columnwidth]{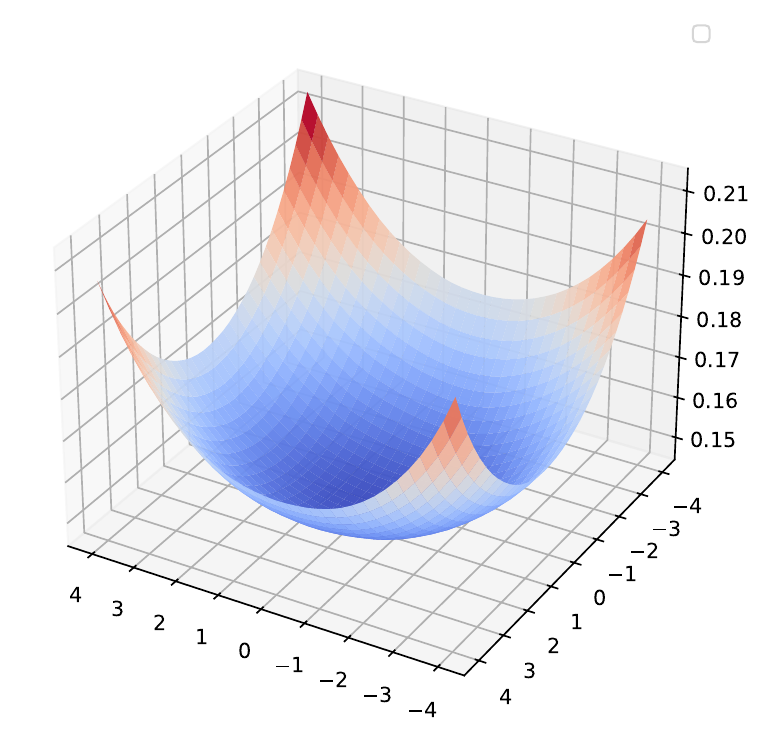}
        \caption{DPPF\textsubscript{SimpleAvg}}
    \end{subfigure}
    %\vspace{-1mm}
    \caption{Test loss landscapes, $\text{lim}=4$, $\text{step}=0.25$, 4 workers, CIFAR-10.}
    
\end{figure}

%\vspace{-10mm}
\begin{figure}[H]
    \centering
    \begin{subfigure}{0.34\columnwidth} % Adjust width as needed
        \centering
        \includegraphics[width=\columnwidth]{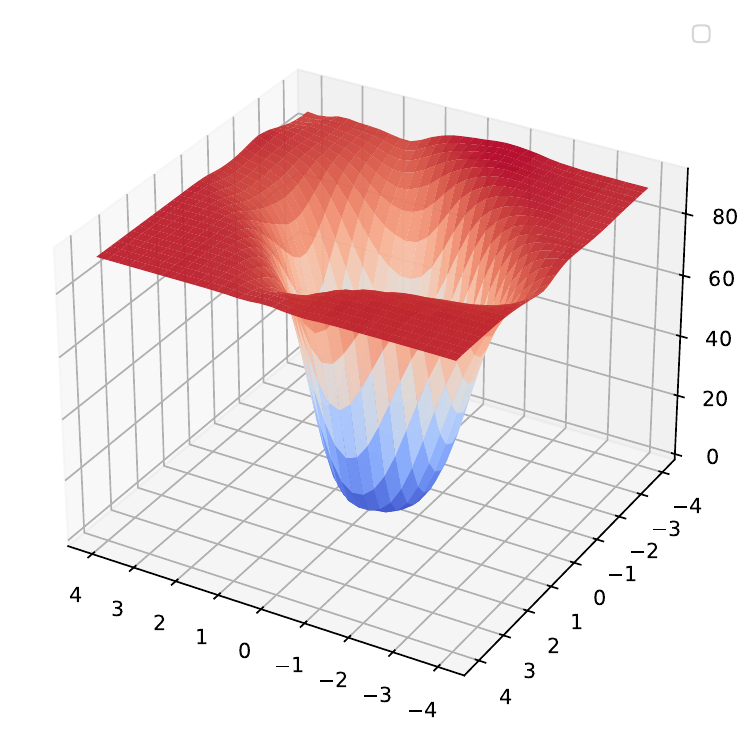}
        \caption{SimpleAvg}
        
    \end{subfigure}
    \hspace{10mm}
    \begin{subfigure}{0.34\columnwidth} % Adjust width as needed
        \centering
        \includegraphics[width=\columnwidth]{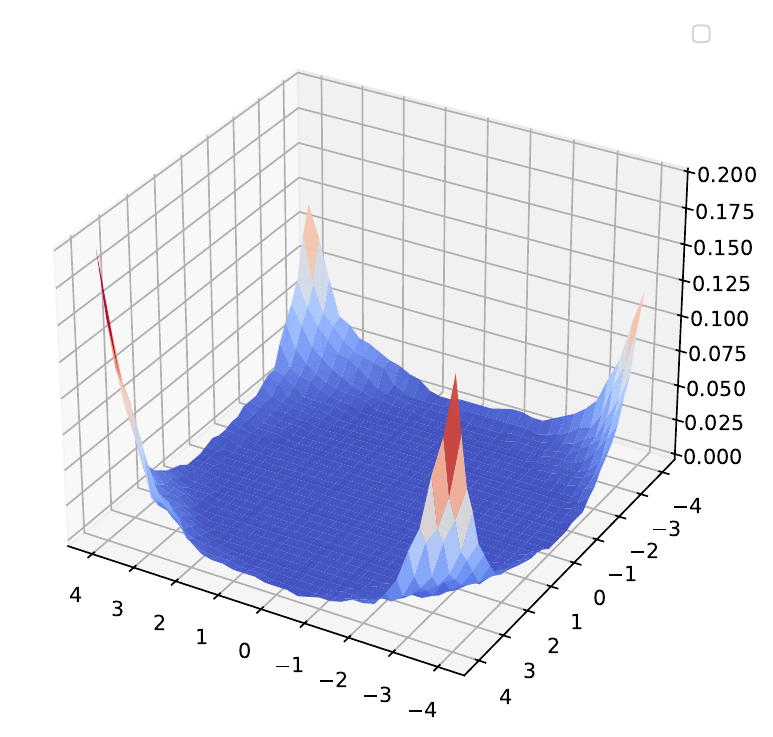}
        \caption{DPPF\textsubscript{SimpleAvg}}
    \end{subfigure}
    %\vspace{-1mm}
    \caption{Training error (\%) landscapes, $\text{lim}=4$, $\text{step}=0.25$, 4 workers, CIFAR-10.}
    
\end{figure}
%\vspace{-10mm}
\begin{figure}[H]
    \centering
    \begin{subfigure}{0.34\columnwidth} % Adjust width as needed
        \centering
        \includegraphics[width=\columnwidth]{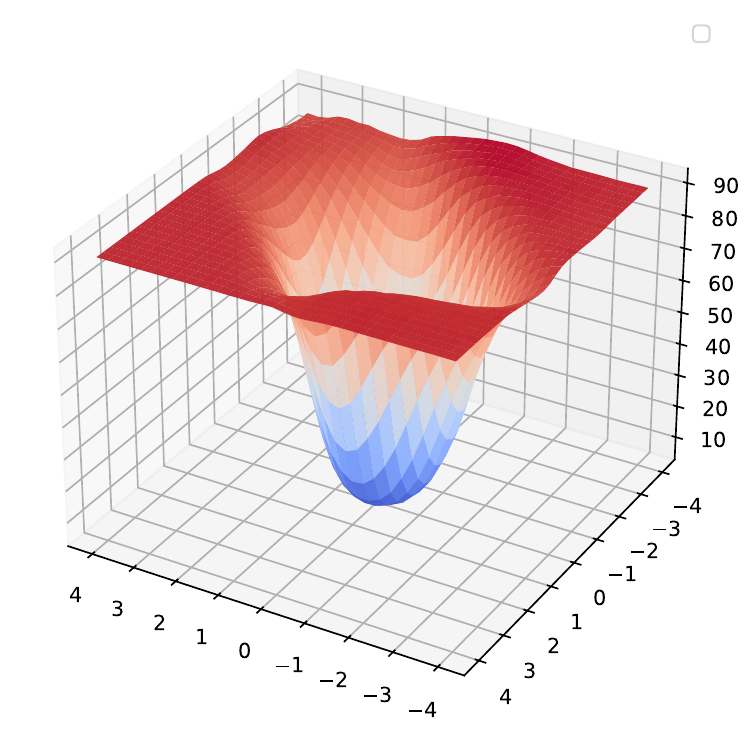}
        \caption{SimpleAvg}
        
    \end{subfigure}
    \hspace{10mm}
    \begin{subfigure}{0.34\columnwidth} % Adjust width as needed
        \centering
        \includegraphics[width=\columnwidth]{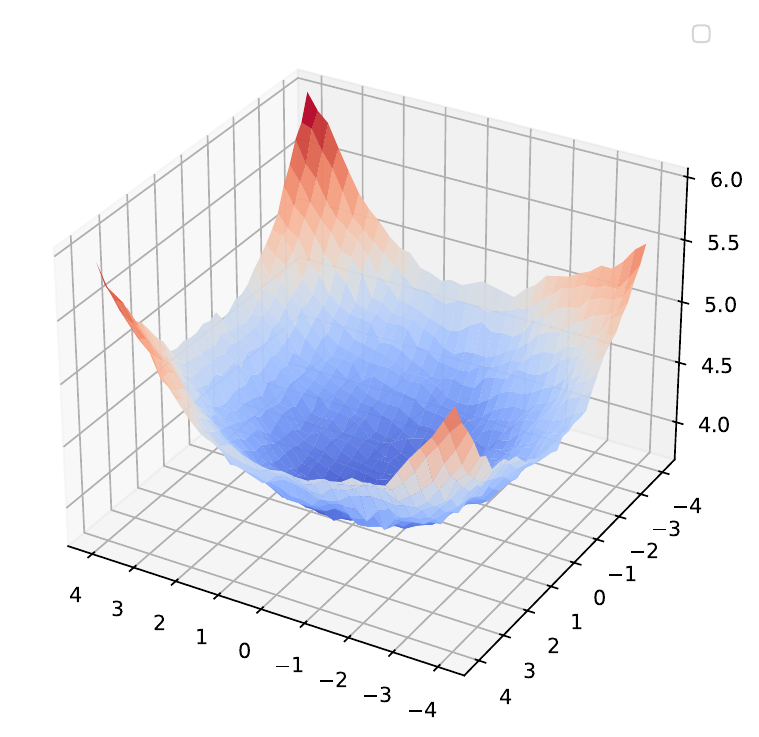}
        \caption{DPPF\textsubscript{SimpleAvg}}
    \end{subfigure}
    %\vspace{-1mm}
    \caption{Test error (\%) landscapes, $\text{lim}=4$, $\text{step}=0.25$, 4 workers, CIFAR-10.}
    
\end{figure}

% \subsection{CIFAR-100 (4 Workers) - 3D Visualizations}
\label{appendix:subsubsec:c100_landscape_3d}
%\vspace{-10mm}
\begin{figure}[H]
    \centering
    \begin{subfigure}{0.34\columnwidth} % Adjust width as needed
        \centering
        \includegraphics[width=\columnwidth]{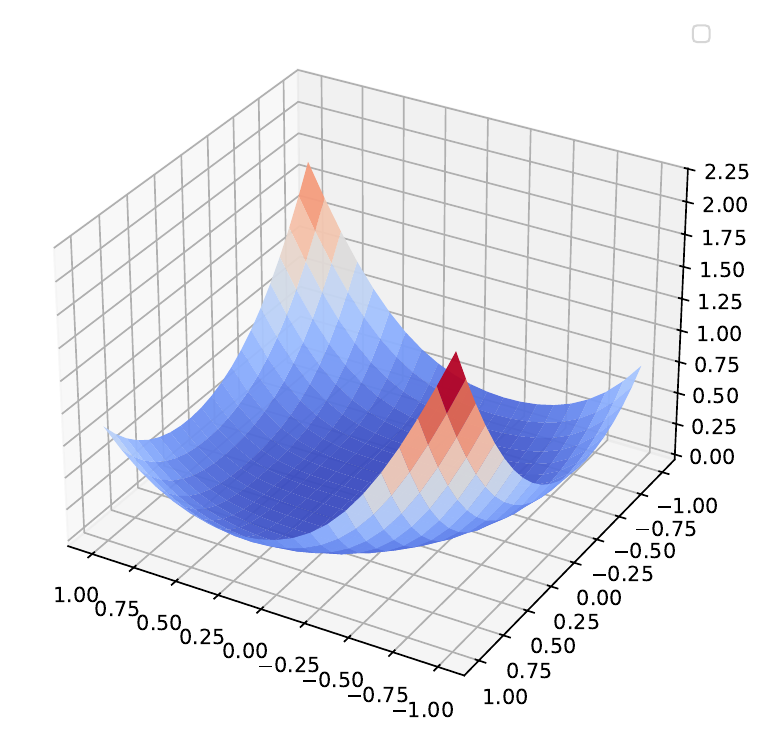}
        \caption{SimpleAvg}
        
    \end{subfigure}
    \hspace{10mm}
    \begin{subfigure}{0.34\columnwidth} % Adjust width as needed
        \centering
        \includegraphics[width=\columnwidth]{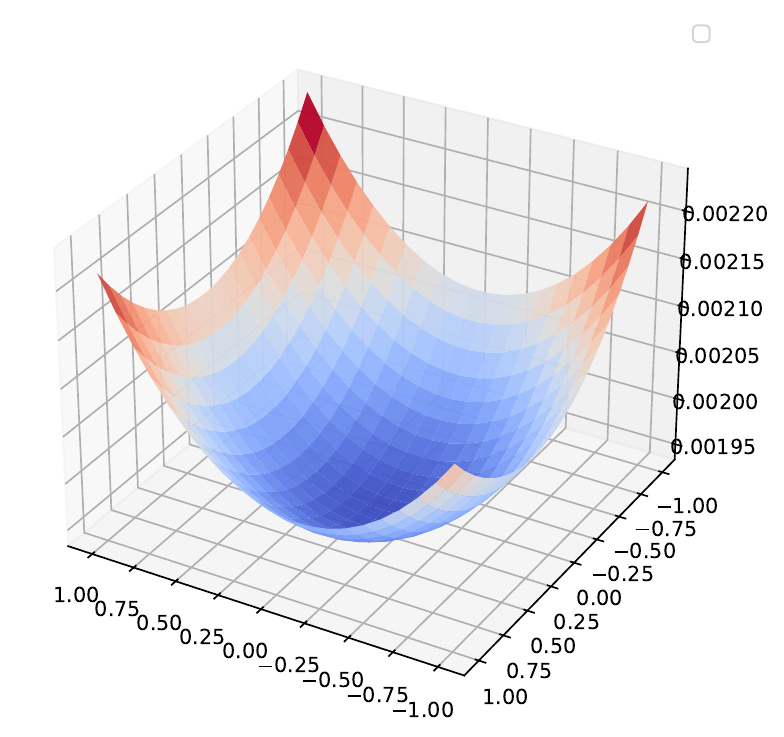}
        \caption{DPPF\textsubscript{SimpleAvg}}
    \end{subfigure}
    %\vspace{-1mm}
    \caption{Training loss landscapes, $\text{lim}=1$, $\text{step}=0.1$, 4 workers, CIFAR-100.}
    
\end{figure}
%\vspace{-10mm}
\begin{figure}[H]
    \centering
    \begin{subfigure}{0.34\columnwidth} % Adjust width as needed
        \centering
        \includegraphics[width=\columnwidth]{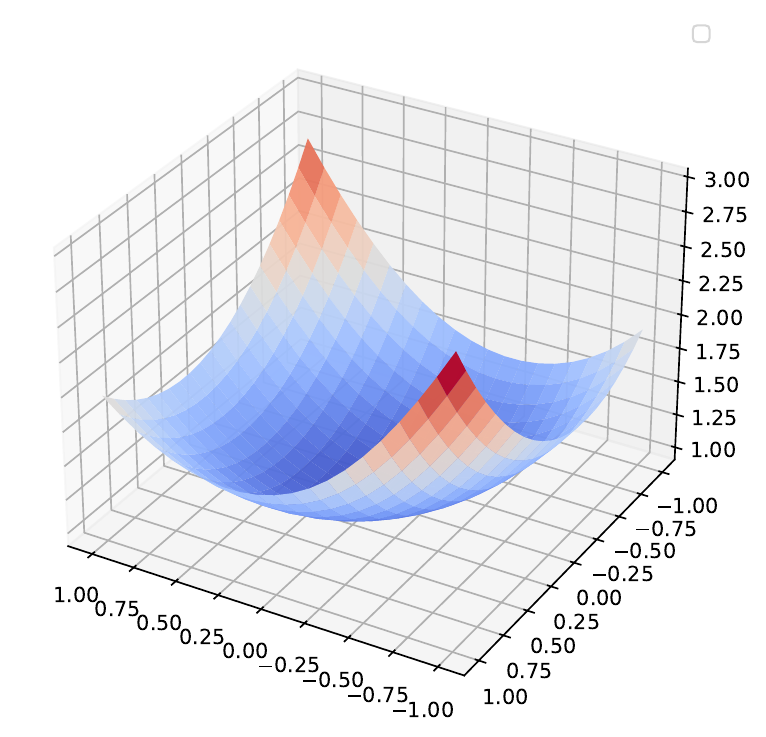}
        \caption{SimpleAvg}
        
    \end{subfigure}
    \hspace{10mm}
    \begin{subfigure}{0.34\columnwidth} % Adjust width as needed
        \centering
        \includegraphics[width=\columnwidth]{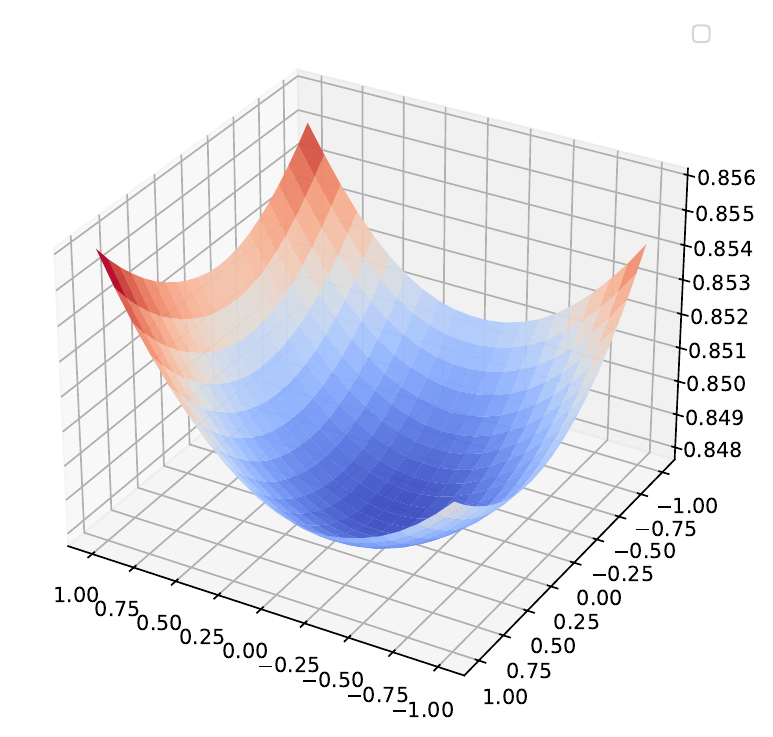}
        \caption{DPPF\textsubscript{SimpleAvg}}
    \end{subfigure}
    %\vspace{-1mm}
    \caption{Test loss landscapes, $\text{lim}=1$, $\text{step}=0.1$, 4 workers, CIFAR-100.}
    
\end{figure}

%\vspace{-10mm}
\begin{figure}[H]
    \centering
    \begin{subfigure}{0.34\columnwidth} % Adjust width as needed
        \centering
        \includegraphics[width=\columnwidth]{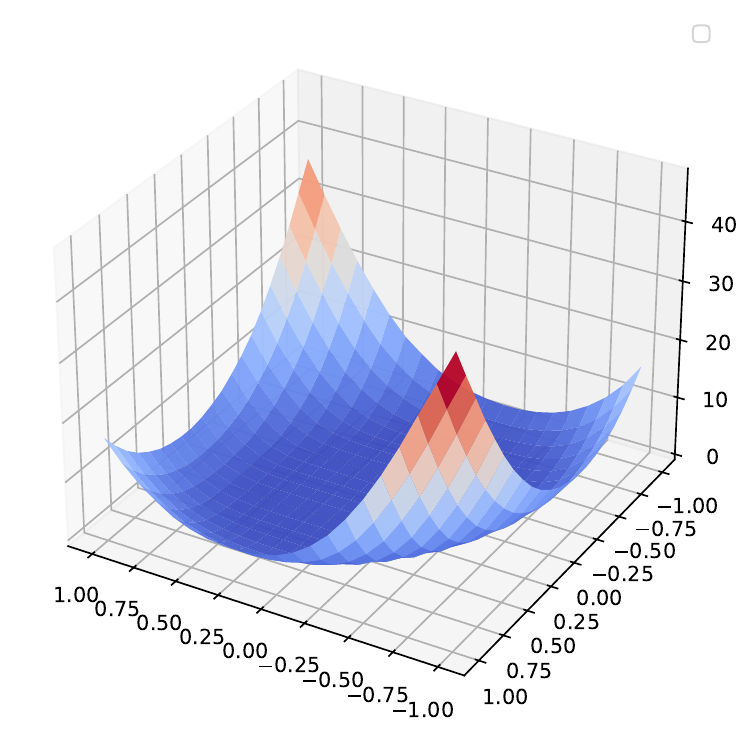}
        \caption{SimpleAvg}
        
    \end{subfigure}
    \hspace{10mm}
    \begin{subfigure}{0.34\columnwidth} % Adjust width as needed
        \centering
        \includegraphics[width=\columnwidth]{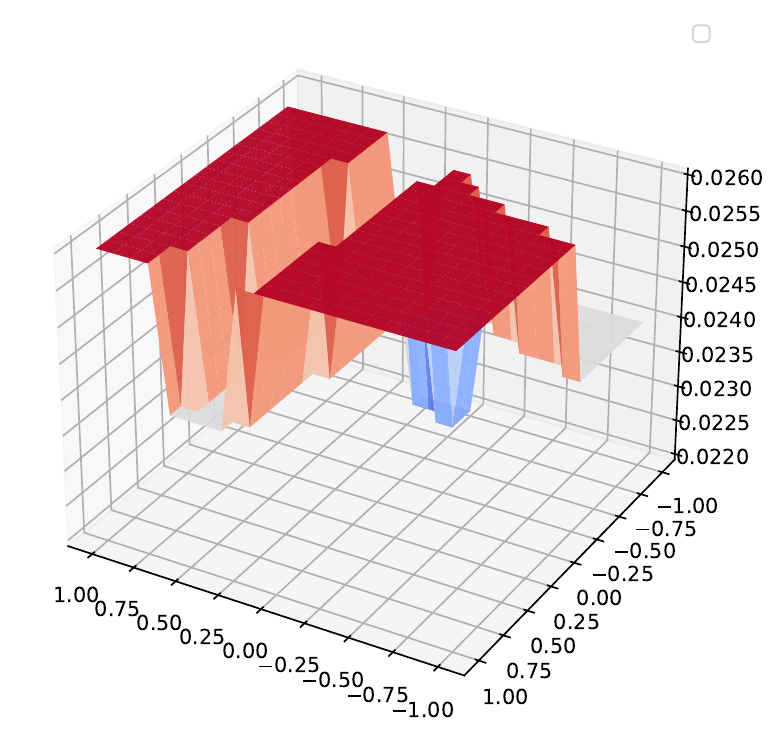}
        \caption{DPPF\textsubscript{SimpleAvg}}
    \end{subfigure}
    %\vspace{-1mm}
    \caption{Training error (\%) landscapes, $\text{lim}=1$, $\text{step}=0.1$, 4 workers, CIFAR-100.}
    
\end{figure}
%\vspace{-10mm}
\begin{figure}[H]
    \centering
    \begin{subfigure}{0.34\columnwidth} % Adjust width as needed
        \centering
        \includegraphics[width=\columnwidth]{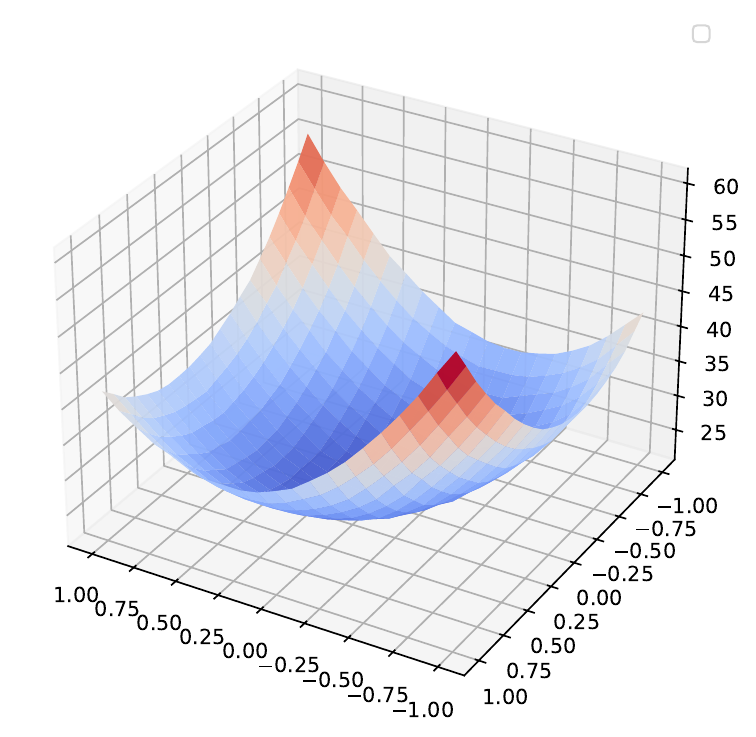}
        \caption{SimpleAvg}
        
    \end{subfigure}
    \hspace{10mm}
    \begin{subfigure}{0.34\columnwidth} % Adjust width as needed
        \centering
        \includegraphics[width=\columnwidth]{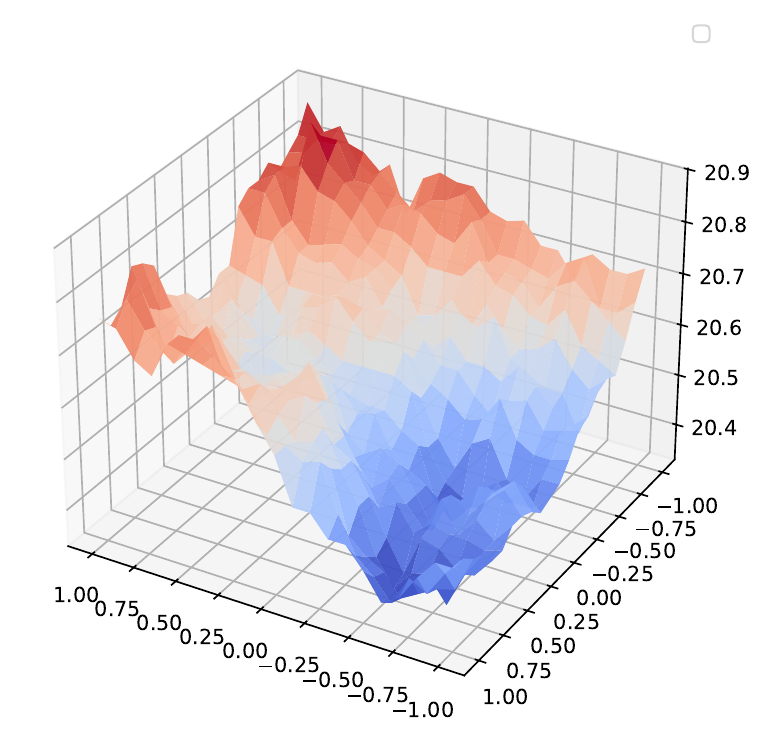}
        \caption{DPPF\textsubscript{SimpleAvg}}
    \end{subfigure}
    %\vspace{-1mm}
    \caption{Test error (\%) landscapes, $\text{lim}=1$, $\text{step}=0.1$, 4 workers, CIFAR-100.}
    
\end{figure}

%%%%%%%%%%%%%%%%%%%%%%%%%%%%%%%%%%%%%%%%%%%%%%%%%%%%%%%%%%%%%%%%%%%%%%%%%%%%%%%%%%%%%%%%%%%%%%%%%%%%%%%%%%%%%
%\vspace{-10mm}
\begin{figure}[H]
    \centering
    \begin{subfigure}{0.34\columnwidth} % Adjust width as needed
        \centering
        \includegraphics[width=\columnwidth]{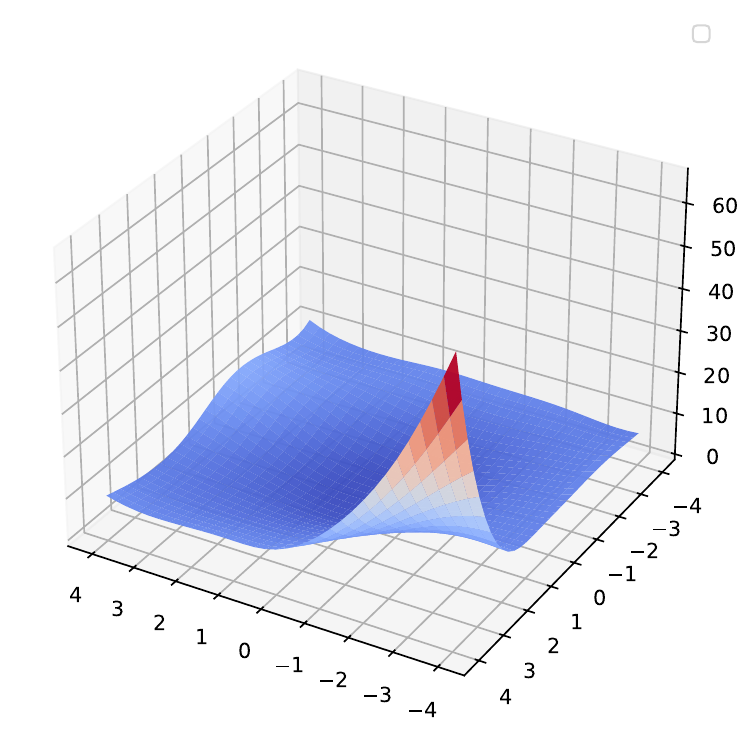}
        \caption{SimpleAvg}
        
    \end{subfigure}
    \hspace{10mm}
    \begin{subfigure}{0.34\columnwidth} % Adjust width as needed
        \centering
        \includegraphics[width=\columnwidth]{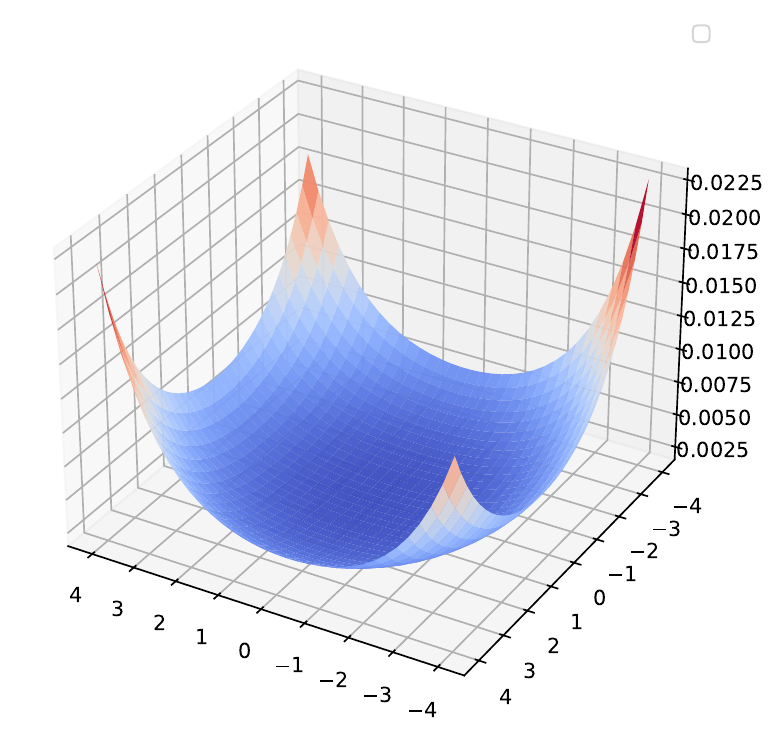}
        \caption{DPPF\textsubscript{SimpleAvg}}
    \end{subfigure}
    %\vspace{-1mm}
    \caption{Training loss landscapes, $\text{lim}=4$, $\text{step}=0.25$, 4 workers, CIFAR-100.}
    
\end{figure}

%\vspace{-10mm}
\begin{figure}[H]
    \centering
    \begin{subfigure}{0.34\columnwidth} % Adjust width as needed
        \centering
        \includegraphics[width=\columnwidth]{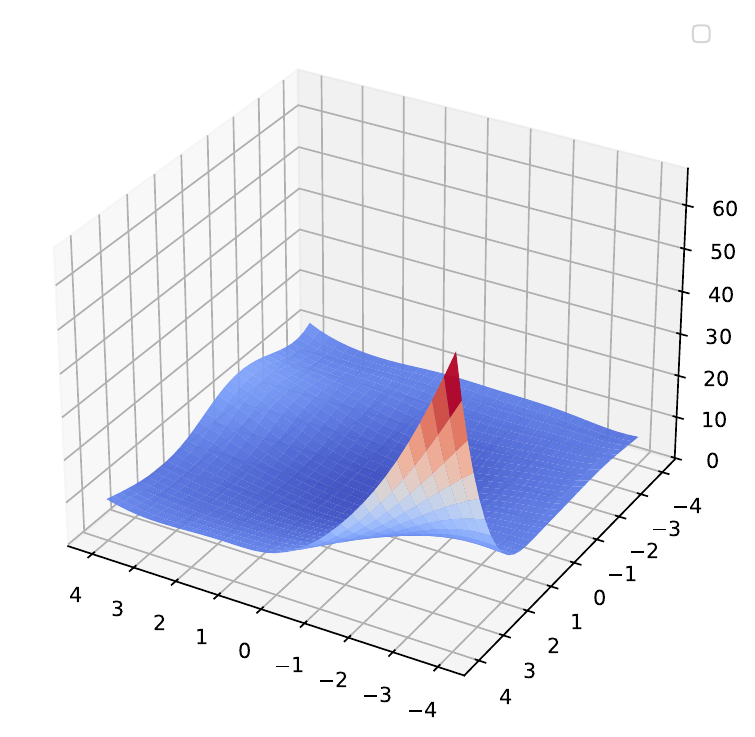}
        \caption{SimpleAvg}
        
    \end{subfigure}
    \hspace{10mm}
    \begin{subfigure}{0.34\columnwidth} % Adjust width as needed
        \centering
        \includegraphics[width=\columnwidth]{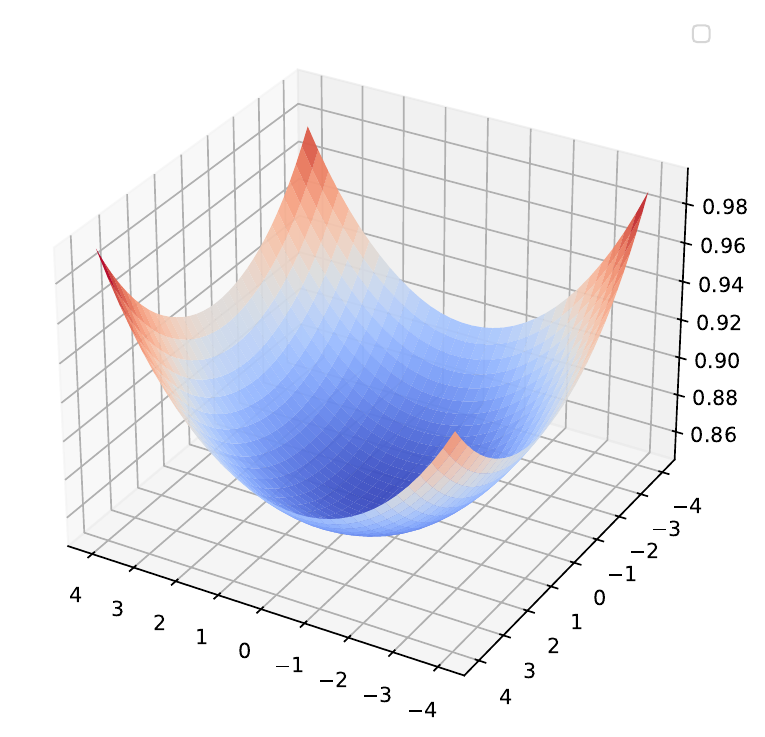}
        \caption{DPPF\textsubscript{SimpleAvg}}
    \end{subfigure}
    %\vspace{-1mm}
    \caption{Test loss landscapes, $\text{lim}=4$, $\text{step}=0.25$, 4 workers, CIFAR-100.}
    
\end{figure}

%\vspace{-10mm}
\begin{figure}[H]
    \centering
    \begin{subfigure}{0.34\columnwidth} % Adjust width as needed
        \centering
        \includegraphics[width=\columnwidth]{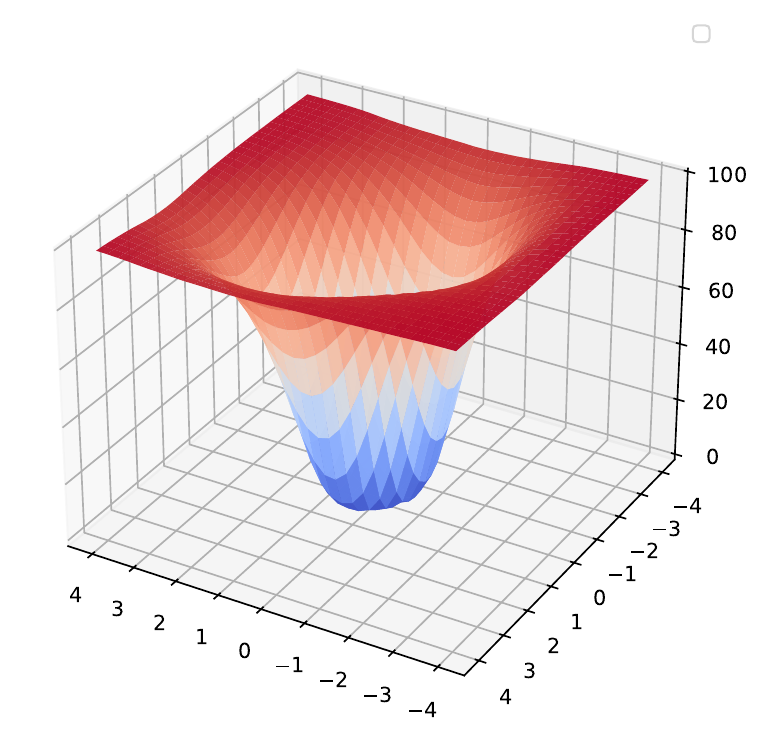}
        \caption{SimpleAvg}
        
    \end{subfigure}
    \hspace{10mm}
    \begin{subfigure}{0.34\columnwidth} % Adjust width as needed
        \centering
        \includegraphics[width=\columnwidth]{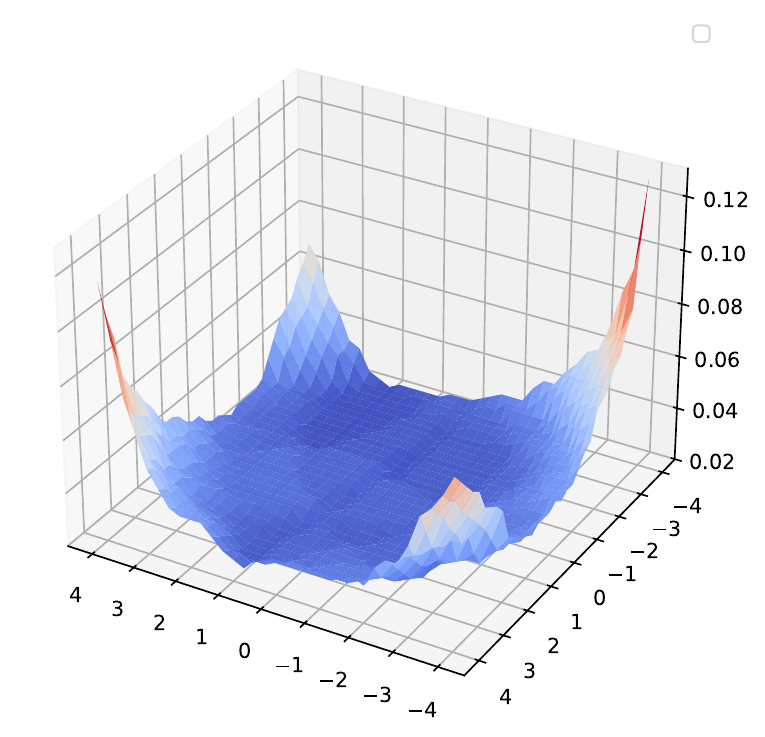}
        \caption{DPPF\textsubscript{SimpleAvg}}
    \end{subfigure}
    %\vspace{-1mm}
    \caption{Training error (\%) landscapes, $\text{lim}=4$, $\text{step}=0.25$, 4 workers, CIFAR-100.}
    
\end{figure}

%\vspace{-10mm}
\begin{figure}[H]
    \centering
    \begin{subfigure}{0.34\columnwidth} % Adjust width as needed
        \centering
        \includegraphics[width=\columnwidth]{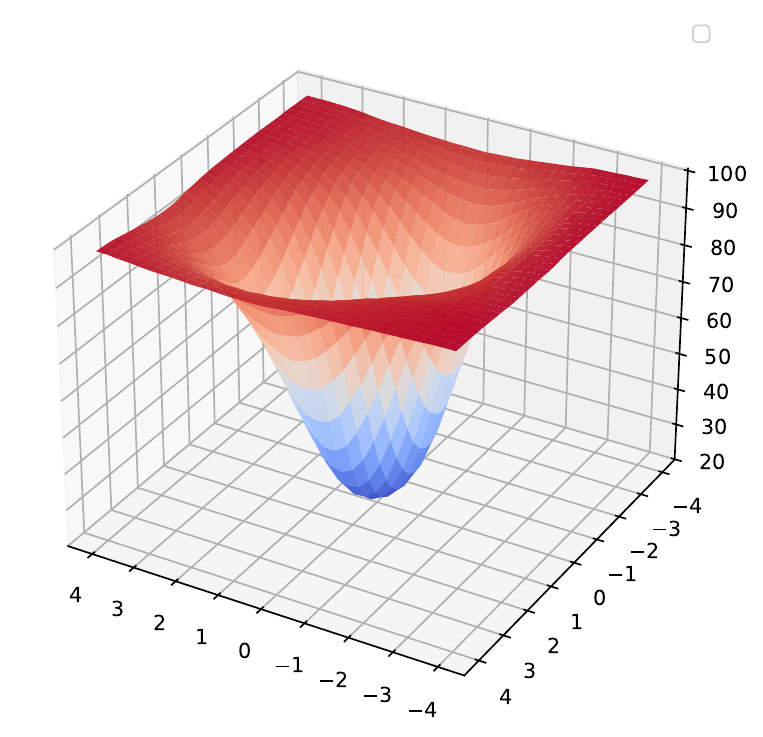}
        \caption{SimpleAvg}
        
    \end{subfigure}
    \hspace{10mm}
    \begin{subfigure}{0.34\columnwidth} % Adjust width as needed
        \centering
        \includegraphics[width=\columnwidth]{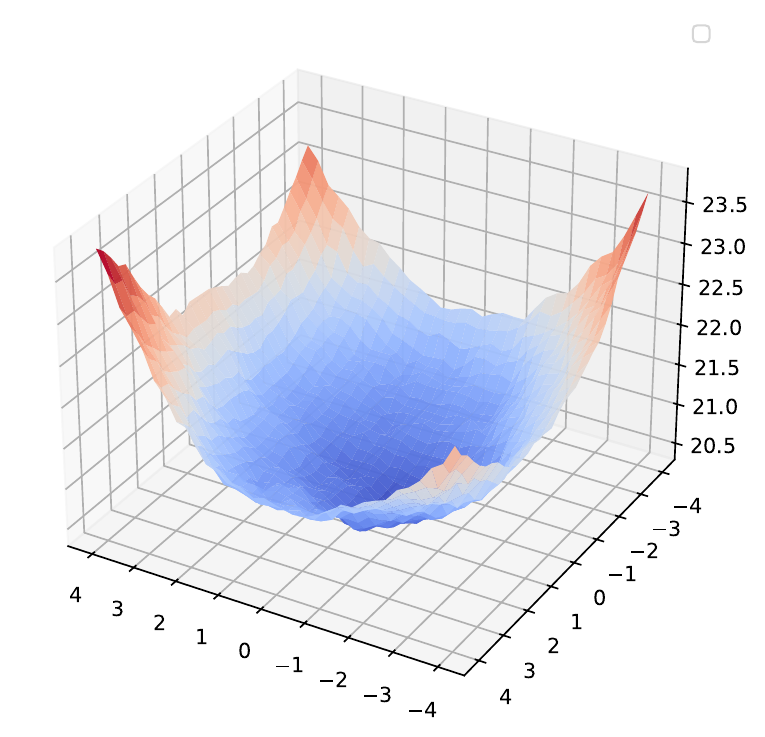}
        \caption{DPPF\textsubscript{SimpleAvg}}
    \end{subfigure}
    %\vspace{-1mm}
    \caption{Test error (\%) landscapes, $\text{lim}=4$, $\text{step}=0.25$, 4 workers, CIFAR-100.}
    
\end{figure}